\documentclass{article} %
\usepackage{arxiv}

\usepackage[utf8]{inputenc} %
\usepackage[T1]{fontenc}    %
\usepackage{hyperref}       %
\usepackage{url}            %
\usepackage{booktabs}       %
\usepackage{amsfonts}       %
\usepackage{nicefrac}       %
\usepackage{microtype}      %
\usepackage{xcolor}         %
\usepackage{bm}
\usepackage{natbib}
\usepackage{algorithm} 
\usepackage{algpseudocode}
\usepackage{amsmath}  
\allowdisplaybreaks
\usepackage{amssymb}     
\usepackage{amsthm}
\usepackage[nameinlink,noabbrev,capitalise]{cleveref}

\usepackage{xspace}
\usepackage[inline]{enumitem}
\usepackage{stmaryrd}
\usepackage{multirow}
\usepackage{soul}
\usepackage{makecell}
\usepackage{etoc}
\usepackage{placeins}

\usepackage{tablefootnote}
\usepackage{caption}

\usepackage{tikz}
\usetikzlibrary{fit,calc}

\usepackage[dvipsnames]{xcolor}

\newcolumntype{L}{>{$}l<{$}}

\hypersetup{hidelinks}

\input{macros}

\title{\hetalgname: Sketch-Aggregated Federated LoRA with Heterogeneous Client Ranks}

\author{%
  Yue Xia 
  \\
  Technical University of Munich\\
  Munich, 80333 Germany \\
  \texttt{yue1.xia@tum.de} \\
  \And
  Tayyebeh Jahani-Nezhad \\
  Technische Universität Berlin \\
  Berlin, 10587 Germany \\
  \texttt{t.jahani.nezhad@tu-berlin.de} \\
  \And
  Mayank Bakshi \\
  Northern Arizona University \\
  Flagstaff, AZ 86011 USA \\
  \texttt{mayank.bakshi@nau.edu} \\
  \And
  Rawad Bitar \\
  Technical University of Munich\\
  Munich, 80333 Germany \\
  \texttt{rawad.bitar@tum.de} \\
}

\begin{document}

\maketitle

\begin{abstract}
    We consider federated parameter-efficient fine-tuning of foundation models, e.g., LLMs and LVMs, using low-rank adaptation (LoRA; ~Hu et al.\ 2022). Combining LoRA with federated learning introduces challenges absent from either setting alone: clients may use different LoRA ranks, making their factor matrices dimension-incompatible, and factor-wise averaging suffers from a bilinear mismatch. We propose \hetalgname, a sketch-aggregated federated LoRA algorithm in which each client transmits a linear sketch of its local $\Update$ updates, enabling direct aggregation at the federator. As a result, \hetalgname overcomes the rank heterogeneity challenge and alleviates the bilinear mismatch. We establish convergence to a neighborhood of a first-order stationary point at rate $\cO(T^{-1/2})$. Extensive numerical experiments on language and vision models showcase the performance of our algorithm compared to existing methods.
\end{abstract}

\etocdepthtag.toc{mainmatter}

\section{Introduction}
Foundation models %
have remarkable learning capabilities and demonstrate strong performance across a diverse range of tasks. %
Adapting such foundation models to downstream tasks relies on access to vast, potentially heterogeneous datasets and requires tremendous computational resources, especially when fine-tuning all model parameters. Since fine-tuning all model parameters requires substantial computational and memory resources, parameter-efficient fine-tuning (PEFT) methods, which adapt a pretrained model by updating only a small, structured subset of parameters, are employed, see \citep{houlsbyParameterEfficientTransferLearning2019,liaoParameterEfficientFineTuningIntroducing2023,hu2022lora,tianHydraLoRAAsymmetricLoRA2025, hanParameterEfficientFineTuningLarge2024}. %

Training on large amounts of distributed data can be achieved using
federated learning (FL)~\citep{mcmahan2017communication}. FL is a distributed learning paradigm in which a central server (federator) coordinates the training of a neural network on heterogeneous data generated and owned by participating clients. The clients keep their data, train the neural network locally, and share only their model updates with the federator. As such, FL enhances data privacy by keeping raw data local, reduces communication and storage costs associated with transferring data to a central server, and scales well due to its ease of implementation at the clients. FL has attracted significant attention from the scientific community and is being applied across several domains~\citep{kairouz2021advances,jiang2025federated}.

Combining PEFT with FL %
spans adapter-based methods \citep{cai2023efficient,ghiasvand2024communication}, prompt tuning \citep{zhao2023fedprompt,qiu2023text}, selective parameter optimization \citep{yu2023bridging}, and low-rank gradient subspace optimization~\citep{peng2026rethinking}. We focus on federated PEFT methods using Low-Rank Adaptation (LoRA)~\citep{hu2022lora}, which has emerged as one of the most popular PEFT techniques, e.g., \citep{tianHydraLoRAAsymmetricLoRA2025,zhang2023adalora,hayou2024loraplus,liu2024dora,zhou2025loradrop,chen2024llamalora,kim2025lora} and used in FL, e.g., \citep{zhang2023fedit,liu2026flexlora,liu2026rethinking,cho2024heterogeneous,chen2024rolora,sun2024improving}. %

LoRA-based FL introduces challenges that do not arise in either centralized LoRA or conventional FL alone. Unlike standard FL, LoRA updates cannot be averaged directly. In LoRA, for each layer of the network, each client $i$ computes an update of the form\footnote{We omit the layer index as the operation is the same for all layers.} $\Update_i = \B_i\A_i$, where the model update $\Update_i \in \mathbb{R}^{d\times s}$ is represented as the product of two factor matrices called \emph{adapters} $\B_i \in \mathbb{R}^{d \times r}$, $\A_i \in \mathbb{R}^{r \times s}$, for a chosen $r \ll \min(d,s)$. This factorization, together with the choice of $r$, reduces the number of trainable parameters per layer from $ds$ to $(d+s)r$ and restricts the rank of $\Update_i$ to $r$. In LoRA-based FL, the federator aims to obtain $\Update \defeq \sum_i p_i \Update_i = \sum_{i}p_i\B_i \A_i$, where $p_i\geq 0$ and $\sum_ip_i = 1$. Transmitting $\Update_i$ to the federator is undesirable, as it incurs the same communication cost as transmitting a full fine-tuning update. 
Instead, the clients can transmit $\B_i$ and $\A_i$ separately. The federator has two options: \begin{enumerate*}[label=\emph{(\roman*)}]
    \item compute each $\Update_i=\B_i\A_i$ and average them to obtain $\Update$;
    or \item average the $\A_i$'s and $\B_i$'s first to obtain $\widebar{\A}=\sum_i p_i \A_i$ and $\widebar{\B}=\sum_i p_i \B_i$ and compute $\widebar{\Update} = \widebar{\B} \widebar{\A}$.
\end{enumerate*}
Each approach introduces its own challenges, affecting either the computational efficiency or the accuracy of the aggregated update, which in turn influences the convergence of the federated fine-tuning algorithm. We focus on the following challenges.

\paragraph{Challenge 1: Rank-heterogeneity} In the rank-heterogeneous setting, clients use different adapter ranks $r_i$ to match their local compute budgets or the complexity of their data~\citep{hu2022lora}, resulting in matrices with incompatible dimensions in aggregation. 

\paragraph{Challenge 2: Bilinear mismatch} Zero-padding and averaging the adapters first leads to the \emph{bilinear mismatch} problem (also known as aggregation bias): $\widebar{\B}\widebar{\A} \neq \sum_i p_i \B_i \A_i$. The bilinear mismatch problem is exacerbated when considering differential privacy~\citep{liu2026rethinking} and the non-linearity does not allow for secure aggregation, used in privacy-preserving FL~\citep{kairouz2021advances}.

\paragraph{Closely related work} We summarize the closest work and give an extensive review in \cref{sec:related}.

Rank heterogeneity is addressed either by operating in the full model space or by padding the adapters to the same dimensions.
In FLoRA~\citep{wang2024flora}, the federator stacks the adapters and sends them to the clients; thus increasing the federator-to-client communication cost. %
A key difference is that clients update the backbone model and reinitialize their adapters at each iteration.
In FlexLoRA~\citep{bai2024federated}, the federator computes and averages the updates $\Update_i=\B_i\A_i$ to obtain $\Update$ and applies SVD on $\Update$ with truncation to obtain rank-$r_i$ update factors sent to client $i$; incurring an $\cO(Ndrs)$ computational overhead per layer for computing the $\Update_i$'s and $\cO(d^2 s)$ for the SVD.
In FedHL~\citep{peng2025fedhl}, the federator also computes the dense updates $\Update_i=\B_i\A_i$ and applies SVD truncation like FlexLoRA, but it introduces a compensation aggregation rule and dynamically adjusted aggregation weights to mitigate the biases caused by truncating the global model to different local ranks. 
In~\citep{cho2024heterogeneous}, the authors remark that computing $\Update_i$ first loses information on the cross-relation across clients and corroborate this observation through numerical experiments. Hence, in their HetLoRA, the federator zero-pads the clients' $\A_i$'s and $\B_i$'s and aggregates them using a careful weighted average. %
The computational complexity of computing the weights and the global update is $\cO(Ndrs)$. Nevertheless, this method does not mitigate the bilinear mismatch problem. Furthermore, zero-padding the matrices is not always faithful to the clients' rank representations, which is why careful computation of the $p_i$'s is needed and why other works, e.g., \citep{byun2025towards,ha2026rb}, replace zero-padding with column repetition.
In FSLoRA~\citep{fang2025federated}, the LoRA parametrization is modified to $\Update = \B\mathbf{S}\A$, where $\mathbf{S}$ is a diagonal $r\times r$ matrix for a given global rank $r$. Each client draws a random diagonal $r\times r$ matrix $\mathbf{S}_i$ with only $r_i$ non-zero entries and updates $\B_i\mathbf{S}_i$ and $\mathbf{S}_i\A_i$ with rank $r_i$. Notice that $\B_i\mathbf{S}_i$ and $\mathbf{S}_i\A_i$ have the same dimension across clients, enabling aggregation. The drawbacks of this method are that it does not mitigate the bilinear mismatch and requires clients to randomly choose the subspace to be updated, as dictated by the choice of $\mathbf{S}_i$, thereby incurring additional variance in convergence.

To our knowledge, there is no rank-heterogeneous LoRA-based FL that mitigates the bilinear mismatch and does not operate in the full model space. Also, only \cite{fang2025federated} provide theoretical insights into the convergence of the rank-heterogeneous algorithm for non-convex loss functions. %

\paragraph{Our contributions}
We introduce \hetalgname, a rank-agnostic federated LoRA-based PEFT method. It alleviates the bilinear mismatch and does not operate in the full model space. Instead of transmitting $\A_i$ and $\B_i$, each client sketches its $\Update_i$ into two fixed-size matrices using the linear sketching method in~\citep{tropp2017practical}. The size of the sketch is comparable to the size of $\B_i$ and $\A_i$. Since the sketching operator is linear, aggregating the sketches at the federator is equivalent to computing the sketch of the aggregated update $\Update$. 
The main properties of \hetalgname are summarized below. 

\begin{enumerate}
    \item \hetalgname aggregates fixed-dimensional sketches, irrespective of clients' local ranks, and can be seamlessly combined with adaptive rank-truncation at the clients. %
    It converges at rate $\mathcal{O}(T^{-1/2})$ to a neighborhood of a first-order stationary point characterized by sketching error, rank truncation, and client heterogeneity.
    \item Numerical experiments on fine-tuning language and vision models in the rank-heterogeneous setting showcase the performance of \hetalgname compared to existing methods.
    \item Mitigating the bilinear mismatch is particularly desirable in settings where differential privacy is required, and linear aggregation enables secure aggregation. Since this work is not directly concerned with privacy guarantees, we will explore those benefits in future work.
\end{enumerate}

\section{Setup and Problem Formulation}
\label{sec:setup}
For an integer $N$, let $[N]=\{1,\cdots, N\}$. Vectors and matrices are denoted by bold lowercase and uppercase letters, respectively. For a matrix $\mathbf{A}$ with rank $r$, we denote its singular values as $\sigma_1(\mathbf{A}) \geq \cdots \geq \sigma_r(\mathbf{A}) \geq 0$. Let $\| \mathbf{A} \|_2 = \sigma_1(\mathbf{A})$ denote the spectral norm and $\| \mathbf{A} \|_F = \sqrt{\sum_{i=1}^r \sigma_i (\mathbf{A})^2}$ the Frobenius norm. For a matrix $\A$, $\A_{:,1:r}$ and $\A_{1:r,:}$ denote the submatrices formed by its first $r$ columns and first $r$ rows, respectively, and $\A_{1:r,1:r}$ denotes its leading $r\times r$ submatrix.
For a tuple of matrices $\TupleB = (\B_{1}, \cdots, \B_{L})$, we let $\| \TupleB \|_2 \triangleq \max_{l \in [L]}  \| \B_{l} \|_2$, $\| \TupleB \|_F \triangleq \sqrt{\sum_l \| \B_{l} \|_F^2}$, and we say that $\TupleB$ has rank $r$ if $\rank(\B_{l}) \leq r$ for every $l \in [L]$. The transpose of a tuple $\TupleB$ is denoted by $\TupleB^\top \triangleq ((\B_{1})^\top, \cdots, (\B_{L})^\top)$. The product and addition of two tuples of matrices $\TupleA$ and $\TupleB$ are defined as $\TupleB \TupleA = (\B_{1}\A_{1}, \cdots,  \B_{L}\A_{L})$ and $\TupleB + \TupleA = (\B_{1} + \A_{1}, \cdots,  \B_{L} + \A_{L})$, respectively. The inner product of two tuples $\TupleA$ and $\TupleB$ is $\langle \TupleA, \TupleB \rangle \triangleq \sum_{l=1}^L \langle \A_{l}, \B_{l} \rangle = \sum_{l=1}^L \mathrm{tr}({(\A_{l})}^\top \B_{l})$. Let $\mathrm{QR}(\cdot)$ denote the reduced QR factorization, and $\mathrm{SVD}(\cdot)$ denote the compact SVD decomposition.

\paragraph{Learning Objective}
We consider an FL system consisting of a federator and $N$ clients, indexed by $i \in [N]$. Client $i$ holds a local dataset $\Dataset_i$, with $|\Dataset_i|>0$ data samples. Define $\Dataset=\bigcup_{i \in [N]} \Dataset_i$.
Let a tuple of matrices $\TupleModel = (\model_{1}, \cdots, \model_{L})$ denote the $L$  linear layers selected for fine-tuning, where for each layer $\model_{l} \in \mathbb{R}^{d_l \times s_l}$. Let $\TupleModel_0$ denote the pre-trained values of $\TupleModel$ and $\TupleUpdate$ be the change of $\TupleModel$ during fine-tuning, such that $\TupleModel = \TupleModel_0 + \TupleUpdate$. For a sample $(\bm{x},y)$, let $f(\TupleModel;\bm{x})$ denote the output of a neural network parameterized by $\TupleModel$, and $\ell(f(\TupleModel;\bm{x}),y)$ be the sample-wise loss function. 
The local empirical loss of client $i$ is defined as
\begin{align*}
\Loss_i^{\text{full}} (\TupleUpdate_i;\Dataset_i) \triangleq \frac{1}{|\Dataset_i|} \sum_{(\bm{x},y) \in \Dataset_i} \ell(f(\TupleModel_0 + \TupleUpdate_i;\bm{x}),y).
\end{align*}
Here, ``full'' refers to an unfactorized update of the selected fine-tuning layers. Let $p_i >0$ with $\sum_{i=1}^{N} p_i = 1$ be the aggregation weight assigned to client $i$. The global learning objective is
\begin{align*}
    \Loss^\text{full}(\TupleUpdate; \Dataset) \triangleq \sum_{i=1}^N p_i \Loss_i^{\text{full}} (\TupleUpdate; \Dataset_i).
\end{align*}

\paragraph{Fine-tuning and LoRA}
In LoRA, $\TupleModel_0$ is frozen, i.e., it does not receive gradient updates during fine-tuning. For each layer $l \in [L]$, LoRA represents the update $\Update_{l}$ as the product of two factor matrices of smaller dimensions $\Update_{l}=\B_{l} \A_{l}$, where $\B_{l} \in \mathbb{R}^{d_l \times r_l}$, and $\A_{l} \in \mathbb{R}^{r_l \times s_l}$, and $r_l \ll \min( d_l,s_l )$.
We denote the tuple of low-rank factor matrices for all the layers as $\TupleB = (\B_{1}, \cdots, \B_{L})$ and $\TupleA = (\A_{1}, \cdots, \A_{L})$, thus, $\TupleUpdate=\TupleB\TupleA$.

The local LoRA objective is the full update objective evaluated under this factorization: 
\begin{align*}
    \Loss_i^{\text{lora}} (\TupleB_i, \TupleA_i; \Dataset_i) \triangleq \Loss_i^{\text{full}} (\TupleB_i\TupleA_i; \Dataset_i).
\end{align*}
Similarly, the global LoRA objective is
\begin{align*}
    \Loss^\text{lora}(\TupleB, \TupleA; \Dataset) \triangleq \sum_{i=1}^N p_i \Loss_i^{\text{lora}} (\TupleB, \TupleA; \Dataset_i) = \sum_{i=1}^N p_i \Loss_i^{\text{full}} (\TupleB \TupleA; \Dataset_i) = \Loss^\text{full}(\TupleB\TupleA; \Dataset).
\end{align*}

\paragraph{Federated LoRA}
Client $i$ uses the local rank profile $\mathbf{r}_i = (r_{i,1},\cdots,r_{i,L})$ with local factors $\TupleB_i =(\B_{i,1},\cdots, \B_{i,L})$ and $\TupleA_i = (\A_{i,1},\cdots, \A_{i,L})$, where $\B_{i,l} \in \mathbb{R}^{d_l \times r_{i,l}}$ and $\A_{i,l} \in \mathbb{R}^{r_{i,l} \times s_l}$. The ranks may differ across clients and layers according to local computational and memory constraints. Define the global rank profile $\mathbf{r}=(r_1,\cdots,r_L)$ as $\displaystyle r_{l} = \max_{i \in [N]} r_{i,l}$ for all $l\in[L]$.

The training proceeds for $T$ communication rounds. At round $t \in \{ 0, \cdots, T-1 \}$, the federator holds the current global adapters $(\TupleB^{(t)},\TupleA^{(t)})$, constructs initial client-specific local adapters $(\TupleB_{i}^{(t)}, \TupleA_{i}^{(t)})$ according to the client's local rank profile, and transmits these factors to client $i$. Client $i$ then performs $E$ local steps of mini-batch stochastic gradient descent (SGD) on its local LoRA objective using step size $\eta^{(t)}>0$. We denote by $(\TupleB_{i}^{(t,e)}, \TupleA_{i}^{(t,e)})$ the local parameters at step $e \in \{0, \cdots, E-1\}$ during communication round $t$. Specifically, at each local step $e$, client $i$ samples a mini-batch of data $\xi_{i}^{(t,e)} \subset \Dataset_i$ and computes the stochastic gradient of the local objective. The local update rules are given by
    $\TupleB_{i}^{(t,e+1)} = \TupleB_{i}^{(t,e)} - \eta^{(t)} \nabla_{\TupleB} \Loss_{i}^{\text{lora}}(\TupleB_{i}^{(t,e)}, \TupleA_{i}^{(t,e)}; \xi_{i}^{(t,e)}),$ 
    $\TupleA_{i}^{(t,e+1)} = \TupleA_{i}^{(t,e)} - \eta^{(t)} \nabla_{\TupleA} \Loss_{i}^{\text{lora}}(\TupleB_{i}^{(t,e)}, \TupleA_{i}^{(t,e)}; \xi_{i}^{(t,e)}),$
where $\TupleB_{i}^{(t,0)}=\TupleB_{i}^{(t)}$ and $\TupleA_{i}^{(t,0)}= \TupleA_{i}^{(t)}$.
After $E$ local steps, clients upload their updates to the federator.

\paragraph{Challenges of LoRA with FedAvg} Upon receiving the clients' local updates, the federator needs to compute a global $\TupleB^{(t+1)}$ and $\TupleA^{(t+1)}$. In rank-homogeneous settings where all clients have the same rank for each layer, the federator could use vanilla FedAvg as in FedIT~\citep{zhang2023fedit}. However, in rank-heterogeneous settings, the adapters have different dimensions and averaging is not directly possible. Existing approaches either pad the adapters to make them addition-compatible or operate in the full model space. The former faces the bilinear mismatch problem and the faithful representation of the clients' ranks, whereas the latter incurs high computational cost at the federator.
The question that remains unanswered in the literature and is answered affirmatively in this work is:
\begin{center}
    \emph{\textbf{Research question:} Is it possible to construct a rank-heterogeneous LoRA-based FL method that mitigates the bilinear mismatch and does not operate in the full model space?}
\end{center}

\section{Our Algorithm and Contributions}
We jointly tackle rank heterogeneity and the bilinear mismatch without operating in the full model space. The main ingredient of our algorithm is operating in a sketched-space of the full model using the linear matrix sketching method introduced by~\citet{tropp2017practical}. 
For ease of presentation, we describe the method for a single layer. The same procedure is applied independently to each layer, and the resulting matrices are collected into tuples. Therefore, we omit the layer index $l$. Each client sketches (compresses) its local update $\Update_i \in \mathbb{R}^{d\times s}$ into two matrices, $\Y_i\in \mathbb{R}^{d\times r}$ and $\Z_i \in \mathbb{R}^{k\times s}$, where $r = r_l = \max_{i\in [N]} r_{i,l}$ is the maximum local rank for each layer and $k$ is a constant satisfying $k>r+1$. By aggregating the $\Y_i$ and $\Z_i$ sent by the clients, the federator  
obtains a sketch of the global update $\Update=\sum_{i \in [N]} p_i \Update_i$, from which it computes a low-rank approximation of $\Update$. Therefore, this maps local adapters of rank $r_i$ to sketches of dimension $r$ and $k$, respectively, thereby enabling linear aggregation and bypassing the bilinear mismatch while avoiding the aggregation in the full $ds$-dimensional space $\Update$. %
This incurs an additional communication of $(r-r_i)d$ and $(k-r_i)s$, which remains low due to the choice of $k$. %
Before presenting our theoretical guarantees and comparing \hetalgname numerically to the state of the art, we first explain the sketching method introduced by~\citet{tropp2017practical}.%

\paragraph{Sketching as Low-Rank Approximation~\citep{tropp2017practical}} 
Suppose $\Update \in \mathbb{R}^{d \times s}$ is an arbitrary matrix. Let $r \ll \min\{ d,s\}$ be the target rank. Given a parameter $k>r+1$, let $\mOme \in \mathbb{R}^{s \times r}$ and $\mPsi \in \mathbb{R}^{k \times d}$ be independent random matrices with entries drawn independently from the standard normal distribution. Two \emph{sketch matrices} of $\Update$ are produced via left and right matrix multiplication:
\begin{align*}
    \Y \triangleq \Update \mOme \quad \text{and} \quad \Z \triangleq \mPsi \Update.
\end{align*}
Given $\Y\in \mathbb{R}^{d \times r}$, and $\Z\in \mathbb{R}^{k \times s}$, the low-rank approximation $\widehat \Update$ of $\Update$ with rank $r$ is computed using the following steps:
\begin{enumerate*}[label=\emph{(\roman*)}]
    \item factorize $\Y$ via QR decomposition: $(\Bh, \sim) \gets \mathrm{QR}(\Y)$;
    \item compute $\mPsi \Bh$ and find its QR decomposition: $(\mathbf{Q}, \mathbf{R}) \gets \mathrm{QR}(\mPsi \Bh)$;
    \item compute $\Ah = \mathbf{R}^{-1} (\mathbf{Q}^T \Z)$; and,
    \item produce $\Bh$ and $\Ah$, which form the low-rank approximation $\widehat{ \Update} = \Bh \Ah$ of $\Update$.
\end{enumerate*}
This sketch-inversion procedure is summarized in \cref{alg:low_rank_app}.

\subsection{\hetalgname: Sketch-aggregated federated LoRA}

For notational simplicity, we present the algorithm under the assumption that all adapted layers share the same input and output dimensions, i.e., $d_l=d$ and $s_l=s$. The method extends directly to heterogeneous dimensions by using compatible sizes.

\paragraph{Initialization} The federator has the pretrained model $\TupleModel_0$. Each client sets a local rank profile\footnote{\hetalgname is agnostic to how the clients adjust their local ranks. They may use adaptive techniques such as those proposed in~\citep{yan2026fedsrd,wu2026adaptive}, adapting the ranks based on their computational and memory resources as well as the contribution of each layer to fine-tuning in terms of effective rank. We note that adaptive LoRA rank allocation has been extensively studied in the centralized LoRA setting, e.g., \citep{zhang2023adalora,valipour2023dylora,ding2023sparse,liu2026flexlora}, and can be used in the federated LoRA setting.} $\mathbf{r}_i = (r_{i,1},\cdots,r_{i,L})$ based on local compute and memory constraints, i.e., $r_{i,l}$ is the maximum rank that client $i$ can allocate to layer $l$. This profile is communicated to the federator. The federator creates a global rank profile $\mathbf{r} = (r_1,\cdots,r_L)$ where $r_l = \max_{i\in [N]}r_{i,l}$ for all $l\in [L]$, and chooses a sketching dimension $k_l>r_l+1$. We further define $r_{\max}\triangleq \max_{l \in [L]} r_l$, and $k_{\max}\triangleq \max_{l \in [L]} k_l$. The rank profiles are fixed throughout training. Following standard practice~\citep{hu2022lora}, the federator initializes the global adapters $\B_l^{(0)} \in \mathbb{R}^{d \times r_l}$ as the zero matrix and $\A_l^{(0)} \in \mathbb{R}^{r_l \times s}$ with independent random Gaussian entries.

In the remainder, without loss of generality, we focus on a single layer $l$ and a single round $t$. 

\paragraph{Federator-side computation} In round $t$, the federator draws two independent random matrices $\mOme^{(t)} \in \mathbb{R}^{s \times r_{\max}}$ and $\mPsi^{(t)} \in \mathbb{R}^{k_{\max} \times d}$ from the standard normal distribution, and communicates them to the clients via a shared random seed. It does the following steps.%

\underline{\emph{Client-specific truncation.}} To reduce the downlink communication, the federator truncates the global adapters according to the client's local rank for each client $i$: 
$$\B_{i,l}^{(t)} = (\B_l^{(t)})_{:,1:r_{i,l}} \quad \text{and} \quad \A_{i,l}^{(t)} = (\A_l^{(t)})_{1:r_{i,l},:},$$ 
and broadcasts $\TupleB_i^{(t)}$ and $\TupleA_i^{(t)}$ to the client $i$. 

\underline{\emph{Global-adapter sketching for usage in next round.}} Federator computes the layer-specific matrices $\mOme_l^{(t)} = \mOme^{(t)}_{:,1:r_l} \in\mathbb R^{s\times r_l}$ and $\mPsi_l^{(t)} = \mPsi^{(t)}_{1:k_l,:} \in\mathbb R^{k_l \times d}$ for all $l \in [L]$, and stores the sketches of the current global update as
$$\Y_{\text{glo},l}^{(t)} = \B_l^{(t)} (\A_l^{(t)} \mOme_l^{(t)}) \quad \text{and} \quad \Z_{\text{glo},l}^{(t)} = (\mPsi_l^{(t)} \B_l^{(t)}) \A_l^{(t)}.$$

\paragraph{Client per-round computation} 
Client $i$ receives a shared seed from which it regenerates $\mOme^{(t)}\in \mathbb{R}^{s \times r_{\max}}$ and $\mPsi^{(t)} \in \mathbb{R}^{k_{\max} \times d}$, and computes the layer-specific matrices $\mOme_l^{(t)}$ and $\mPsi_l^{(t)}$.

\underline{\emph{Local fine-tuning.}} Client $i$ receives the adapters $\TupleB_i^{(t)}$ and $\TupleA_i^{(t)}$ with the correct local dimensions, initiate the local factors as $\B_{i,l}^{(t,0)}=\B_{i,l}^{(t)}$ and $\A_{i,l}^{(t,0)}=\A_{i,l}^{(t)}$, and fine-tunes them for $E$ local iterations to produce $\A_{i,l}^{(t,E)}$ and $\B_{i,l}^{(t,E)}$.

\underline{\emph{Adapter-change sketching.}} Rather than sketching the final local update, client $i$ sketches its change relative to the client-specific initialization, $\Update_{i,l}^{(t,E)}\!-\!\Update_{i,l}^{(t,0)} = \B_{i,l}^{(t,E)} \A_{i,l}^{(t,E)}\!-\!\B_{i,l}^{(t,0)}\A_{i,l}^{(t,0)}$, i.e., 
\begin{align*}
    \Y^{(t)}_{i,l}&= \B_{i,l}^{(t,E)} (\A_{i,l}^{(t,E)} \mOme_l^{(t)})-\B_{i,l}^{(t,0)} (\A_{i,l}^{(t,0)} \mOme_l^{(t)}) \in \mathbb{R}^{d \times r_l}, \\
    \Z^{(t)}_{i,l}&= (\mPsi_l^{(t)} \B_{i,l}^{(t,E)}) \A_{i,l}^{(t,E)} -(\mPsi_l^{(t)} \B_{i,l}^{(t,0)}) \A_{i,l}^{(t,0)} \in \mathbb{R}^{k_l \times s},
\end{align*}
by first computing the multiplication in parentheses to reduce the computation cost. The client sends $\Y_{i,l}^{(t)}$ and $\Z_{i,l}^{(t)}$ to the federator incurring a communication cost of $d r_l + k_l s$ per layer, which is low since $k_l$ can be chosen to be $2 r_l+1$~\citep{tropp2017practical}. For clients with $r_{i,l} \ll r_l$, the clients can opt to communicate the adapters to the federator, who computes the sketches.

\paragraph{Federator aggregation and communication}
Due to the linearity of the sketching operation, aggregating $\Y_{i,l}^{(t)}$ and $\Z_{i,l}^{(t)}$ is implicitly aggregating the client local changes $\Update_{i,l}^{(t,E)}-\Update_{i,l}^{(t,0)}$, but in a smaller subspace. Hence, the federator performs the steps below.

\underline{\emph{Aggregation.}}
The federator adds the aggregated client-change sketches to the global sketches:
$$\Y_l^{(t+1)} = \Y_{\text{glo},l}^{(t)}+ \sum_{i \in [N]} p_i \Y^{(t)}_{i,l} \quad \text{ and } \quad \Z_l^{(t+1)} = \Z_{\text{glo},l}^{(t)}+ \sum_{i \in [N]} p_i \Z^{(t)}_{i,l}.$$
Since clients start from different rank-truncated versions of the global adapter, this updating approach keeps all clients at the same federated-averaging reference point. 

\underline{\emph{Unsketching.}}
The federator uses $\Y_l^{(t+1)}$ and $\Z_l^{(t+1)}$ as inputs for \cref{alg:low_rank_app} to obtain $\Bh_l^{(t+1)}$ and $\Ah_l^{(t+1)}$, i.e.,  $(\Bh_l^{(t+1)}, \Ah_l^{(t+1)}) \gets \Call{Unsketching}{\Y_l^{(t+1)}, \Z_l^{(t+1)}, \mPsi_l^{(t)}}$. Then, it computes $(\U_l^{(t+1)},\bSigma_l^{(t+1)},\V_l^{(t+1)}) \gets \mathrm{SVD}(\Ah_l^{(t+1)})$ and updates the global adapters as
$$\B_l^{(t+1)} = \Bh_l^{(t+1)}\U_l^{(t+1)}(\bSigma_l^{(t+1)})^{\frac{1}{2}} \quad \text{ and } \quad \A_l^{(t+1)} = (\bSigma_l^{(t+1)})^{\frac{1}{2}}(\V_l^{(t+1)})^\top.$$

To initiate the next round, the federator starts with a fresh shared seed for the new random matrices, and sketches and truncates the updated global adapter.
The procedure is summarized in \cref{alg:hetsefora}.

\section{Theoretical Analysis and Insights}
\label{sec:theory}

This section analyzes the proposed method theoretically. We use the tuple notations introduced in \cref{sec:setup} throughout. We first state the smoothness and stochastic gradient assumptions used below. 
\begin{assumption}[Per-sample Lipschitz smoothness]
\label{assump:smooth}
    There exists a $\mu>0$ such that for any $\TupleUpdate$, $\TupleUpdate'$:
    $$ 
    \| \nabla_{\TupleUpdate} \ell(f(\TupleModel_0 + \TupleUpdate;\bm{x}),y) - \nabla_{\TupleUpdate} \ell(f(\TupleModel_0 + \TupleUpdate';\bm{x}),y) \|_F \leq \mu \| \TupleUpdate - \TupleUpdate' \|_F.
    $$
\end{assumption}

\begin{assumption}[Unbiased stochastic gradients with bounded moments]
\label{assump:variance}
    For client $i$ and all $\TupleUpdate \in \prod_{l=1}^L \mathbb{R}^{d \times s}$, the stochastic gradient $\nabla_{\TupleUpdate }\Loss_i^{\text{full}}(\TupleUpdate; \xi_i)$ is an unbiased estimator of $\nabla_{\TupleUpdate} \Loss_i^{\text{full}} (\TupleUpdate)$, i.e., 
    $$
        \mathbb{E}_{\xi_i} \left[ \nabla_{\TupleUpdate }\Loss_i^{\text{full}}(\TupleUpdate; \xi_i) \right] = \nabla_{\TupleUpdate} \Loss_i^{\text{full}} (\TupleUpdate).
    $$ 
    Moreover, there are constants $\chi, \kappa >0$ such that
    $$
        \mathbb{E}_{\xi_i} \left[ \| \nabla_{\TupleUpdate }\Loss_i^{\text{full}}(\TupleUpdate; \xi_i) \|_F^2 \right] \leq \chi^2, \quad \text{and} \quad \mathbb{E}_{\xi_i} \left[ \| \nabla_{\TupleUpdate} \Loss_i^{\text{full}} (\TupleUpdate; \xi_i) \|_F^4 \right] \leq \kappa^4.
    $$

\end{assumption}

\begin{assumption}[Bounded Client Heterogeneity]
\label{assump:heterogeneity}
    There exists a real value $\zeta$ such that for all $\TupleUpdate=\TupleB\TupleA$,
    \begin{align*}
        \sum_{i \in [N]} p_i \| \nabla_{\TupleUpdate }\Loss_i^{\text{full}} (\TupleUpdate) - \nabla_{\TupleUpdate }\Loss^{\text{full}} (\TupleUpdate) \|_F^2 \leq \zeta^2.
    \end{align*}
\end{assumption}

To facilitate the convergence analysis, we adopt the following bounded-factor assumption, which is also used in prior theoretical analyses of federated LoRA~\citep{guo2025selective, chen2026robust}.
\begin{assumption}[Bounded LoRA factors]
\label{assump:normbound}
    There exist constants $M_{\A},M_{\B} >0$ such that, for every layer $l$, round $t$, client $i$, and local step $e$, the spectral norms of the local and global LoRA factors are bounded, i.e., 
    $\| \A_{i,l}^{(t,e)} \|_2 \leq M_{\A},\ \| \B_{i,l}^{(t,e)} \|_2 \leq M_{\B},\ \| \A_l^{(t)} \|_2 \leq M_{\A}, \text{ and } \| \B_l^{(t)} \|_2 \leq M_{\B}.$
\end{assumption}

\cref{lem:low_rank_error} quantifies the error incurred by reconstructing the aggregated update from its sketches.
\begin{lem}[Low-Rank Approximation (Sketching) Error~\citep{tropp2017practical}]
\label{lem:low_rank_error}
    Let $\model \in \mathbb{R}^{d \times s}$, $r\geq 2$, $\varrho \in \{ 0, \cdots,r-2\}$, and $k>r+1$. Draw random matrices $\mOme \in \mathbb{R}^{s \times r}$ and $\mPsi \in \mathbb{R}^{k \times d}$ independently from the standard normal distribution. The rank-$r$ approximation $\widehat{\model}$ satisfies
    \begin{align*}
        \mathbb{E} \| \widehat{\model} - \model \|_F^2 \leq (1+\frac{r}{k-r-1}) \cdot \min_{\varrho < r-1} (1+\frac{\varrho}{r-\varrho-1}) \cdot \tau_{\varrho+1}^2(\model),
    \end{align*}
    where
    $\displaystyle \tau_j^2(\model) \! \triangleq \! \min_{\text{rank}(\model')<j} \| \model \!- \!\model' \|_F^2 \!= \!\sum_{i\geq j} \!\sigma_i^2(\model)$
    is the $j$-th tail energy of $\model$.
\end{lem}
For a tuple of layerwise updates, \cref{lem:low_rank_error} is applied independently to each layer.

Let $\mathcal F_t$ be the algorithmic history at the beginning of round $t$, and
write $\mathbb E_t[\cdot]=\mathbb E[\cdot\mid\mathcal F_t]$. Define constants $C_1 \triangleq M_{\B}^2 + M_{\A}^2$, $C_2 \triangleq \chi M_{\A} M_{\B} \kappa^2 + 2\mu (M_{\A}^4+M_{\B}^4) \chi^2$, $C_3 \triangleq \chi^2 \Big(\chi^2 (M_\A^2+ M_\B^2) + 2\mu^2 (M_\A^2+ M_\B^2) (M_\B^4 + M_\A^4)\Big)$, and $C_4 \triangleq \mu M_{\A}^2 M_{\B}^2 \kappa^4$. 
Define the sketching error as $\mathcal E_{\text{sketch}}^{(t+1)}\triangleq \TupleUpdate_{\hetalgname}^{(t+1)}-\mathbf{\mathtt{V}}^{(t+1)}$, where $\mathbf{\mathtt{V}}^{(t+1)} \triangleq \TupleUpdate^{(t)} + \sum_{i \in [N]} p_i (\TupleUpdate_{i,r_i}^{(t,E)}-\TupleUpdate_{i,r_i}^{(t,0)})$. Let $q_{\min} \triangleq \min_{l \in [L]} \sum_{i: r_{i,l}=r_l} p_i$, the sketching floor $S_1^{(t+1)} \triangleq \chi \cdot \mathbb{E}_t \Big\| \mathcal{E}_{\text{sketch}}^{(t+1)} \Big\|_F + 2\mu \mathbb{E}_t \Big\|\mathcal{E}_{\text{sketch}}^{(t+1)} \Big\|_F^2$, and $\widebar{S_1} \triangleq \frac{1}{T} \sum_{t=0}^{T-1} \mathbb{E} \left[ S_1^{(t+1)} \right]$.
\cref{thm:hetero_thm} shows that \hetalgname converges to a neighborhood of a first-order stationary point under \cref{assump:smooth,assump:variance,assump:heterogeneity,assump:normbound}.

\begin{thm}
\label{thm:hetero_thm}
Suppose \cref{assump:smooth,assump:variance,assump:heterogeneity,assump:normbound} hold. Let the global minimum of the unregularized loss be $\Loss^{\text{full},\star} \triangleq \min_{\TupleUpdate} \Loss^{\text{full}}(\TupleUpdate)$ and $\Delta_0 \triangleq \mathbb{E} \left[ \Loss^{\text{full}} (\TupleUpdate^{(0)}) \right] - \Loss^{\text{full}, \star}$ and take a constant learning rate $\eta^{(t)}=\eta$. Let $\sigma_{l,j}^{(t)}$ be the $j$-th singular value of $\Update_l^{(t)}$. \hetalgname satisfies:
\begin{align*}
    & \frac{1}{T} \sum_{t=0}^{T-1} \left( \mathbb{E} \left[ \left\| \nabla_{\TupleA} \Loss^{\text{lora}} (\TupleB^{(t)}, \TupleA^{(t)}) \right\|_F^2 \right] + \mathbb{E} \left[ \left\| \nabla_{\TupleB} \Loss^{\text{lora}} (\TupleB^{(t)}, \TupleA^{(t)}) \right\|_F^2 \right] \right) \leq \frac{2\Delta_0}{\eta E T q_{\min}} + \underbrace{\frac{2\widebar{S_1}}{\eta E q_{\min}}}_{\text{sketching error}} \\
    & \!\!+ \!\frac{2C_2\eta E }{q_{\min}} \!+\! \frac{2C_3 \eta^2 E^2}{q_{\min}^2} \! + \! \frac{4C_4\eta^3 E^3}{q_{\min}} \! + \! \underbrace{\frac{6C_1}{2q_{\min}^2}  \frac{\mu^2 }{T} \sum_{t=0}^{T-1} \!\mathbb{E} \Big[ \sum_{i \in [N]} p_i \sum_{l \in [L]} \sum_{j=r_{i,l}+1}^{r_l} \!\!\big( \sigma_{l,j}^{(t)} \big)^2 \Big]}_{\text{clients local rank truncation}} \!+ \!\!\!\!\! \underbrace{\frac{6C_1}{2q_{\min}^2} \zeta^2}_{\text{client heterogeneity}} \!\!.
\end{align*}
\end{thm}

The proof is in \cref{proof:sefora}.
\cref{thm:hetero_thm} separates the optimization error from the three error sources: sketching, truncation to clients' local ranks, and client heterogeneity. The dependence on $q_{\min}$ indicates that the bound grows when clients with the maximum layer-wise rank have smaller aggregation weights. Under the spectral-tail condition in \cref{assump:tail}, we present \cref{cor:tail} in which we select a constant learning rate that balances the optimization and sketching errors.

\begin{assumption}
\label{assump:tail}
    Let $\mathbf{\mathtt{V}}^{(t+1)} \triangleq \TupleUpdate^{(t)} + \sum_{i \in [N]} p_i (\TupleUpdate_{i,r_i}^{(t,E)}-\TupleUpdate_{i,r_i}^{(t,0)})$. For all $l \in [L]$, there exist a fixed number $\varrho_l \in \{ 0, \cdots,r_l-2\}$ and a constant $\widebar{\tau_l} \geq 0$ such that, for every $T \geq 1$,
    \begin{align*}
        \frac{1}{T} \sum_{t=0}^{T-1} \mathbb{E} \left[ \tau_{\varrho_l+1}^2(\mathbf{V}_l^{(t+1)}) \right] \leq \widebar{\tau_l}^2.
    \end{align*}
\end{assumption}

\begin{corollary}
\label{cor:tail}
    Suppose \cref{assump:smooth,assump:variance,assump:heterogeneity,assump:normbound,assump:tail} hold, and $\Delta_0 >0$. Define $c_l \triangleq (1+\frac{r_l}{k_l-r_l-1}) \cdot (1+\frac{\varrho_l}{r_l-\varrho_l-1})$ and $\epsilon^2 \triangleq \sum_{l=1}^L c_l \cdot \widebar{\tau_l}^2$. Choose $\gamma_T \triangleq \sqrt{\frac{\frac{\Delta_0}{T} + \chi \epsilon + 2\mu \epsilon^2}{C_2}}$ and $\eta = \frac{\gamma_T}{E}$, \hetalgname satisfies:
    \begin{align*}
        & \frac{1}{T} \sum_{t=0}^{T-1} \left( \mathbb{E} \left[ \left\| \nabla_{\TupleA} \Loss^{\text{lora}} (\TupleB^{(t)}, \TupleA^{(t)}) \right\|_F^2 \right] + \mathbb{E} \left[ \left\| \nabla_{\TupleB} \Loss^{\text{lora}} (\TupleB^{(t)}, \TupleA^{(t)}) \right\|_F^2 \right] \right) \\
        & \leq \frac{4C_2}{q_{\min}} \gamma_T + \frac{2C_3}{q_{\min}^2} \gamma_T^2 +\frac{4C_4}{q_{\min}} \gamma_T^{3} + \frac{6C_1}{2q_{\min}^2} \bigg( \mu^2 \frac{1}{T} \sum_{t=0}^{T-1} \mathbb{E} \Big[ \sum_{i \in [N]} p_i \sum_{l \in [L]} \sum_{j=r_{i,l}+1}^{r_l} \big( \sigma_{l,j}^{(t)} \big)^2 \Big] + \zeta^2 \bigg).
    \end{align*}
\end{corollary}
The proof is in \cref{app:cor}. \cref{cor:tail} shows that \hetalgname converges with the rate $\cO(T^{-1/2})$ through $\gamma_T$ when the energy tail is bounded. Bounding the energy tail amounts to assuming a good sketching approximation, which we expect to hold in regimes where $\Update$ is indeed of low rank.

\section{Experiments}
We conduct experiments across language understanding, mathematical reasoning, and image classification tasks, to evaluate the performance of \hetalgname in the rank-heterogeneous setting.

\paragraph{Models and datasets}
For language understanding, we fine-tune RoBERTa-Large~\citep{liu2019roberta} on SST-2, QNLI, QQP, and MNLI from the General Language Understanding Evaluation (GLUE) benchmark~\citep{wang2018glue}. 
For mathematical reasoning, we focus on GSM8K~\citep{cobbe2021training} using Llama 3.2 3B~\citep{grattafiori2024llama} and Qwen 2.5 3B~\citep{2024qwen2}.
For image classification, we use ViT~\citep{dosovitskiy2020image} model on CIFAR-100~\citep{krizhevsky2009learning}, Tiny ImageNet~\citep{le2015tiny}, Food-101~\citep{bossard14}, and DomainNet~\citep{peng2019moment}. We use 40 clients for GLUE, 20 for GSM8K, and 30 for the vision tasks, with full client participation. We show the stability of \hetalgname under the client's partial participation in \cref{fig:partial_participation} of \cref{app:add_results}.

\paragraph{Data heterogeneity}
For GLUE and the image datasets except DomainNet, we simulate data heterogeneity by partitioning the training data class-wise according to a Dirichlet distribution. The parameter $\alpha$ controls the level of heterogeneity, with smaller values corresponding to higher heterogeneity. We use $\alpha=0.1$ unless otherwise specified. Since GSM8K has no class labels, we instead use a Dirichlet distribution to vary the number of training samples assigned to each client.

DomainNet contains images from $6$ distinct domains, i.e., clipart, infograph, painting, quickdraw, real, and sketch, providing a naturally heterogeneous setting. We divide the $30$ clients into $6$ groups of five clients, with each group containing only images from a particular domain. Within each group, the data is distributed according to the Dirichlet distribution with parameters $\alpha=0.1$ and $0.5$. Consequently, DomainNet combines inter-group domain heterogeneity and intra-group label heterogeneity.

\paragraph{Rank heterogeneity} We introduce rank heterogeneity across both clients and layers to simulate heterogeneous computation and memory budgets. The layer-wise rank profile assigns smaller ranks to earlier layers and gradually increases the rank towards the maximum in later layers. To model client-wise resource heterogeneity, we divide the clients into several tiers, each obtained by scaling the layer-wise rank profile by a different factor. In particular, for the GLUE experiments, we set $r_{\max}=16$, and use the three scaling factors $0.5, 0.75$, and $1.0$. For DomainNet, $r_{\max}=16$, a five-tier profile is used, with scaling factors of $0.5, 0.625, 0.75, 0.875$, and $1.0$, ensuring that each domain contains one client from each tier. The extreme vision experiment uses three equally sized client groups with $r_{\max}=256$ and scaling factors $\frac{1}{32}, \frac{1}{16}$, and $1.0$. In \hetalgname, we choose the sketch parameter $k_l=r_l+\max \{ 2, \lceil \frac{r_l}{8} \rceil \}$. The detailed configuration is reported in \cref{fig:rank_profile} in \cref{app:imple_details}.

\paragraph{Baselines} We compare \hetalgname with the three most relevant approaches: FlexLoRA, FSLoRA, and FedHL. FlexLoRA reconstructs and aggregates the client's full update from the LoRA adapters, thereby eliminating the bilinear mismatch at the cost of dense reconstruction at the server side. FSLoRA communicates selected components of fixed-rank local adapters, allowing clients to use heterogeneous local ranks without reconstructing dense updates; however, its aggregation does not eliminate the bilinear mismatch. FedHL employs error compensation together with dynamically adjusted aggregation weights to mitigate truncation-induced biases.

The implementation details are provided in \cref{app:imple_details}. We repeat each experiment using three random seeds, and report the mean accuracy and standard deviation. The best and second-best results are shown in bold and underline, respectively. \cref{app:add_results} contains additional experiments, ablation studies, and the time spent on selected experiments.

\subsection{Experiment Results}
\paragraph{Domain and label heterogeneity} We first evaluate the methods on DomainNet, which presents both domain and label heterogeneity. %
We report the results in \cref{tab:domainnet}. %
\hetalgname achieves the highest average accuracy in both settings, achieving results competitive with FSLoRA and consistently outperforming FlexLoRA and FedHL. These results demonstrate that \hetalgname outperforms the other methods when rank heterogeneity is combined with structured domain and label heterogeneity. 

\begin{table}[!h]
\centering
\scriptsize
\caption{Accuracy ($\%$) on DomainNet using ViT, with $r_{\max}=16$.} 
\begin{tabular}{c|c|cccccc|c}
     & methods & clipart & infograph & painting & quickdraw & real & sketch & average \\ \hline
    \multirow{4}{*}{$\alpha=0.1$} & FlexLoRA & $63.0 \pm 0.5$ & $31.3 \pm 0.4$ & $61.1 \pm 0.3$ & $32.6 \pm 0.2$ & $78.2 \pm 0.2$ & $53.9 \pm 0.5$ & $53.4 \pm 0.3$\\
    & FSLoRA & $\underline{66.7 \pm 0.2}$ & $\mathbf{32.9 \pm 0.4}$ & $\underline{63.4 \pm 0.5}$ & $\underline{38.3 \pm 0.4}$ & $\mathbf{80.4 \pm 0.3}$ & $\mathbf{56.9 \pm 0.1}$ & $\underline{56.4 \pm 0.3}$ \\
    & FedHL & $52.1 \pm 1.0$ & $25.3 \pm 0.4$ & $49.6 \pm 1.0$ & $21.4 \pm 1.0$ & $59.9 \pm 0.3$ & $44.3 \pm 0.9$ & $42.1 \pm 0.7$ \\ 
    & \hetalgname & $\mathbf{67.0 \pm 0.2}$ & $\underline{32.8 \pm 0.2}$ & $\mathbf{63.6 \pm 0.4}$ & $\mathbf{38.9 \pm 0.2}$ & $\underline{80.2 \pm 0.3}$ & $\underline{56.8 \pm 0.1}$ & $\mathbf{56.6 \pm 0.1}$ \\ [1pt] \hline 
    \multirow{4}{*}{$\alpha=0.5$} & FlexLoRA & $63.4 \pm 0.2$ & $32.0 \pm 0.3$ & $61.7 \pm 0.4$ & $33.5 \pm 0.1$ & $78.6 \pm 0.3$ & $54.3 \pm 0.3$ & $53.9 \pm 0.1$ \\
    & FSLoRA & $\underline{67.5 \pm 0.3}$ & $\underline{33.4 \pm 0.3}$ & $\mathbf{64.1 \pm 0.2}$ & $\underline{39.2 \pm 0.4}$ & $\underline{80.8 \pm 0.2}$ & $\underline{57.5 \pm 0.4}$ & $\underline{57.1 \pm 0.2}$ \\
    & FedHL & $64.0 \pm 0.3$ & $31.1 \pm 0.6$ & $59.2 \pm 1.3$ & $27.5 \pm 0.6$ & $72.0 \pm 1.6$ & $53.3 \pm 1.4$ & $51.2 \pm 0.9$ \\ 
    & \hetalgname & $\mathbf{67.6 \pm 0.3}$ & $\mathbf{33.8 \pm 0.1}$ & $\underline{64.0 \pm 0.3}$ & $\mathbf{41.2 \pm 0.3}$ & $\mathbf{80.8 \pm 0.1}$ & $\mathbf{57.7 \pm 0.1}$ & $\mathbf{57.5 \pm 0.1}$ \\
\end{tabular}
\label{tab:domainnet}
\end{table}

\paragraph{Mathematical reasoning and language understanding} We report the accuracy on GSM8K using Llama 3.2 3B and Qwen 2.5 3B in \cref{tab:gsm8k-glue}. On Llama, \hetalgname achieves $36.2\%$, closely matching the $36.3 \%$ obtained by FSLoRA, and outperforming FlexLoRA and FedHL. On Qwen, \hetalgname achieves the highest accuracy of $56.1 \%$. The results across the two models show that \hetalgname remains competitive with the strongest baselines on large language models. The GLUE results are reported in \cref{tab:gsm8k-glue}, showing that \hetalgname remains competitive with FSLoRA across the four language understanding tasks while avoiding aggregation with bilinear mismatch.

\begin{table}[!h]
    \centering
    \caption{Accuracy ($\%$) on GSM8K and the GLUE tasks, all with $r_{\max}=16$.}
    \resizebox{0.9\textwidth}{!}{
    \begin{tabular}{c|cc||cccc}
         & Llama & Qwen & SST-2 & QNLI & QQP & MNLI \\ \hline
        FlexLoRA & $34.5 \pm 0.8$ & $\underline{56.0 \pm 0.4}$ & $\mathbf{95.3 \pm 0.3}$ & $90.4 \pm 1.3$ & $70.6 \pm 12.4$ & $\underline{85.1 \pm 2.5}$ \\
        FSLoRA & $\mathbf{36.3 \pm 1.1}$ & $55.7 \pm 0.8$ & $\underline{95.0 \pm 0.1}$ & $\mathbf{90.9 \pm 1.1}$ & $70.9 \pm 13.5$ & $\mathbf{87.0 \pm 1.0}$ \\
        FedHL & $21.5 \pm 4.6$ & $55.5 \pm 0.7$ & $93.8 \pm 1.6$ & $82.2 \pm 5.9$ & $\mathbf{75.4 \pm 7.7}$ & $77.0 \pm 10.6$ \\
        \hetalgname & $\underline{36.2 \pm 0.2}$ & $\mathbf{56.1 \pm 0.3}$ & $94.5 \pm 1.0$ & $\underline{90.7 \pm 1.7}$ & $\underline{74.8 \pm 10.5}$ & $84.3 \pm 2.8$ \\
    \end{tabular}
    }
    \label{tab:gsm8k-glue}
\end{table}

\paragraph{Extreme rank heterogeneity} We consider a more extreme rank-heterogeneous setting, in which $30$ clients are divided equally with maximum ranks $8,16$ and $256$. The rank distribution is illustrated in \cref{fig:rank_profile}. As shown in \cref{tab:vitextreme}, \hetalgname consistently outperforms FSLoRA, with improvements of $0.8\%$, $1.0\%$ and $0.7\%$ on CIFAR-100, Tiny ImageNet, and Food-101. In comparison, under the standard rank-heterogeneous profile reported in \cref{tab:vitstandard} of \cref{app:addi_exp}, the two methods differ by at most $0.3\%$. This shows the advantage of \hetalgname under extreme rank heterogeneity.

FlexLoRA achieves the highest accuracy; however, the federator reconstructs the dense update $\Update_i$ for $i \in [N]$, incurring higher computational cost. In contrast, \hetalgname performs linear aggregation in the sketch space and approaches the accuracy of FlexLoRA without computing the dense updates $\Update$. The fixed-size sketches may incur higher uplink communication costs for clients with small ranks. To reduce this overhead, \hetalgname can employ a flexible upload scheme in which client $i$ transmits either $\Y_i$ and $\Z_i$ or $\B_i$ and $\A_i$, whichever is smaller. If the client uploads the LoRA factors, the server constructs the corresponding sketches. In this way, the communication for low-rank clients remains small while retaining sketch-space aggregation.

\begin{table}[!h]
    \centering
    \caption{Accuracy ($\%$) on CIFAR-100, Tiny ImageNet, Food-101 under extreme rank heterogeneity.}%
    \resizebox{0.52\textwidth}{!}{
    \begin{tabular}{c|ccc}
        & CIFAR-100 & Tiny ImageNet & Food-101 \\ \hline
        FlexLoRA & $\mathbf{91.0 \pm 0.3}$ & $\mathbf{88.0 \pm 0.2}$ & $\mathbf{84.3 \pm 0.5}$ \\
        FSLoRA & $88.9 \pm 0.3$ & $86.3 \pm 0.2$ & $83.3 \pm 0.9$ \\
        FedHL & $85.0 \pm 0.6$ & $82.2 \pm 0.3$ & $76.8 \pm 3.2$ \\
        \hetalgname & $\underline{89.7 \pm 0.7}$ & $\underline{87.3 \pm 0.3}$ & $\underline{84.0 \pm 1.1}$ \\
    \end{tabular}
    }
    \label{tab:vitextreme}
\end{table}

\section{Conclusion}
We propose \hetalgname, a federated LoRA method that addresses the rank heterogeneity challenge without facing the bilinear mismatch problem nor aggregating in the full $ds$-dimensional model space. \hetalgname constructs fixed-dimensional linear sketches and aggregates client updates in a common sketch space. Numerical experiments show that \hetalgname performs effectively across various rank- and data-heterogeneous settings and achieves competitive performance relative to the baselines. 

The limitations of \hetalgname are a sketching error and increased communication costs for clients with lower ranks. The communication overhead for those clients can be mitigated by transmitting the adapters and requiring the federator to compute the corresponding sketches. Future work can study novel sketching methods that balance communication costs for clients without additional computation at the federator. Furthermore, the combination of differential privacy with \hetalgname remains an interesting future work that we will explore.

{
\small
\bibliographystyle{unsrtnat}  
\bibliography{refs}
}

\clearpage

\appendix
\etocdepthtag.toc{appendixmatter}

\begingroup
    \etocsettagdepth{mainmatter}{none}

    \etocsettagdepth{appendixmatter}{subsubsection}
    \etocsetnexttocdepth{subsubsection}

    \etocsettocstyle
        {\section*{\centering Appendix}\vspace{0.5em}}
        {}

    \tableofcontents
\endgroup

\section{Algorithms}
\label{sec:sefora_app}
The algorithmic representation of \hetalgname is summarized in \cref{alg:hetsefora}. 
The sketch-inversion algorithm is described in \cref{alg:low_rank_app}.

\begin{algorithm}[!htbp]
\caption{\hetalgname} %
\label{alg:hetsefora}
\footnotesize
\begin{algorithmic}[1]
\Require The pretrained model $\model_0$. Total rounds $T$ and local steps $E$. In total, $N$ clients, each client $i$ holding their local dataset $\Dataset_i$. The aggregation weights $p_i >0$, with $\sum_{i=1}^{N} p_i = 1$. The local rank profiles $\mathbf{r}_i = (r_{i,1},\cdots,r_{i,L})$. For $l \in [L]$, set $r_l \triangleq \max_{i\in [N]}r_{i,l}$ and $k_l > r_l+1$. Set $r_{\max}\triangleq \max_{l \in [L]} r_l$, and $k_{\max}\triangleq \max_{l \in [L]} k_l$. Initialize global matrices $\B_l^{(0)} \in \mathbb{R}^{d \times r_l}$ as a full zero matrix and $\A_l^{(0)} \in \mathbb{R}^{r_l \times s}$ as a random Gaussian matrix. $\TupleB^{(0)} = (\B_1^{(0)}, \cdots, \B_L^{(0)})$ and $\TupleA^{(0)} = ( \A_1^{(0)}, \cdots, \A_L^{(0)})$. A function $\Call{Unsketching}{\Y,\Z, \mPsi}$ that unsketch the sketch matrices (\cref{alg:low_rank_app}).

\For{round $t = 0, 1, \cdots, T-1$}
\State Federator initializes $\mPsi^{(t)} \in \mathbb{R}^{k_{\max} \times d} $ and $\mOme^{(t)} \in \mathbb{R}^{s \times r_{\max}}$ from a standard normal distribution, and send the shared seed that generates $\mPsi^{(t)}$ and $\mOme^{(t)}$ to all clients. 
\State Federator extracts $\mOme_l^{(t)} \!\!= \!\mOme^{(t)}_{:,1:r_l} \!\in\!\mathbb R^{s\times r_l}$ and $\mPsi_l^{(t)} \!\!= \!\mPsi^{(t)}_{1:k_l,:} \!\in\!\mathbb R^{k_l \times d}$, and computes $\Y_{\text{glo},l}^{(t)} = \B_l^{(t)} (\A_l^{(t)} \mOme_l^{(t)})$, and $\Z_{\text{glo},l}^{(t)} = (\mPsi_l^{(t)} \B_l^{(t)}) \A_l^{(t)}$ for all $l \in [L]$.
\State For all clients $i$ and layer $l$, federator computes $\B_{i,l}^{(t)} = (\B_l^{(t)})_{:,1:r_{i,l}}$ and $\A_{i,l}^{(t)} = (\A_l^{(t)})_{1:r_{i,l},:}$, and broadcasts $\TupleB_i^{(t)}$ and $\TupleA_i^{(t)}$ to the client $i$.
\ForAll{clients $i$ \textbf{in parallel}}
    \ForAll{adapter layer $l \in [L]$}
        \State Regenerate $\mOme_l^{(t)} \!\!= \!\mOme^{(t)}_{:,1:r_l} \!\in\!\mathbb R^{s\times r_l}$ and $\mPsi_l^{(t)} \!\!= \!\mPsi^{(t)}_{1:k_l,:} \!\in\!\mathbb R^{k_l \times d}$ from the shared seed.
        \State Receive $\B_{i,l}^{(t)}$ and $\A_{i,l}^{(t)}$.
\EndFor
\State Using $\B_{i,l}^{(t,0)}=\B_{i,l}^{(t)}$ and $\A_{i,l}^{(t,0)}=\A_{i,l}^{(t)}$, perform $E$ local steps to obtain $\B_{i,l}^{(t,E)}$ and $\A_{i,l}^{(t,E)}$ for all $l \in [L]$.
\ForAll{adapter layer $l \in [L]$}
\State Compute the sketch matrices
\vspace{-0.1cm}
    \begin{align*}
        \Y^{(t)}_{i,l}&= \B_{i,l}^{(t,E)} (\A_{i,l}^{(t,E)} \mOme_l^{(t)})-\B_{i,l}^{(t,0)} (\A_{i,l}^{(t,0)} \mOme_l^{(t)}) \in \mathbb{R}^{d \times r_l}, \\
        \Z^{(t)}_{i,l}&= (\mPsi_l^{(t)} \B_{i,l}^{(t,E)}) \A_{i,l}^{(t,E)} -(\mPsi_l^{(t)} \B_{i,l}^{(t,0)}) \A_{i,l}^{(t,0)} \in \mathbb{R}^{k_l \times s}.
    \end{align*}
    \EndFor
    \State Send $\TupleY^{(t)}_{i} = ( \Y^{(t)}_{i,1}, \cdots, \Y^{(t)}_{i,L} )$ and $\TupleZ^{(t)}_{i} = ( \Z^{(t)}_{i,1}, \cdots, \Z^{(t)}_{i,L} )$ to federator.
    \EndFor
    \ForAll{adapter layer $l \in [L]$}
    \State Federator aggregates $$\Y_l^{(t+1)} = \Y_{\text{glo},l}^{(t)}+ \sum_{i \in [N]} p_i \Y^{(t)}_{i,l} \; \text{ and }\;
            \Z_l^{(t+1)} = \Z_{\text{glo},l}^{(t)}+ \sum_{i \in [N]} p_i \Z^{(t)}_{i,l}.$$
    \State $(\Bh_l^{(t+1)}, \Ah_l^{(t+1)}) \gets \Call{Unsketching}{\Y_l^{(t+1)}, \Z_l^{(t+1)}, \mPsi_l^{(t)}}$
    \State Federator computes $(\U_l^{(t+1)},\bSigma_l^{(t+1)},\V_l^{(t+1)}) \gets \mathrm{SVD}(\Ah_l^{(t+1)})$, and 
    $$\B_l^{(t+1)} = \Bh_l^{(t+1)}\U_l^{(t+1)}(\bSigma_l^{(t+1)})^{\frac{1}{2}} \quad \text{and} \quad \A_l^{(t+1)} = (\bSigma_l^{(t+1)})^{\frac{1}{2}}(\V_l^{(t+1)})^\top.$$
    \EndFor
\EndFor
    \State \textbf{Return} $(\TupleB^{(T)}, \TupleA^{(T)})$
\end{algorithmic}
\end{algorithm}

\begin{algorithm}[!htbp]
\caption{Sketch-inversion Algorithm}
\begin{algorithmic}[1]
\label{alg:low_rank_app}
\footnotesize
\Require Given the random matrix $\mPsi \in \mathbb{R}^{k \times d} $ drawn from a standard normal distribution, and two sketch matrices $\Y \in \mathbb{R}^{d \times r}$ and $\Z \in \mathbb{R}^{k \times s}$, produced from the target large matrix $\model \in \mathbb{R}^{d \times s}$, i.e., $\Y=\model \mOme$ and $\Z=\mPsi\model$.
\Ensure Return the low-rank factors $\Bh \in \mathbb{R}^{d \times r}$ and $\Ah \in \mathbb{R}^{r \times s}$ of a rank-$r$ approximation of $\model$.
\Function{Unsketching}{$\Y,\Z, \mPsi$}
\State $(\Bh, \sim) \gets \mathrm{QR}(\Y)$
\State $(\mathbf{Q}, \mathbf{R}) \gets \mathrm{QR}(\mPsi \Bh)$
\State $\Ah = \mathbf{R}^{-1} (\mathbf{Q}^T \Z)$
\State \textbf{Return} $(\Bh, \Ah)$ 
\EndFunction
\end{algorithmic}
\end{algorithm}

\section{Implementation Details}
\label[appendix]{app:imple_details}
\paragraph{Model configuration} For RoBERTa-Large, the LoRA adapters are inserted into the query, key, and value projection matrices of every self-attention layer for RoBERTa-Large. For Llama 3.2 3B, Qwen 2.5 3B, and ViT, the LoRA adapters are inserted into the query and value matrices. The pretrained backbone parameters are frozen during federated fine-tuning. For classification tasks, the adapters and the classification head are optimized, while only the adapters are optimized for GSM8K. The adapter contribution is scaled by $0.1$. 

\paragraph{Federated optimization} Unless otherwise stated, all clients participate in every communication round, with weight $p_i=\frac{|\Dataset_i|}{\sum_{j=1}^N|\Dataset_j|}$. Each client performs $E =5$ local optimization steps per communication round. At each local step, clients randomly sample a mini-batch from their local training dataset, with batch sizes specified in \cref{tab:details1} and \cref{tab:details2}. We use Adam as the optimizer, and sweep the learning rate over the values reported in \cref{tab:details1} and \cref{tab:details2}. The total number of communication rounds for each task is also reported in the tables. The clients' rank distribution is illustrated in \cref{fig:rank_profile}.

\paragraph{Computation resources} Experiments were conducted on machines equipped with NVIDIA RTX PRO 5000 Blackwell, NVIDIA RTX PRO 4500 Blackwell, or NVIDIA GeForce RTX 4090 GPUs. Each experimental run used a single GPU, with the GPU model determined by resource availability.

\begin{table}[!h]
    \centering
    \caption{The detailed configuration for the language understanding and mathematical reasoning tasks.}
    \setlength{\tabcolsep}{4pt}
    \begin{tabular}{c|c|c}
         & GLUE & GSM8K \\ \hline
        Model & RoBERTa-Large & Llama 3.2 3B and Qwen 2.5 3B \\ \hline
        Clients & 40 & 20 \\ \hline
        Batch Size & 64 & 4 \\ \hline
        Maximum length & 128 & 512\\ \hline
        Communication rounds & 400 & 400 for Llama, 200 for Qwen \\ \hline
        Local steps & 5 & 5 \\ \hline
        Learning rates & \scriptsize{$\{ 5e-4,1e-3,2e-3,3e-3,4e-3 \}$} & \scriptsize{$\{ 1e-5,2e-5,3e-4,7e-4,1e-3,2e-3 \}$} \\ \hline
        Adapting layers & 24 & 28 for Llama, 36 for Qwen \\ \hline
        Metric & Accuracy & Exact-match accuracy\\
    \end{tabular}
    \label{tab:details1}
\end{table}

\begin{table}[!h]
    \centering
    \caption{The detailed configuration for the image classification tasks.}
    \begin{tabular}{c|c}
         & CIFAR-100/Tiny ImageNet/Food-101/DomainNet \\ \hline
        Model & ViT-Base\\ \hline
        Clients & 30\\ \hline
        Batch Size & 32\\ \hline
        Communication rounds & 300 \\ \hline
        Local steps & 5 \\ \hline
        Learning rates & $\{ 1e-3,3e-3,5e-3,7e-3,9e-3,1e-2,2e-2 \}$ \\ \hline
        Adapting layers & 12 \\ \hline
        Metric & Accuracy \\
    \end{tabular}
    \label{tab:details2}
\end{table}

\begin{figure}
    \centering
    \includegraphics[width=1\linewidth]{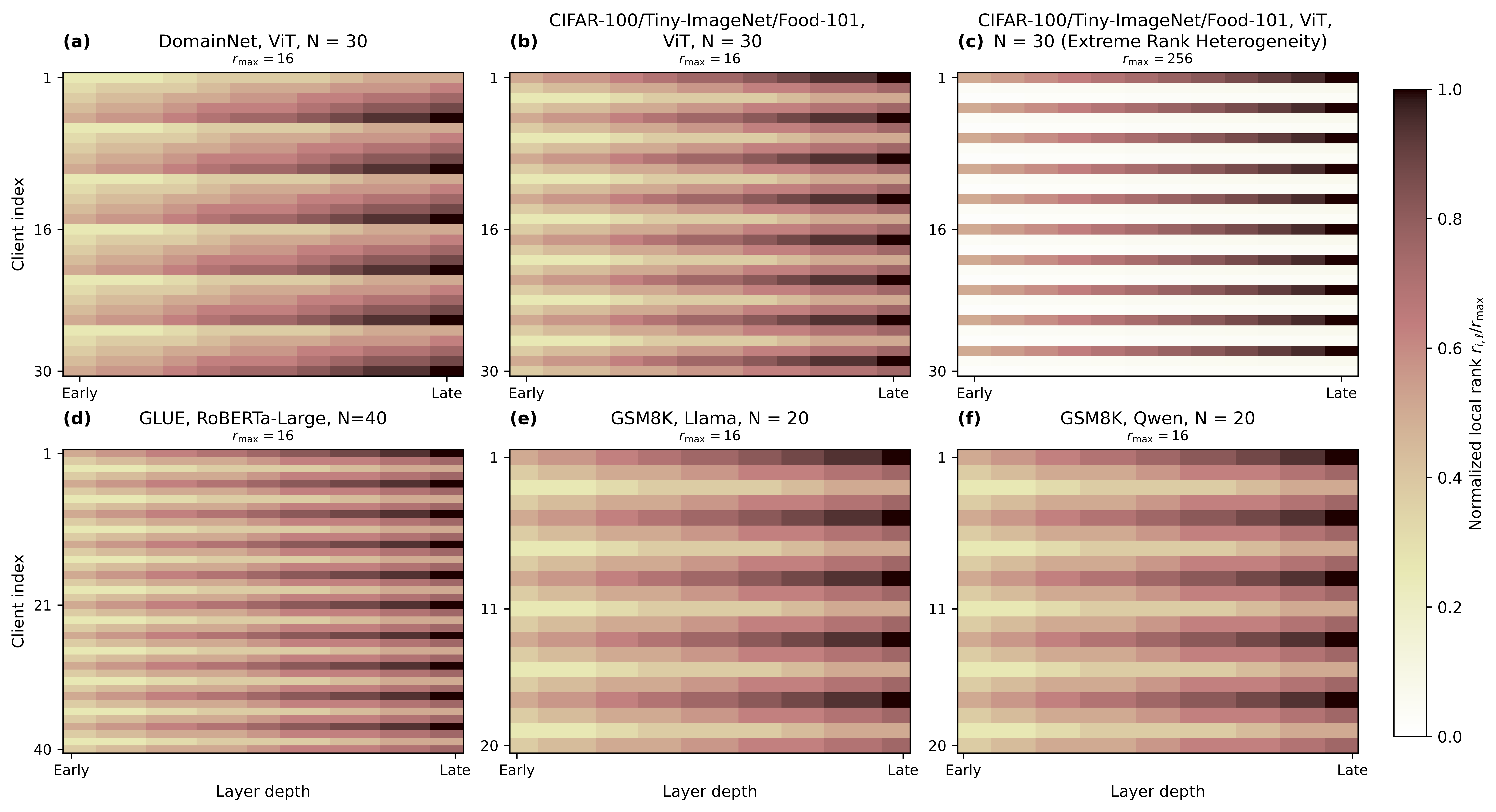}
    \caption{We report the client-wise and layer-wise heterogeneous rank profile used in the experiments. Rows are the clients, columns correspond to the adapted layers from early to late stage, and colors indicate ranks normalized by the maximum rank $r_{\max}$. Specifically, $(a)$ is used in the experiment producing \cref{tab:domainnet}, $(b)$ is used for \cref{tab:vitstandard} and \cref{fig:partial_participation}, \cref{fig:numclients} and \cref{fig:oversampling}, $(c)$ is used for \cref{tab:vitextreme}, $(d)$, $(e)$ and $(f)$ are used to produce \cref{tab:gsm8k-glue}.}
    \label{fig:rank_profile}
\end{figure}

\section{Additional Experiment Results}
\label[appendix]{app:add_results}
\paragraph{Standard rank heterogeneity}
\cref{tab:vitstandard} shows the image classification results under the standard three-tier rank profile, as illustrated in subplot $(b)$ in \cref{fig:rank_profile}. \hetalgname obtains the highest accuracy on Tiny ImageNet, and comparable accuracies on CIFAR-100 and Food-101. Together with the extreme rank-heterogeneous results in \cref{tab:vitextreme}, and the results with varying maximum adapter rank, shown in \cref{fig:impactrank}, these results show that \hetalgname remains competitive under standard rank heterogeneity, while its advantage over FSLoRA becomes more significant when the clients' rank profiles differ substantially.

\begin{table}[!h]
    \centering
    \caption{Accuracy ($\%$) on CIFAR-100, Tiny ImageNet, and Food-101 using ViT under standard rank heterogeneity, with $r_{\max}=16$. For \hetalgname, we choose $k=r+ \max\{ 2, \lceil r/8 \rceil \}$. }
    \begin{tabular}{c|ccc}
        & CIFAR-100 & Tiny ImageNet & Food-101 \\ \hline
        FlexLoRA & $\mathbf{91.1 \pm 0.2}$ & $87.9 \pm 0.1$ & $84.2 \pm 0.6$ \\
        FSLoRA & $\underline{91.0 \pm 0.1}$ & $\underline{88.4 \pm 0.2}$ & $\mathbf{86.2 \pm 0.8}$ \\
        FedHL & $83.4 \pm 1.2$ & $79.5 \pm 1.3$ & $75.6 \pm 2.9$ \\
        \hetalgname & $\underline{91.0 \pm 0.2}$ & $\mathbf{88.7 \pm 0.2}$ & $\underline{86.1 \pm 0.3}$ \\
    \end{tabular}
    \label{tab:vitstandard}
\end{table}

\paragraph{Impact of rank}
We vary the maximum adapter rank over $\{ 2,4,8,16,32,64 \}$, and the same rank configuration is applied to all the methods at each value. Across all four methods, performance generally improves as the maximum rank $r_{\max}$ increases, whereas FedHL is more sensitive to the choice of $ r _ {\ max} $. These results indicate that the performance of \hetalgname is relatively less sensitive to the choice of the adapter rank, and remains competitive across the three datasets.
\begin{figure}[!h]
    \centering
    \begin{minipage}{.33\textwidth}
        \centering
        \includegraphics[width=\linewidth, height=0.16\textheight]{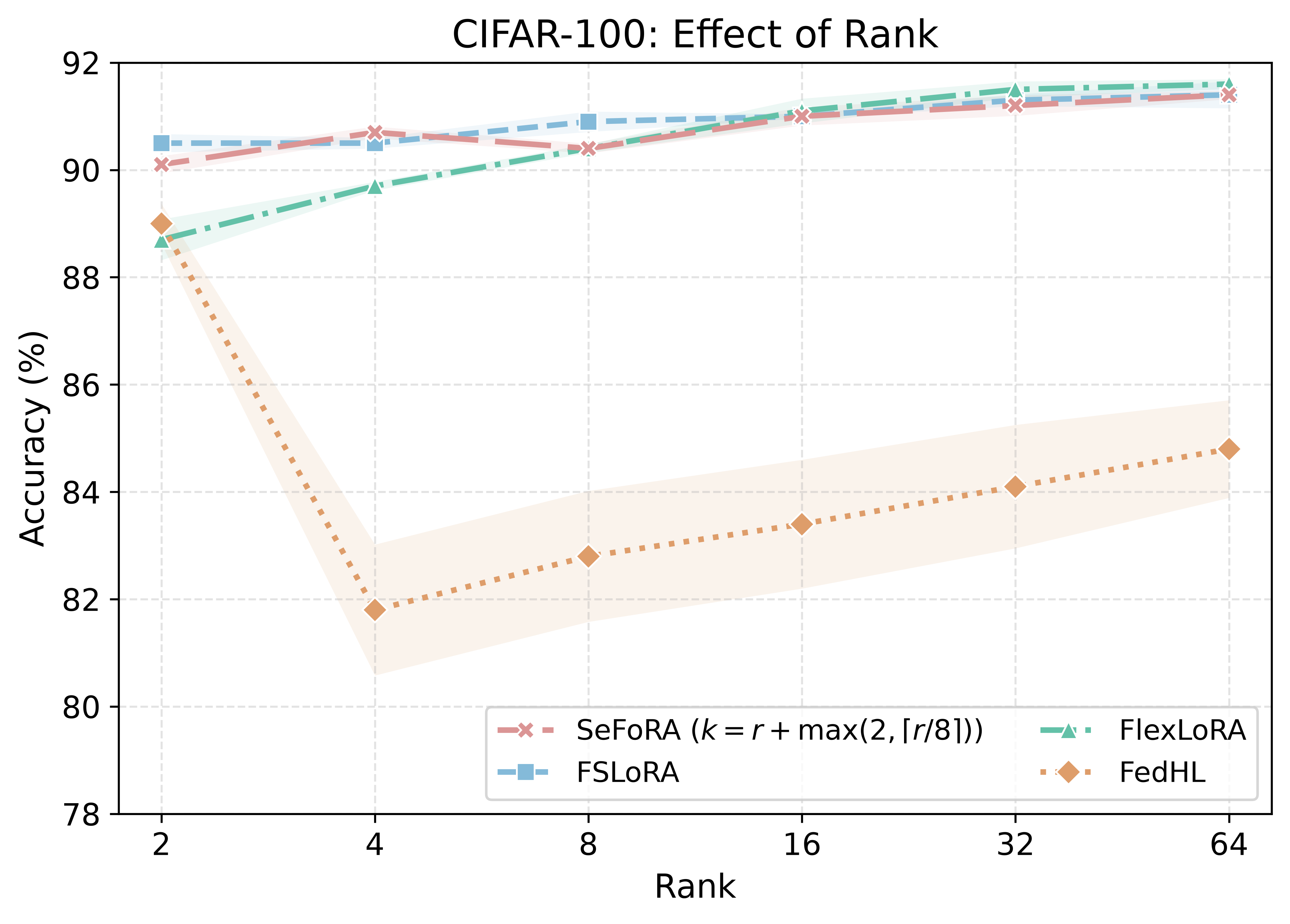}
    \end{minipage}%
    \begin{minipage}{0.33\textwidth}
        \centering
        \includegraphics[width=\linewidth, height=0.16\textheight]{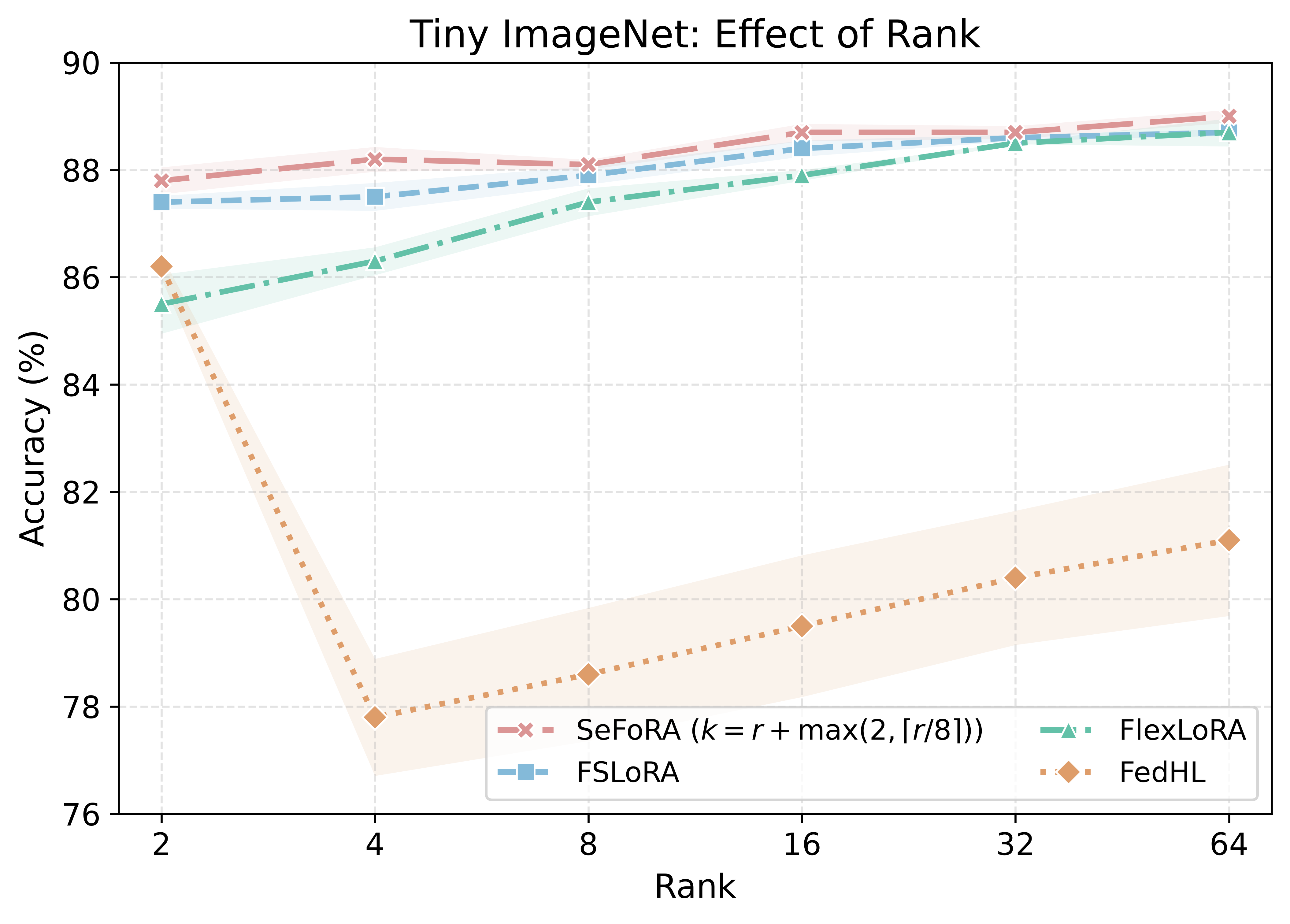}
    \end{minipage}
    \begin{minipage}{0.33\textwidth}
        \centering
        \includegraphics[width=\linewidth, height=0.16\textheight]{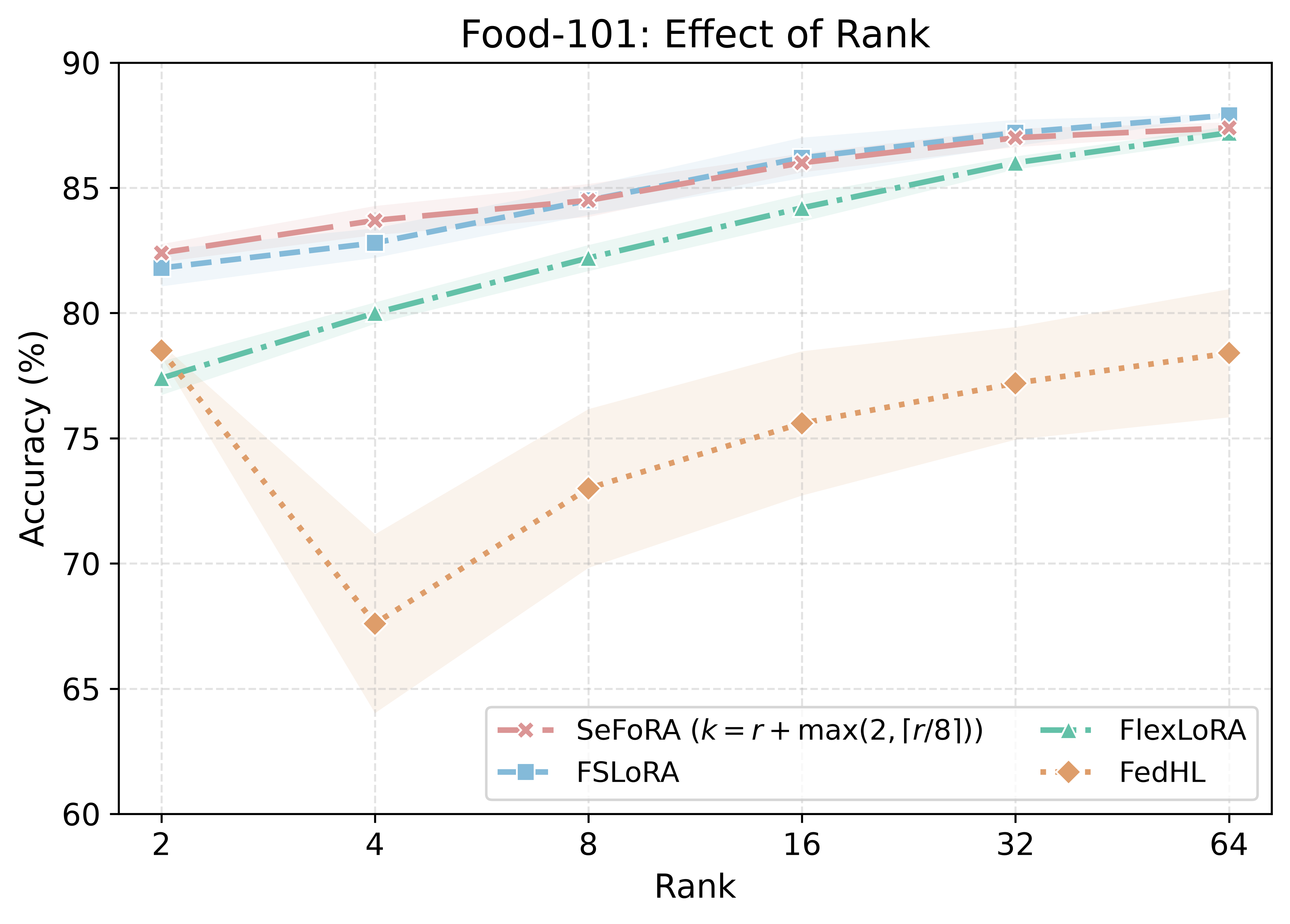}
    \end{minipage}
    \caption{Impact of the maximum adapter rank $r_{\max} \in \{ 2,4,8,16,32,64\}$ across all methods. For \hetalgname, we choose $k=r+ \max\{ 2, \lceil r/8 \rceil \}$.}
    \label{fig:impactrank}
\end{figure}

\paragraph{Impact of partial participation and the number of clients}
In \cref{fig:partial_participation}, we vary the clients' participation rate over $\{ 0.2, 0.5,1 \}$, while the total number of clients remains $n=30$. The clients are sampled uniformly without replacement in each round. 
The results show that \hetalgname remains robust when not all clients participate in each communication round and remains competitive with other methods.
\begin{figure}[!h]
    \centering
    \begin{minipage}{.33\textwidth}
        \centering
        \includegraphics[width=\linewidth, height=0.15\textheight]{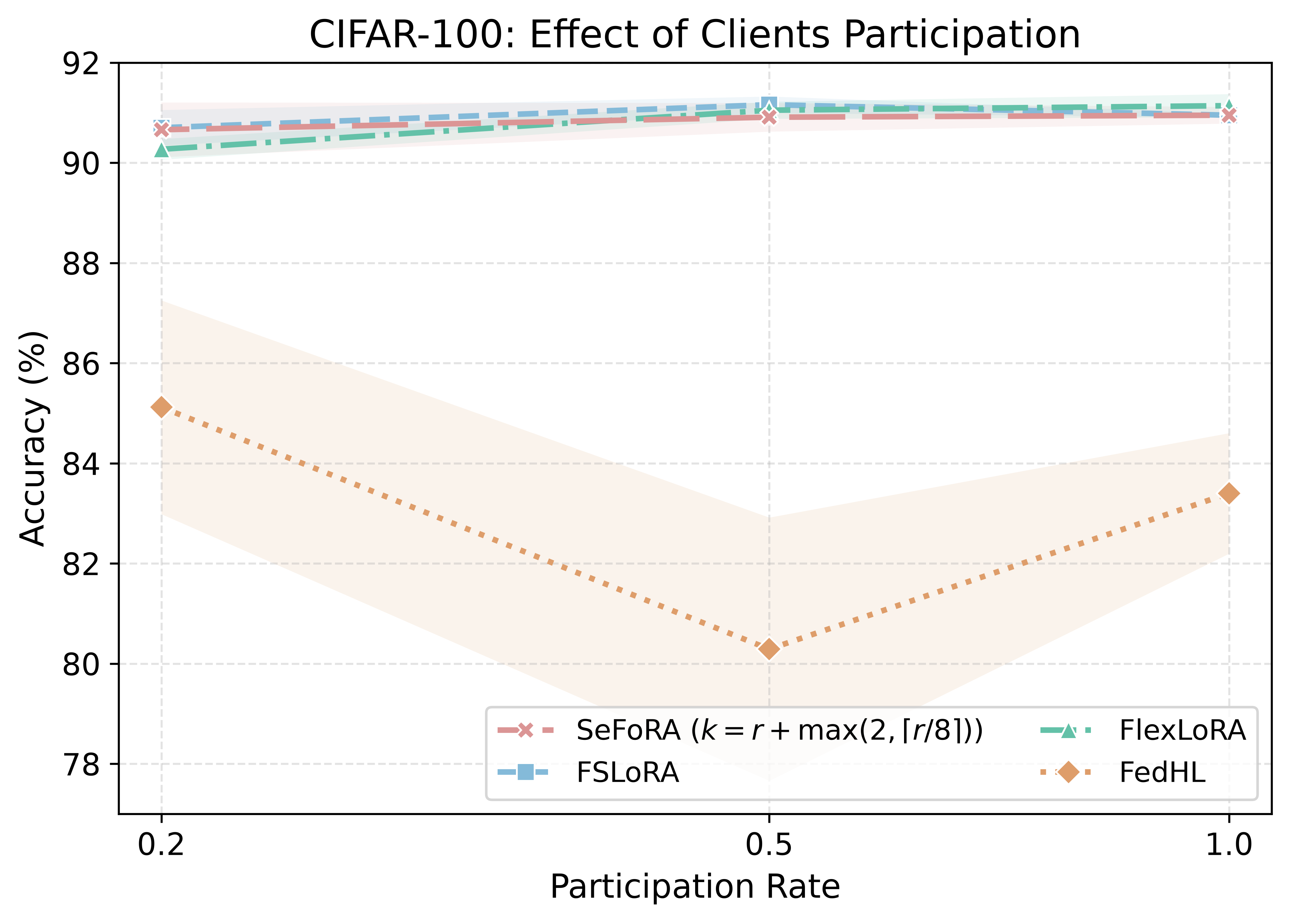}
    \end{minipage}%
    \begin{minipage}{0.33\textwidth}
        \centering
        \includegraphics[width=\linewidth, height=0.15\textheight]{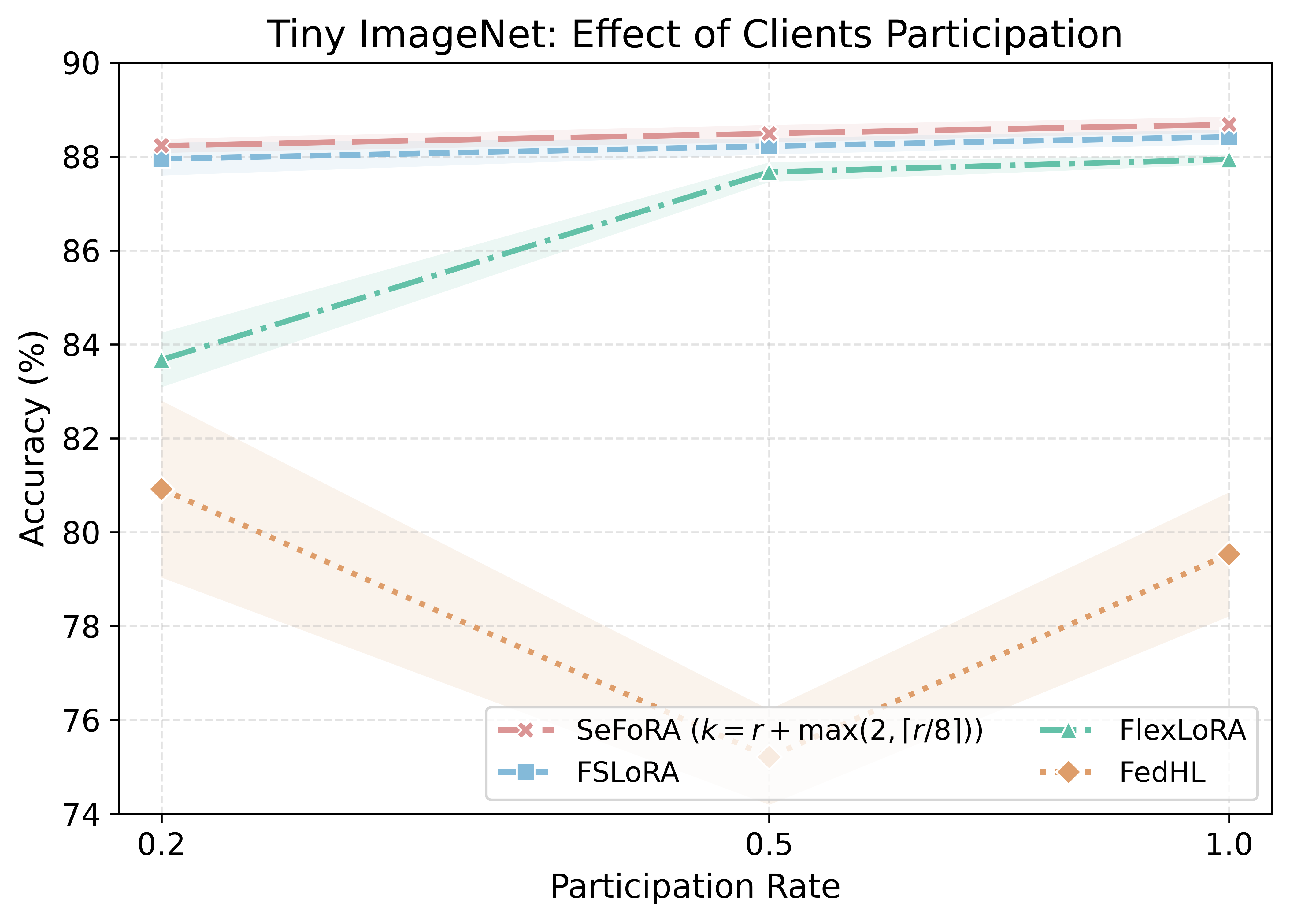}
    \end{minipage}
    \begin{minipage}{0.33\textwidth}
        \centering
        \includegraphics[width=\linewidth, height=0.15\textheight]{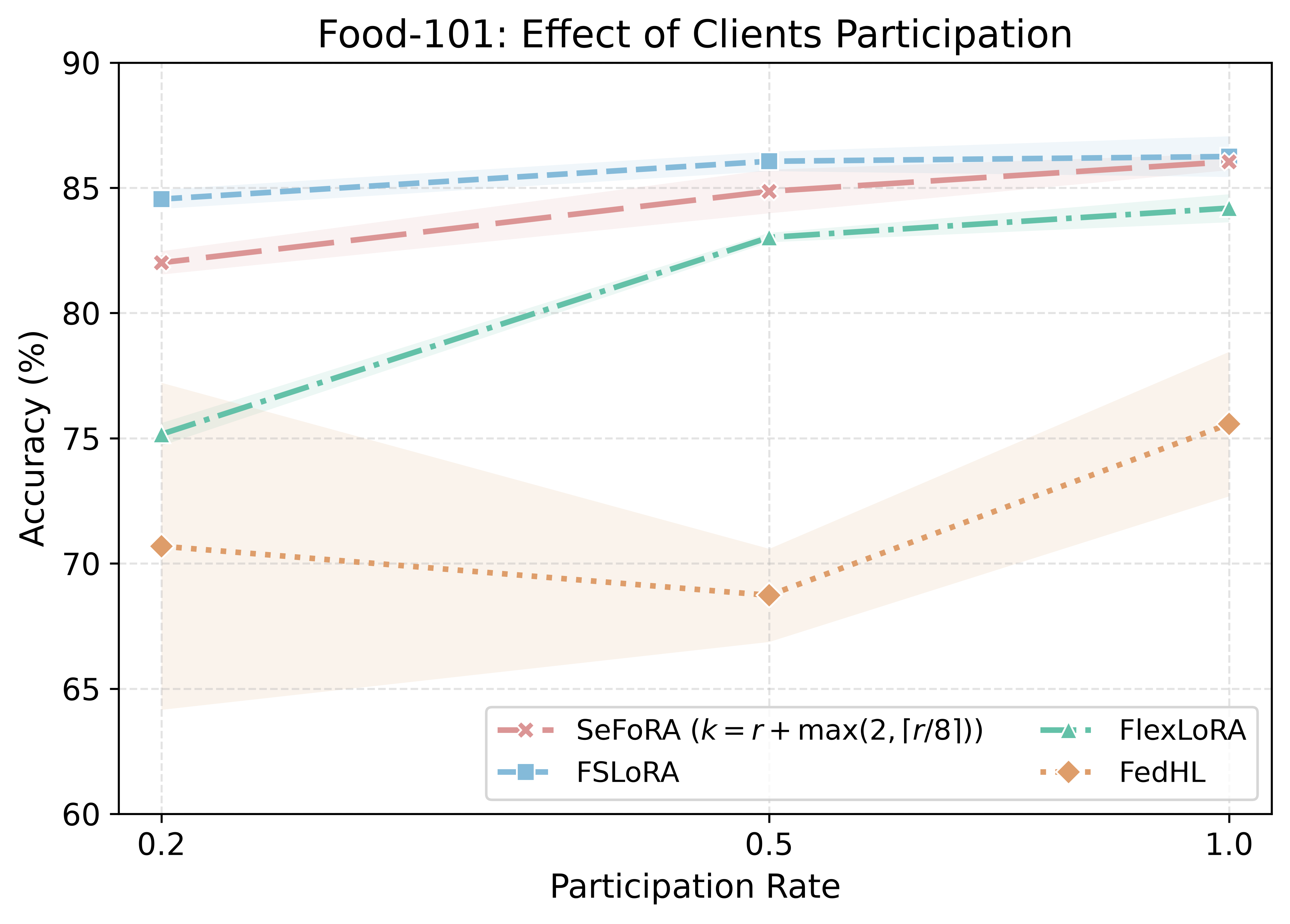}
    \end{minipage}
    \caption{Impact of different clients' participation rates $\{0.2, 0.5, 1.0\}$.}
    \label{fig:partial_participation}
\end{figure}

\cref{fig:numclients} evaluates \hetalgname with $10,30,100$ and $200$ clients, with the batch size and the number of local steps fixed, and all clients participate in every communication round. With the increase in the number of clients from $10$ to $30$, the accuracy on all three datasets improves. However, the results remain nearly unchanged as the number of clients continues to increase. Overall, \hetalgname's performance does not degrade as the size of the federated learning network increases; instead, it gradually approaches saturation as the number of participants grows.
\begin{figure}[!h]
    \centering
    \includegraphics[width=0.5\linewidth]{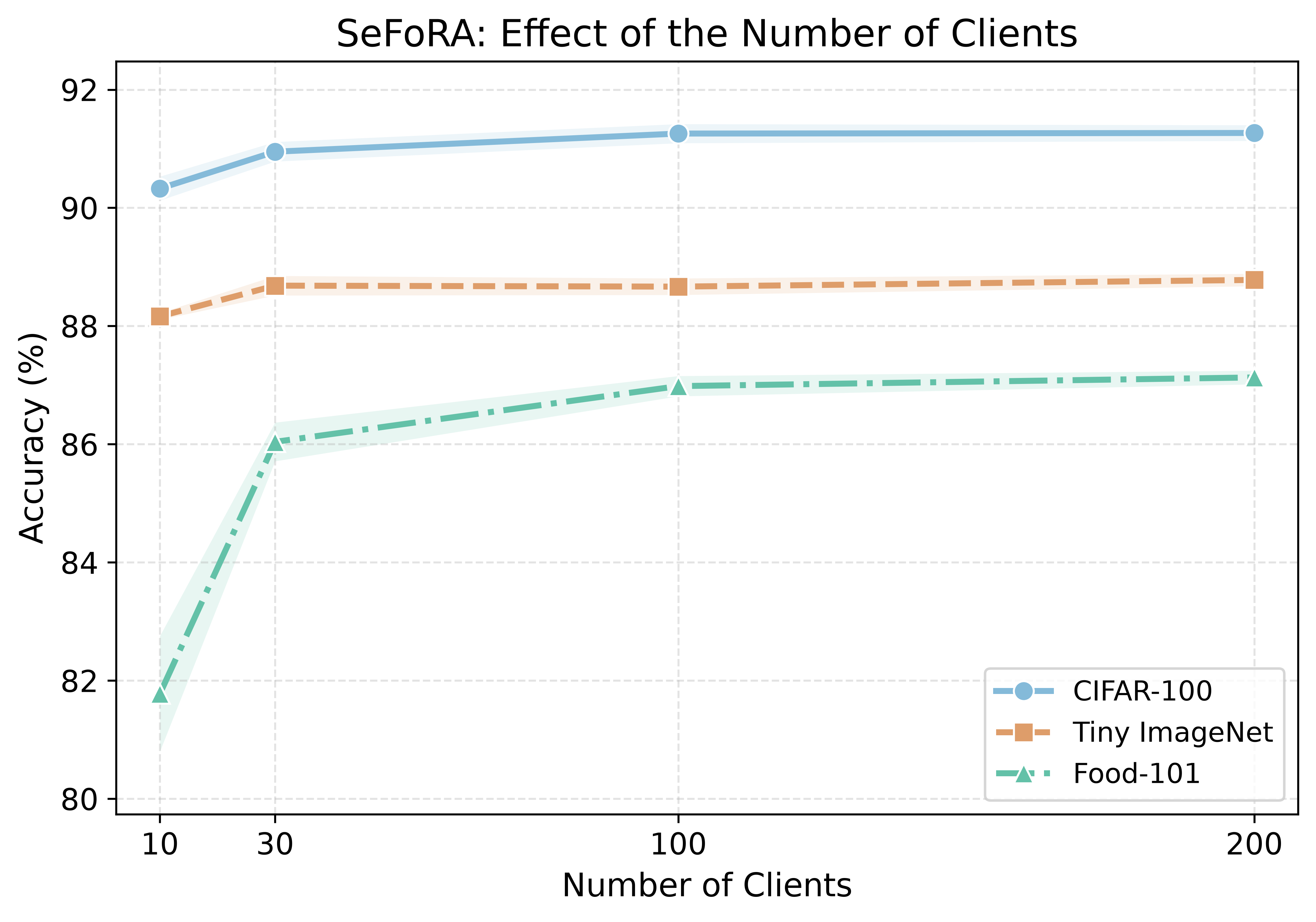}
    \caption{Impact of the number of clients $N \in \{ 10,30,100,200 \}$ on \hetalgname across the three image classification datasets. We choose $r_{\max}=16$ and $k=r+ \max\{ 2, \lceil r/8 \rceil \}$.}
    \label{fig:numclients}
\end{figure}

\paragraph{Impact of sketch size}
We first vary the sketching dimension $k$ from $k=r+2$ to $k=2r+1$, while fixing the adapter rank to $r_{\max}=16$ on the GLUE tasks, as shown in \cref{fig:gluek}. The performance on the four datasets remains largely unchanged for $k_{\max} \in \{ 18,20,24,33 \}$. 
\label[appendix]{app:addi_exp}
\begin{figure}[!h]
    \centering
    \includegraphics[width=0.5\linewidth]{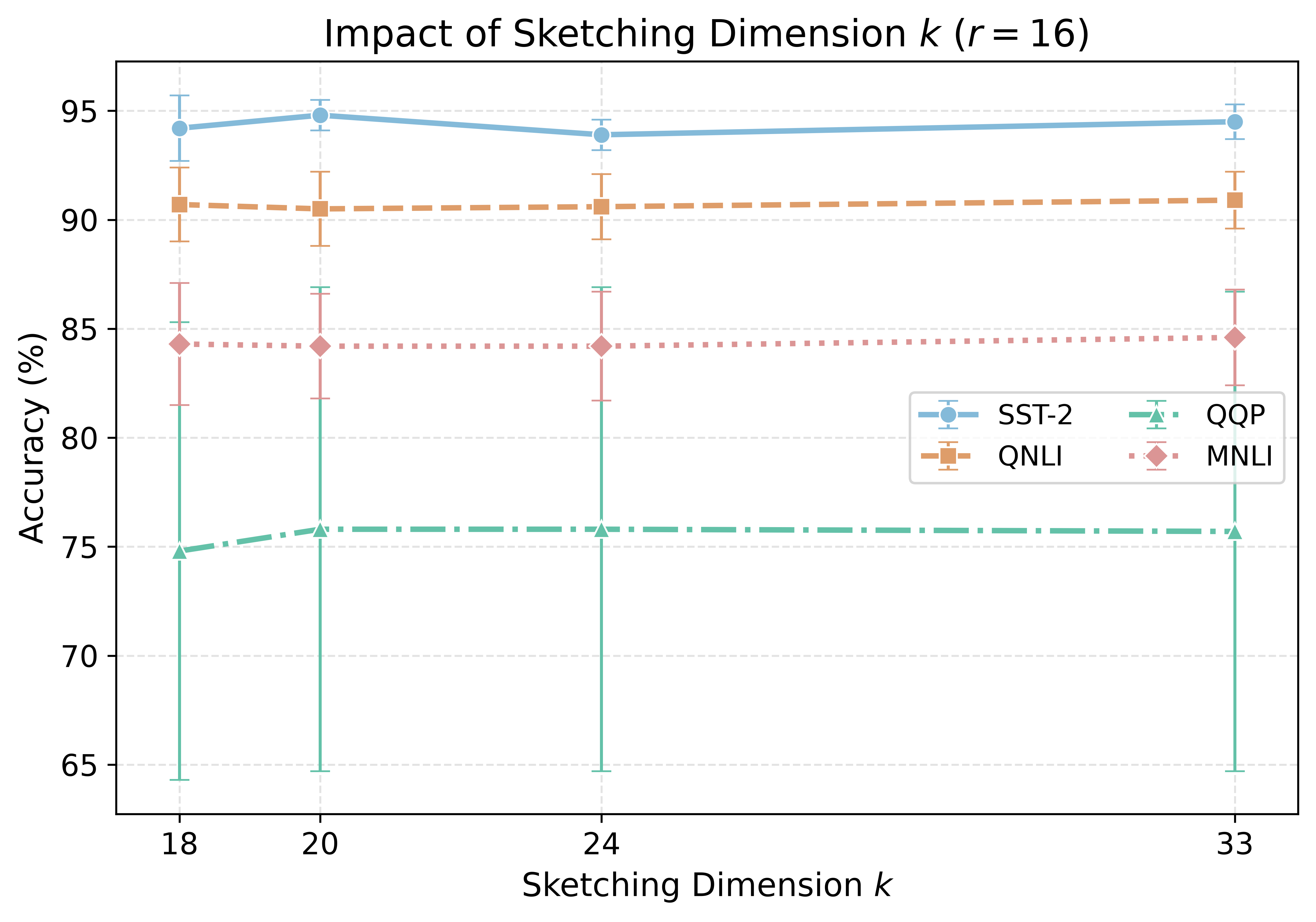}
    \caption{Impact of sketching dimension $k_{\max} \in \{ 18,20,24,33 \}$ on GLUE accuracy ($\%$), with $r_{\max}=16$.}
    \label{fig:gluek}
\end{figure}

We further compare different choices of $k$ while varying the adapter rank $r_{\max}$, as illustrated in \cref{fig:oversampling}. The results show that using the fixed sketching dimension $k=r+2$ becomes insufficient at large ranks, e.g., $r_{\max}=32,64$, with the largest degradation observed on Food-101. In contrast, using a substantially larger sketch size, such as $k=2r+1$ or $k=r+\max \{ 2, \lceil \frac{r}{2} \rceil \}$, does not provide a consistent improvement over the smaller size $k=r+\max \{ 2, \lceil \frac{r}{8} \rceil \}$. The selected choice $k=r+\max \{ 2, \lceil \frac{r}{8} \rceil \}$ therefore provides a balance between the performance and the dimension of the sketching. Note that, when $r \leq 4$, $k=r+\max \{ 2, \lceil \frac{r}{2} \rceil \}$, $r+\max \{ 2, \lceil \frac{r}{8} \rceil \}$ and $r+2$ produce the same value of $k=r+2$, and when $8 \leq r \leq 16$, $k=r+\max \{ 2, \lceil \frac{r}{8} \rceil \}$ and $k=r+2$ also result in the same value of $k=r+2$.
\begin{figure}[!h]
    \centering
    \begin{minipage}{.33\textwidth}
        \centering
        \includegraphics[width=\linewidth, height=0.15\textheight]{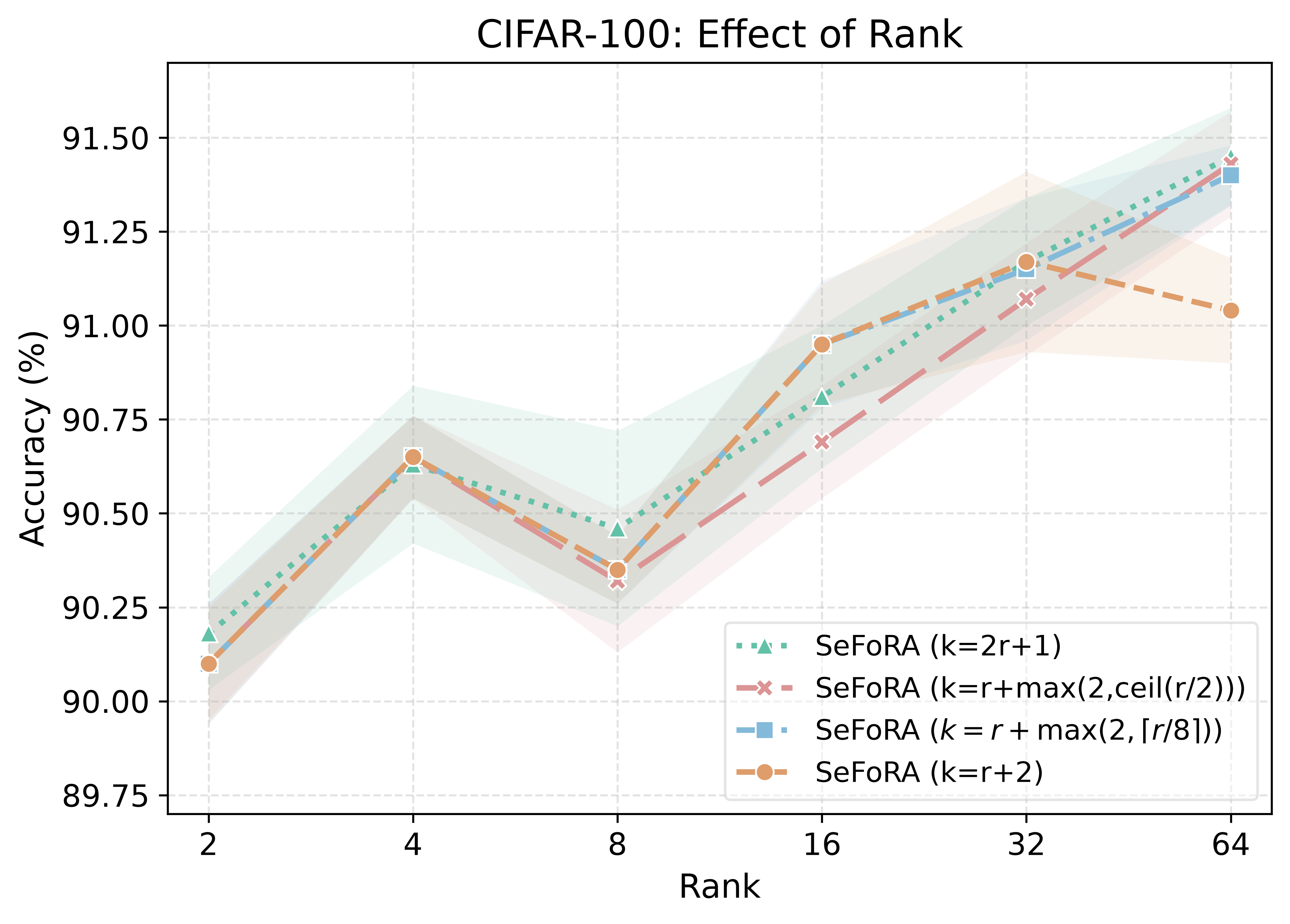}
    \end{minipage}%
    \begin{minipage}{0.33\textwidth}
        \centering
        \includegraphics[width=\linewidth, height=0.15\textheight]{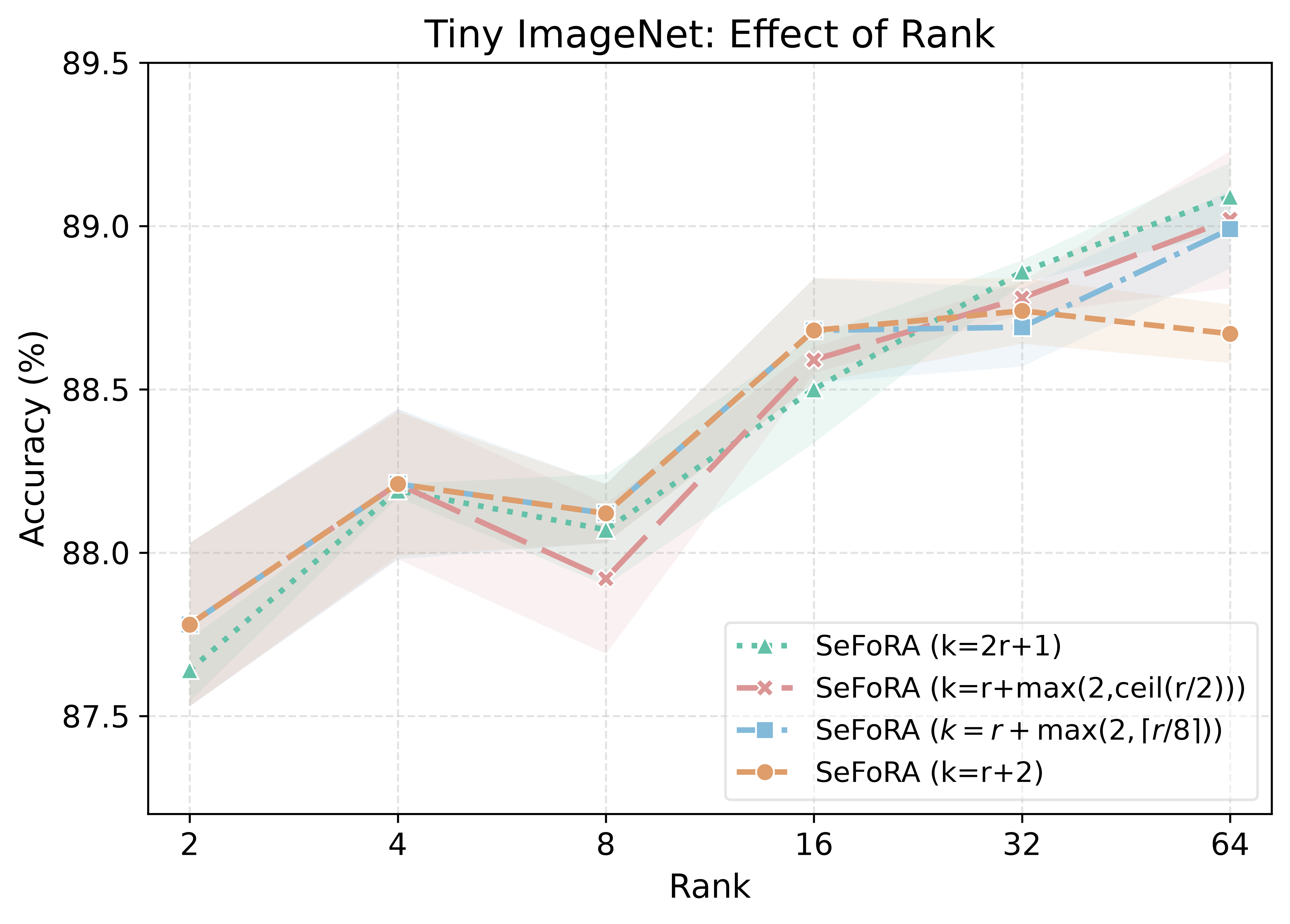}
    \end{minipage}
    \begin{minipage}{0.33\textwidth}
        \centering
        \includegraphics[width=\linewidth, height=0.15\textheight]{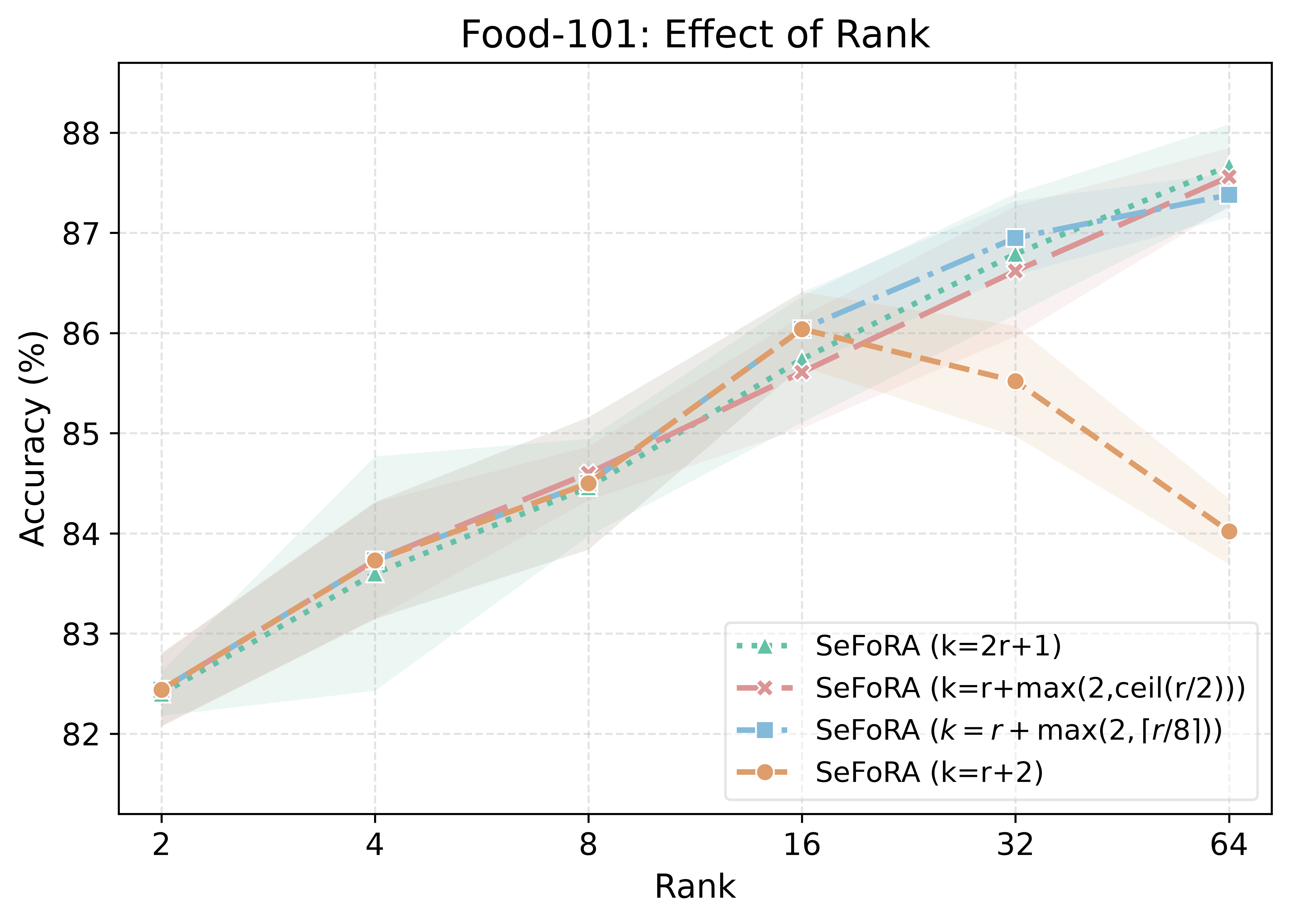}
    \end{minipage}
    \caption{Impact of different sketching dimensions $k$ on \hetalgname across maximum rank $r_{\max} \in \{2,4,8,16,32,64\}$.}
    \label{fig:oversampling}
\end{figure}

\paragraph{System Metrics}
We measure computational cost under a matched experimental configuration on QNLI using RoBERTa-Large, reported in \cref{tab:system}. All methods are trained for $400$ communication rounds, with $40$ clients with full participation, a batch size of $64$, five local steps, $r_{\max}=16$, and standard rank heterogeneity on a single NVIDIA RTX PRO 4500 Blackwell GPU. \hetalgname requires $2.40 \pm 0.15$ GPU-hours per run, comparable to FSLoRA and lower than FedHL and FlexLoRA. 
We additionally report the mean server aggregation time per communication round, averaged across rounds and seeds. \hetalgname reduces server aggregation complexity compared with methods that require dense reconstruction and a full SVD, such as FlexLoRA and FedHL.

\begin{table}
    \centering
    \begin{tabular}{c|c|c}
         & GPU hours per run (h) & Server aggregation time per round (s)\\ \hline
        FlexLoRA & $2.89$ & $4.821$\\
        FSLoRA & $\mathbf{2.38}$ & $\underline{0.127}$ \\
        FedHL & $2.78$ & $3.866$ \\
        \hetalgname & $\underline{2.40}$ & $\mathbf{0.102}$ \\
    \end{tabular}
    \caption{System efficiency on QNLI with RoBERTa-Large.}
    \label{tab:system}
\end{table}

\section{Related Work}
\label[appendix]{sec:related}

\paragraph{Parameter-efficient fine-tuning (PEFT).}
PEFT methods adapt a pretrained model $\model_0$ by updating only a small, structured subset of parameters. They fall into three families: \emph{additive} (inserting trainable modules~\citep{houlsbyParameterEfficientTransferLearning2019,heUnifiedViewParameterEfficient2021,liPrefixTuningOptimizingContinuous2021,lesterPowerScaleParameterEfficient2021}), \emph{selective} (updating a subset of existing weights~\citep{zakenBitFitSimpleParameterefficient2022,liaoParameterEfficientFineTuningIntroducing2023}), and \emph{reparametrized} (imposing low-dimensional structure on the update~\citep{hu2022lora,liu2024dora,tianHydraLoRAAsymmetricLoRA2025,zhang2023adalora,valipour2023dylora}); see~\citep{hanParameterEfficientFineTuningLarge2024} for a survey.

Among reparametrized methods, LoRA~\citep{hu2022lora} is perhaps the most popular method due to its simplicity and negligible inference overhead. The weight updates are parametrized as $\Update = \B\A$ with \emph{factor matrices} $\B \in \mathbb{R}^{d \times r}$, $\A \in \mathbb{R}^{r \times s}$, $r \ll \min(d,s)$, thus reducing the number of trainable parameters from $ds$ to $(d+s)r$. 
Beyond the original LoRA formulation \citep{hu2022lora}, numerous variants have been proposed to improve adaptation efficiency, convergence, and parameter utilization. LoRA+ \citep{hayou2024loraplus} accelerates optimization by assigning different learning rates to the two low-rank matrices. DoRA \citep{liu2024dora} improves fine-tuning performance by decomposing pretrained weights into magnitude and directional components, allowing LoRA to optimize only the directional updates. Other methods focus on improving parameter efficiency through adaptive resource allocation. LoRA-drop \citep{zhou2025loradrop} prunes redundant LoRA parameters based on their contribution to the model output, whereas AdaLoRA \citep{zhang2023adalora} dynamically reallocates the adaptation rank across layers according to their relative importance under a fixed parameter budget. HydraLoRA~\citep{tianHydraLoRAAsymmetricLoRA2025} modifies the factorization structure. Furthermore, LoRA has been adopted in application-specific settings, such as automatic prompt engineering for logical reasoning \citep{chen2024llamalora}.
On the theory side, \citep{kim2025lora} shows that $\ell_2$-regularized LoRA is equivalent to nuclear-norm-regularized full fine-tuning subject to a rank constraint.

\paragraph{PEFT and federated learning}
Federated learning (FL) has become the de facto paradigm for distributed machine learning, enabling collaborative model training while keeping data decentralized. By allowing clients to train models locally without sharing raw data, FL offers better privacy guarantees and has consequently found widespread adoption in privacy-sensitive applications, including finance \citep{long2020openbanking,chatterjee2023financial} and healthcare \citep{feng2022specificity,jiang2023fair,feng2023fedprompt,yan2024crossmodal}. Motivated by the success of large language models (LLMs), recent research has extended FL to support their efficient fine-tuning and deployment \citep{qu2024survey,fan2023fatellm,zhang2023fedit}. Several frameworks have subsequently been proposed to accelerate research in this emerging area. FederatedScope-LLM \citep{kuang2024federatedscope} provides a comprehensive platform for federated fine-tuning and instruction tuning of LLMs, while OpenFedLLM \citep{ye2024openfedllm} offers a unified framework for training LLMs over decentralized private data and systematically studies federated instruction tuning and value alignment. Complementing these systems, FedLLM-Bench \citep{ye2024fedllmbench} introduces realistic benchmarks and extensive empirical evaluations for federated fine-tuning, enabling standardized assessment of optimization algorithms and communication-efficient methods. Together, these works have established the foundational infrastructure, benchmarks, and evaluation methodologies that drive the rapidly growing field of federated fine-tuning of LLMs.

Combining PEFT with FL enables for adapting large foundation models on vast amount of data while substantially reducing communication overhead, computational complexity, and privacy risks. Existing approaches span adapter-based methods \citep{cai2023efficient,ghiasvand2024communication}, prompt tuning \citep{zhao2023fedprompt,qiu2023text}, selective parameter optimization \citep{yu2023bridging}, and low-rank gradient subspace optimization~\citep{peng2026rethinking}. More recently, Low-Rank Adaptation (LoRA)~\citep{hu2022lora} has become a popular PEFT technique on its own and in FL. 

\paragraph{Federated LoRA.}
Early works such as FedIT \citep{zhang2023fedit} integrate LoRA into the standard FedAvg framework for federated instruction tuning, while subsequent methods address the limitations of naive LoRA aggregation under heterogeneous data and system conditions. However, combining LoRA with FL raises several challenges that do not exist in centralized LoRA or in federated averaging alone. Namely: the error introduced by the bilinear mismatch resulting from first aggregating the $\A_i$'s and $\B_i$'s and then computing $\Updatebar = (\sum_ip_i \B_i)(\sum_i p_i\A_i)$ as done in FedIT~\citep{zhang2023fedit}; or the rank-$r$ approximation error resulting from first computing $\Update = \sum_i p_i \B_i \A_i$ which is not necessarily rank-$r$ and then computing a rank-$r$ approximation, e.g., using SVD decomposition as in FlexLoRA~\citep{bai2024federated}. In addition, the proposed aggregation rule should account for the \emph{rank heterogeneity} problem: when clients use different ranks $r_i$, the factor matrices have mismatched dimensions, rendering direct factor averaging infeasible. To overcome those hurdles, several approaches have been proposed, as we detail in the sequel.
Beyond model parameterization, several recent works have focused on improving the optimization process itself. FedBCGD \citep{liu2024fedbcgd} introduces an accelerated block coordinate gradient descent algorithm, FedSWA \citep{liu2024fedswa} employs stochastic weight averaging to improve generalization under highly non-IID data, FedAdamW \citep{liu2025fedadamw} proposes a communication-efficient AdamW-inspired optimizer for federated large models, FedNSAM \citep{liu2025fednsam} analyzes the relationship between local and global sharpness to improve optimization consistency, FedMuon \citep{liu2025fedmuon} accelerates convergence through matrix orthogonalization, and DP-FedPGN \citep{liu2025dpfedpgn} encourages globally flatter minima in differentially private FL by penalizing gradient norms. Client data heterogeneity is tackled in~\citep{yan2025federated}. In LA-LoRA~\citep{liu2026rethinking}, during local iterations, the clients alternately update $\B$ and $\A$ and send both updated adapters to the federator for aggregation. This alternation is shown to provide stability during training and to allow for differential privacy with lower noise.  FedASK~\citep{wen2025fedask} uses a double-sketching approach primarily to provide a differential privacy guarantee, but does not account for rank-heterogeneity. Robustness to label noise in federated LoRA is studied in~\citep{fang2026towards}. Personalized federated adaptation is also considered, e.g., \citep{zhang2026personalized,lu2026fdlora}.

\paragraph{Tackling the bilinear mismatch}
The bilinear mismatch introduces an error of $O(\eta^2)$ per round for a single local step and grows as $O(\eta^2 E^2)$ for multiple local steps~\citep{wang2024flora,singhal2024fedex} and is more prominent when considering differential privacy guarantees. In~\citep{sun2024improving}, FFA-LoRA freezes the initialized adapter $\A$ and trains only the adapter $\B$. As such, the federator avoids the bilinear mismatch problem. However, freezing the adapter $\A$ reduces the expressiveness of the adaptation and affects the performance. 
In FedSVD~\citep{lee2026fedsvd}, the clients only train $\B$ locally, while the global adapter $\A^{(t)}$ is computed based on an SVD of the matrix $\B^{(t)}\A^{(t-1)}$; thus enabling the update of the factor $\A$ while clients still train and aggregate only $\B$. In a similar spirit, FedSA-LoRA~\citep{guo2025selective} trains $\A$ globally while each client trains $\B$ locally. It provides a formal convergence analysis; however, since $\B$ is different for each client, the algorithm is suitable for personalized FL.
The concurrent works introducing
LoRA-A$^2$~\citep{koo2025towards} and RoLoRA~\citep{chen2026robust} alternate between freezing $\A$ training $\B$, and training $\A$ freezing $\B$. In LoRA-A$^2$, the federator selects an adaptive rank for each client to handle data heterogeneity. The client updates only a subset of all possible ranks, using a mask to perform the update without modifying the dimensionality of the adapters; however, the paper lacks a theoretical foundation. In RoLoRA, all clients maintain the same rank, and a convergence guarantee is given. 
In~\citep{zhang2026fedrotlora}, the authors propose FedRot-LoRA, which takes a different approach: aligning local subspaces via orthogonal rotations at the clients' side before aggregating. FedEx-LoRA~\citep{singhal2024fedex} keeps the LoRA factors trainable across rounds and absorbs the mismatch into the frozen pretrained weights as a residual term; however, it increases both computation and communication overheads and does not provide a convergence guarantee. This reduces the bilinear mismatch error but does not eliminate it. LoRA-FAIR~\citep{bian2024lorafair} introduces a correction term in $\B$ to prevent the $\widebar{\B}\widebar{\A}$ from deviating too much from the actual update $\sum_i p_i \B_i \A_i$. Despite their advantages, those methods are not applicable in rank-heterogeneous settings. 

\paragraph{Rank-heterogeneous federated LoRA}  In~\citep{babakniya2023slora}, the authors investigate methods for finding a good initialization of the adapters via full fine-tuning and then running federated LoRA, without altering the federated LoRA process. Adapting the federation of LoRA to account for rank heterogeneity is proving to be more successful. The authors of \citep{wang2024flora} propose FLoRA, in which the federator stacks the received $\A_i$'s and $\B_i$'s from the clients in a way that the multiplication of the stacked matrices emulates the computation and aggregation of the local $\Update_i$'s. As such, the globally updated $\A$ and $\B$ have larger dimensions, and the resulting $\Update$ has a rank that could be larger than $r$; thus increasing the communication and computation cost at the clients. Furthermore, in each round, the updated global model is integrated into the base model and the adapters are reinitialized, thereby potentially losing information about the adapters obtained from the previous rounds. Bai et al.~\citep{bai2024federated} took another approach in their FlexLoRA. The federator computes and aggregates the clients' $\Update_i$'s, and, for each client, the federator uses SVD to decompose and send a rank-$r_i$ approximation of the global $\Update$. This method incurs an $O(d^2s)$ computational cost at the server, and the paper analyzes the number of samples the clients need to have to obtain a good generalization bound, but it does not give a proof of convergence for the algorithm. Similarly, in FedMomentum~\citep{yan2026fedmomentum}, the federator computes and aggregates the local $\Update_i$'s to obtain $\Update$ and then performs a few mathematical transformations on $\Update$ to efficiently compute an approximation of its SVD decomposition. A key difference in this work is computing the global adapters as $\B = \widetilde{\U}\Sigma^{1/2}$ and $\A = \Sigma^{1/2}\V^\top$, where $\widetilde{\U}$, $\Sigma$ and $\V$ are the matrices obtained from the approximation of the SVD of $\Update$. The computation complexity at the federator is $O(dN^2r^2)$. Also, no convergence guarantee is given. Similarly, in FedSRD~\citep{yan2026fedsrd} the clients apply local sparsification and pruning techniques to reduce the ranks of their local adapters, and the federator reconstructs the clients' full update $\B_i \A_i$ and aggregates in the full-rank space, which increases the computation at the server. The server then computes the SVD of $\Update$ to compute a low-rank update and uses sparsification techniques to further reduce the communication cost; this work also lacks a convergence analysis. In HetLoRA~\citep{cho2024heterogeneous}, the clients prune the adapters to smaller ranks by adding a regularization factor to the optimization problem. Then, the federator zero-pads the clients' $\A_i$'s and $\B_i$'s to ensure they have the same dimension. However, when aggregating the zero-padded adapters, a careful weighted average is performed. For each adapter, a weight $p_i$ proportional to the norm of the singular value vector of $\Update_i$ is computed. The computational complexity of this method is $O(Ndrs)$ for computing the global update. Similarly, no convergence guarantee is given. FSLoRA~\citep{fang2025federated} modifies the update to $\B\mathbf{S}\A$ to embed the sketching matrices $\S_i$ at the clients and account for rank heterogeneity. The matrices $\mathbf{S}_i$ are $r\times r$ diagonal matrices with only $r_i$ non-zero entries chosen at random. The algorithm can be theoretically proven to converge with a rate of $O(\widetilde{L}/\sqrt{NT})$ when run for $T$ iterations. However, the choice of which $r_i$ entries are non-zero is random and not based on the properties of client's data.
The authors of~\citep{byun2025towards} replace the zero-padding by replicating rows of $\A$ and columns of $\B$ to make the matrices have the same ranks. In RB-LoRA~\citep{ha2026rb}, a unifying framework for completing the local adapters (via replication) and stacking them for computation at the federator is proposed. A client's local adapters are weighted based on the amount of data the client holds, proportional to the total number of data (for $\A$) and to the total number of data held by clients with ranks less than or equal to the client's (for $\B$). The method's performance is evaluated via numerical experiments.
In~\citep{zhou2025aflora}, the proposed AFLoRA adds a diagonal matrix $\boldsymbol{\Lambda}$ of dimension $r\times r$ to the clients' local adapters to have $\Update_i = \B_i \boldsymbol{\Lambda}_i \A_i$. The diagonal matrix is meant to compute the contribution of each dimension of $\B_i$ to the local $\Update$. Then, the clients accordingly adjust their local ranks. The adapter $\A$ is only trained at the federator using a public dataset, and the $\B_i$ and $\boldsymbol{\Lambda}_i$ are trained at the clients. To aggregate the local adapters, zero padding is used, and aggregation weights proportional to the clients' data volume and local LoRA ranks are computed. FedARA~\citep{wu2026adaptive} adaptively allocates local ranks to each layer at clients, and layers with small ranks are not updated. Through a majority vote at the federator, global layers will also be pruned. However, aggregating client updates still requires zero-padding and matrices of the same dimension.
Recent work has further examined the biases caused by heterogeneous ranks~\citep{wu2026preventing,peng2025fedhl}. In~\citep{wu2026preventing}, rank collapse is identified, i.e., the energy of the aggregated update progressively concentrates on the smallest rank shared by the participating clients. They propose raFLoRA, which splits each local update into rank partitions and aggregates them separately. This prevents higher-rank components from being diluted by clients who do not update them. Nevertheless, the server reconstructs the partitioned updates in the full model space and applies an SVD to recover global factors, incurring computational overhead at the server. FedHL~\citep{peng2025fedhl} focuses on the truncation-induced bias arising in federated LoRA with heterogeneous ranks. The method maintains a full-rank global model as a calibrated aggregation basis, and the truncation bias can be eliminated by using a compensation mechanism. ECLoRA~\citep{ning2025federated} reconstructs the full rank update at the server and utilizes a randomized SVD for broadcasting. It also introduces a compensation term to further improve the performance of the algorithm. 

In~\citep{senarath2026subspace}, the authors introduce a subspace regularization term in the objective, making sure the subspaces spanned by the clients' low-rank updates stay aligned. However, in each round, the clients send $\Update_i =\B_i \A_i$ to the federator, increasing the communication overhead between the clients and the federator.

\section{Proofs}
In the proofs below, we adopt the following notation.
Let $\mathcal{F}_t$ represent all algorithmic history up to the beginning of round $t$. We define $\mathbb{E}_t[\cdot] \triangleq \mathbb{E}[\cdot \mid \mathcal{F}_t]$ as the conditional expectation taking into account all sources of randomness occurring during round $t$. Let $\mathbb{E}_{\xi_i}[\cdot]$ denote the expectation taken with respect to the local mini-batch sampling at round $t$, conditioned on the model state at the moment of sampling.

For each client $i$, let $\mathbf{r}_i = (r_{i,1},\cdots,r_{i,L})$ denote its rank profile. For each layer $l$, define the truncation operation $\TruncOpe(\cdot)$ through a matrix 
$\Trunc_{i,l}=
\begin{bmatrix} 
\mathbf{I}_{r_{i,l}} \\ 
\mathbf{0}
\end{bmatrix} \in \mathbb{R}^{r_l \times r_{i,l}}$, 
and the truncation tuple is $\TupleTrunc_i \triangleq \{ \Trunc_{i,1}, \cdots, \Trunc_{i,L} \}$. For global factors $\TupleB^{(t)}$ and $\TupleA^{(t)}$, define $\llbracket \TupleB^{(t)} \rrbracket_{r_i} \triangleq (\B_l \Trunc_{i,l})_{l=1}^L$ and $\llbracket \TupleA^{(t)} \rrbracket_{r_i} \triangleq ((\Trunc_{i,l})^\top \A_{l})_{l=1}^L$. We further define 
\begin{align*}
    g_{\TupleA,i,r_i}^{(t)} &= \nabla_{\TupleA} \Loss_i^{\text{lora}} ( \llbracket \TupleB^{(t)} \rrbracket_{r_i}, \llbracket \TupleA^{(t)} \rrbracket_{r_i}), & \quad  g_{\TupleB,i,r_i}^{(t)} &= \nabla_{\TupleB} \Loss_i^{\text{lora}} (\llbracket \TupleB^{(t)} \rrbracket_{r_i}, \llbracket \TupleA^{(t)} \rrbracket_{r_i})\\
    \hat{g}_{\TupleA,i,r_i}^{(t,e)} &= \nabla_{\TupleA} \Loss_i^{\text{lora}} (\TupleB_{i,r_i}^{(t,e)}, \TupleA_{i,r_i}^{(t,e)}; \xi_{i}^{(t,e)}), & \quad \hat{g}_{\TupleB,i,r_i}^{(t,e)} &= \nabla_{\TupleB} \Loss_i^{\text{lora}} (\TupleB_{i,r_i}^{(t,e)}, \TupleA_{i,r_i}^{(t,e)}; \xi_{i}^{(t,e)}) \\
    g_{\TupleA,i,r_i}^{(t,e)} &= \nabla_{\TupleA} \Loss_i^{\text{lora}} (\TupleB_{i,r_i}^{(t,e)}, \TupleA_{i,r_i}^{(t,e)}), & \quad g_{\TupleB,i,r_i}^{(t,e)} &= \nabla_{\TupleB} \Loss_i^{\text{lora}} (\TupleB_{i,r_i}^{(t,e)}, \TupleA_{i,r_i}^{(t,e)}) \\
    \widebar{g}_{\TupleA,i,r_i}^{(t)} & = \frac{1}{E} \sum_{e=0}^{E-1} \hat{g}_{\TupleA,i,r_i}^{(t,e)}, & \quad \widebar{g}_{\TupleB,i,r_i}^{(t)} &= \frac{1}{E} \sum_{e=0}^{E-1} \hat{g}_{\TupleB,i,r_i}^{(t,e)} \\
    D_{\TupleA, i, r_i}^{(t)}& = \mathbb{E}_t \left[ \widebar{g}_{\TupleA,i,r_i}^{(t)} \right] - g_{\TupleA, i, r_i}^{(t)}, & \quad D_{\TupleB, i, r_i}^{(t)} &= \mathbb{E}_t \left[ \widebar{g}_{\TupleB,i,r_i}^{(t)} \right] - g_{\TupleB, i, r_i}^{(t)} 
\end{align*}

In what follows, we prove the necessary lemmas.
\begin{lem}
    Under \cref{assump:smooth}, the local empirical loss $\Loss_i^\text{full}(\TupleUpdate)$, the global empirical loss $\Loss^{\text{full}} (\TupleUpdate)$, and the stochastic gradient $\Loss_i^\text{full}(\TupleUpdate; \xi_i)$ evaluated on mini-batch $\xi_i$ are $\mu$-smooth with respect to $\TupleUpdate$, i.e.,
    \begin{align*}
    \| \nabla_{\TupleUpdate} \Loss_i^\text{full}(\TupleUpdate) - \nabla_{\TupleUpdate} \Loss_i^\text{full}(\TupleUpdate') \|_F & \leq \mu \| \TupleUpdate - \TupleUpdate' \|_F, \\
    \| \nabla_{\TupleUpdate}\Loss^{\text{full}} (\TupleUpdate) - \nabla_{\TupleUpdate}\Loss^{\text{full}} (\TupleUpdate') \|_F & \leq \mu \| \TupleUpdate - \TupleUpdate' \|_F, \quad \text{ and }\\
    \mathbb{E}_{\xi_i} \left[\| \nabla_{\TupleUpdate} \Loss_i^\text{full}(\TupleUpdate;\xi_i) - \nabla_{\TupleUpdate} \Loss_i^\text{full}(\TupleUpdate';\xi_i) \|_F \right] & \leq \mu \| \TupleUpdate - \TupleUpdate' \|_F.
    \end{align*}
    \begin{proof}
Each loss in the statement is an average of sample losses. Averaging the
inequality in Assumption~\ref{assump:smooth} and using the triangle inequality concludes the proof.
\end{proof}
\end{lem}

\begin{lem}
\label{lem:model_local_drift}
    Under \cref{assump:variance} and \cref{assump:normbound}, \hetalgname guarantees that, the local drift of the LoRA factors $\TupleB$ and $\TupleA$, and the update $\TupleUpdate$ at the local step $e \in \{ 0, \cdots, E-1 \}$, and for any round $t \in \{ 0, \cdots, T-1 \}$ and client $i \in [N]$ are bounded as:
    \begin{align*}
        \mathbb{E}_t \left\| \TupleB_{i,r_i}^{(t,e)} - \TupleB_{i,r_i}^{(t,0)} \right\|_F^2 & \leq (\eta^{(t)})^2 e^2 \chi^2 M_{\A}^2, \\
        \mathbb{E}_t \left\| \TupleA_{i,r_i}^{(t,e)} - \TupleA_{i,r_i}^{(t,0)} \right\|_F^2 & \leq (\eta^{(t)})^2 e^2 \chi^2 M_{\B}^2, \quad \text{and} \\
        \mathbb{E}_t \left\| \TupleUpdate_{i,r_i}^{(t,e)} - \TupleUpdate_{i,r_i}^{(t,0)} \right\|_F^2 & \leq 2 (\eta^{(t)})^2  e^2 \chi^2 (M_{\B}^4 + M_{\A}^4).
    \end{align*}

\begin{proof}
    The local fine-tuning starts from $(\TupleB_{i,r_i}^{(t,0)},\TupleA_{i,r_i}^{(t,0)})$ and updates
    \begin{align*}
        \TupleB_{i,r_i}^{(t,e+1)}& =\TupleB_{i,r_i}^{(t,e)}-\eta^{(t)} \cdot \hat{g}_{\TupleB,i,r_i}^{(t,e)}\\
        \TupleA_{i,r_i}^{(t,e+1)}& =\TupleA_{i,r_i}^{(t,e)}-\eta^{(t)} \cdot \hat{g}_{\TupleA,i,r_i}^{(t,e)}
    \end{align*}
    where we define $\hat{g}_{\TupleA,i,r_i}^{(t,e)} = \nabla_{\TupleA} \Loss_{i}^{\text{full}}(\TupleB_{i,r_i}^{(t,e)}\TupleA_{i,r_i}^{(t,e)}; \xi_{i}^{(t,e)}) = (\TupleB_{i,r_i}^{(t,e)})^\top \cdot \nabla_{\TupleUpdate} \Loss_{i}^{\text{full}}(\TupleUpdate_{i,r_i}^{(t,e)}; \xi_{i}^{(t,e)})$ and $\hat{g}_{\TupleB,i,r_i}^{(t,e)} = \nabla_{\TupleB} \Loss_{i}^{\text{full}}(\TupleB_{i,r_i}^{(t,e)}\TupleA_{i,r_i}^{(t,e)}; \xi_{i}^{(t,e)}) = \nabla_{\TupleUpdate} \Loss_{i}^{\text{full}}(\TupleUpdate_{i,r_i}^{(t,e)}; \xi_{i}^{(t,e)}) \cdot (\TupleA_{i,r_i}^{(t,e)})^\top$. Thus, for $e \in \{ 0, \cdots, E-1 \}$,
    \begin{align*}
        \mathbb{E}_t \left\| \TupleB_{i,r_i}^{(t,e)} - \TupleB_{i,r_i}^{(t,0)} \right\|_F^2 & = \mathbb{E}_t \left\| - \eta^{(t)} \sum_{j=0}^{e-1}  \hat{g}_{\TupleB,i,r_i}^{(t,j)} \right\|_F^2 \\
        & = \mathbb{E}_t \left[ (\eta^{(t)})^2 \cdot \left\| \sum_{j=0}^{e-1}  \hat{g}_{\TupleB,i,r_i}^{(t,j)} \right\|_F^2 \right] \\
        & \stackrel{(a)}{\leq} (\eta^{(t)})^2 e \cdot \sum_{j=0}^{e-1} \mathbb{E}_t \left\| \hat{g}_{\TupleB,i,r_i}^{(t,j)} \right\|_F^2 \\
        & \stackrel{(b)}{\leq} (\eta^{(t)})^2 e \cdot \sum_{j=0}^{e-1} \mathbb{E}_t \left\| \nabla_{\TupleUpdate} \Loss_{i}^{\text{full}}(\TupleUpdate_{i,r_i}^{(t,j)}; \xi_{i}^{(t,j)}) \right\|_F^2 \cdot \left\| \TupleA_{i,r_i}^{(t,j)} \right\|_2^2 \\
        & \stackrel{(c)}{\leq} (\eta^{(t)})^2 e^2 \chi^2 M_{\A}^2 \\
    \end{align*}
    where $(a)$ holds due to the Jensen's inequality, $(b)$ employs $\|XY\|_F^2 \leq \|X\|_2^2 \cdot \|Y\|_F^2$, and $(c)$ utilizes \cref{assump:variance} and \cref{assump:normbound}. Similarly, $\mathbb{E}_t \left\| \TupleA_{i,r_i}^{(t,e)} - \TupleA_{i,r_i}^{(t,0)} \right\|_F^2 \leq (\eta^{(t)})^2  e^2 \chi^2 M_{\B}^2$.
    Then,
    \begin{align*}
        \mathbb{E}_t \left\| \TupleUpdate_{i,r_i}^{(t,e)} -  \TupleUpdate_{i,r_i}^{(t,0)} \right\|_F^2 & = \mathbb{E}_t \left\| \TupleB_{i,r_i}^{(t,e)} \TupleA_{i,r_i}^{(t,e)} - \TupleB_{i,r_i}^{(t,0)} \TupleA_{i,r_i}^{(t,0)} \right\|_F^2 \\
        & = \mathbb{E}_t \left\| \TupleB_{i,r_i}^{(t,e)} (\TupleA_{i,r_i}^{(t,e)} - \TupleA_{i,r_i}^{(t,0)}) + (\TupleB_{i,r_i}^{(t,e)} - \TupleB_{i,r_i}^{(t,0)}) \TupleA_{i,r_i}^{(t,0)} \right\|_F^2 \\
        & \leq 2\mathbb{E}_t \left\| \TupleB_{i,r_i}^{(t,e)} \right\|_2^2 \left\| \TupleA_{i,r_i}^{(t,e)} - \TupleA_{i,r_i}^{(t,0)} \right\|_F^2  + 2\mathbb{E}_t \left\| \TupleB_{i,r_i}^{(t,e)} - \TupleB_{i,r_i}^{(t,0)} \right\|_F^2 \left\| \TupleA_{i,r_i}^{(t,0)} \right\|_2^2 \\
        & \leq 2 M_{\B}^2 \mathbb{E}_t \left\| \TupleA_{i,r_i}^{(t,e)} - \TupleA_{i,r_i}^{(t,0)}  \right\|_F^2  + 2 M_{\A}^2 \mathbb{E}_t \left\| \TupleB_{i,r_i}^{(t,e)} - \TupleB_{i,r_i}^{(t,0)} \right\|_F^2 \\
        & \leq 2 (\eta^{(t)})^2  e^2 \chi^2 (M_{\B}^4 + M_{\A}^4) 
    \end{align*}
\end{proof}
\end{lem}

\begin{lem}
\label{lem:drift}
For clients $i \in [N]$ and round $t \in \{ 0,\cdots, T-1\}$, under \cref{assump:smooth}, \cref{assump:variance} and \cref{assump:normbound}, the following bounds hold:
    $$\sum_{i \in [N]} p_i \big\|D_{\TupleA,i,r_i}^{(t)}\big\|_F^2 \leq \frac{(E-1)(2E-1)}{3} \chi^2 (\eta^{(t)})^2 \big(\chi^2 M_\A^2 + 2\mu^2 M_{\B}^2 (M_\B^4 + M_\A^4)\big),$$ 
    and
    $$\sum_{i \in [N]} p_i \big\|D_{\TupleB,i,r_i}^{(t)}\big\|_F^2 \leq \frac{(E-1)(2E-1)}{3} \chi^2 (\eta^{(t)})^2 \big(\chi^2 M_\B^2 + 2\mu^2 M_\A^2 (M_\B^4 + M_\A^4)\big).$$

\begin{proof}
    We prove the bound on $\TupleA$; the $\TupleB$ bound follows by the symmetric argument.
    The Frobenius norm of the drifts $\| D_{\TupleA,i,r_i}^{(t)} \|_F^2$ is bounded as:
    \begin{align*}
        & \left\| D_{\TupleA,i,r_i}^{(t)} \right\|_F^2  = \left\| \mathbb{E}_t \left[ \frac{1}{E} \sum_{e=0}^{E-1} \hat{g}_{\TupleA,i,r_i}^{(t,e)} \right] - g_{\TupleA, i, r_i}^{(t)} \right\|_F^2 \\
        & = \left\|  \frac{1}{E} \sum_{e=0}^{E-1} \mathbb{E}_t \left[\hat{g}_{\TupleA,i,r_i}^{(t,e)} - g_{\TupleA, i, r_i}^{(t)} \right] \right\|_F^2 \\
        & \stackrel{(a)}{=} \left\|  \frac{1}{E} \sum_{e=0}^{E-1} \mathbb{E}_t \left[ g_{\TupleA,i,r_i}^{(t,e)} - g_{\TupleA, i, r_i}^{(t)} \right] \right\|_F^2  \\
        & \stackrel{(b)}{\leq}  \frac{1}{E} \sum_{e=0}^{E-1} \mathbb{E}_t \left\| g_{\TupleA,i,r_i}^{(t,e)} - g_{\TupleA, i, r_i}^{(t)} \right\|_F^2  \\
        & = \frac{1}{E} \sum_{e=0}^{E-1} \mathbb{E}_t \left\| (\TupleB_{i,r_i}^{(t,e)})^\top \cdot \nabla_{\TupleUpdate} \Loss_{i}^{\text{full}}(\TupleUpdate_{i,r_i}^{(t,e)}) - (\TupleB_{i,r_i}^{(t,0)})^\top \cdot \nabla_{\TupleUpdate} \Loss_{i}^{\text{full}}(\TupleUpdate_{i,r_i}^{(t,0)}) \right\|_F^2  \\
        & = \frac{1}{E} \sum_{e=0}^{E-1} \mathbb{E}_t \bigg\| \left( (\TupleB_{i,r_i}^{(t,e)})^\top - (\TupleB_{i,r_i}^{(t,0)})^\top \right)  \cdot \nabla_{\TupleUpdate} \Loss_{i}^{\text{full}}(\TupleUpdate_{i,r_i}^{(t,e)}) \\
        & \quad \quad \quad \quad \quad \quad \quad + (\TupleB_{i,r_i}^{(t,0)})^\top \cdot \left( \nabla_{\TupleUpdate} \Loss_{i}^{\text{full}}(\TupleUpdate_{i,r_i}^{(t,e)}) - \nabla_{\TupleUpdate} \Loss_i^{\text{full}} (\TupleUpdate_{i,r_i}^{(t,0)}) \right)  \bigg\|_F^2  \\
        & \stackrel{(c)}{\leq} 2 \frac{1}{E} \sum_{e=0}^{E-1} \mathbb{E}_t \left\| (\TupleB_{i,r_i}^{(t,e)})^\top - (\TupleB_{i,r_i}^{(t,0)})^\top \right\|_F^2 \left\| \nabla_{\TupleUpdate} \Loss_{i}^{\text{full}}(\TupleUpdate_{i,r_i}^{(t,e)}) \right\|_F^2 \\
        & \quad + 2 \frac{1}{E} \sum_{e=0}^{E-1} \mathbb{E}_t \left\| (\TupleB_{i,r_i}^{(t,0)})^\top \right\|_2^2 \left\| \nabla_{\TupleUpdate} \Loss_{i}^{\text{full}}(\TupleUpdate_{i,r_i}^{(t,e)}) - \nabla_{\TupleUpdate} \Loss_i^{\text{full}} (\TupleUpdate_{i,r_i}^{(t,0)}) \right\|_F^2 \\
        & \stackrel{(d)}{\leq} 2  \chi^2 \frac{1}{E} \sum_{e=0}^{E-1} \mathbb{E}_t \left\| (\TupleB_{i,r_i}^{(t,e)})^\top - (\TupleB_{i,r_i}^{(t,0)})^\top \right\|_F^2 + 2 \cdot \mu^2 M_{\B}^2 \frac{1}{E} \sum_{e=0}^{E-1} \mathbb{E}_t \left\| \TupleUpdate_{i,r_i}^{(t,e)} - \TupleUpdate_{i,r_i}^{(t,0)} \right\|_F^2   \\
        & \stackrel{(e)}{\leq} 2  \chi^2 \frac{1}{E} \sum_{e=0}^{E-1} (\eta^{(t)})^2  e^2 \chi^2 M_{\A}^2  + 2 \cdot \mu^2 M_{\B}^2 \frac{1}{E} \sum_{e=0}^{E-1} 2 (\eta^{(t)})^2  e^2 \chi^2 (M_{\B}^4 + M_{\A}^4)\\
        & = 2 \chi^2 (\eta^{(t)})^2 \left(\chi^2 M_{\A}^2  + 2 \mu^2 M_{\B}^2 (M_{\B}^4 + M_{\A}^4)\right) \cdot \frac{1}{E} \sum_{e=0}^{E-1} e^2 \\
        & = \frac{(E-1)(2E-1)}{3} \chi^2 (\eta^{(t)})^2 \left(\chi^2 M_{\A}^2  + 2 \mu^2 M_{\B}^2 (M_{\B}^4 + M_{\A}^4)\right)
    \end{align*}
    where $(a)$ holds since $\mathbb{E}_t \left[ g_{\TupleA,i,r_i}^{(t,e)} \right] = \mathbb{E}_t \left[ \hat{g}_{\TupleA,i,r_i}^{(t,e)} \right]$, which is obtained by conditioning on the history before sampling of the mini-batch at local step $e$ and applying the tower property, $(b)$ holds by applying Jensen's inequality, $(c)$ holds since $\|XY\|_F^2 \leq \|X\|_2^2 \cdot \|Y\|_F^2 \leq \|X\|_F^2 \cdot \|Y\|_F^2$, $(d)$ holds since $\left\| \nabla_{\TupleUpdate} \Loss_{i}^{\text{full}}(\TupleUpdate_{i}^{(t,e)}) \right\|_F^2 = \left\| \mathbb{E}_{\xi_i} \left[ \nabla_{\TupleUpdate} \Loss_{i}^{\text{full}}(\TupleUpdate_{i}^{(t,e)}; \xi_i) \right] \right\|_F^2 \leq \mathbb{E}_{\xi_i} \left\| \nabla_{\TupleUpdate} \Loss_{i}^{\text{full}}(\TupleUpdate_{i}^{(t,e)}; \xi_i) \right\|_F^2 \leq \chi^2$, \cref{assump:smooth} and \cref{assump:normbound}, and $(e)$ holds due to \cref{lem:model_local_drift} and $\sum_{i=0}^{E-1} e^2 = \frac{(E-1)E(2E-1)}{6}$. Similarly, 
    \begin{align*}
        \| D_{\TupleB,i,r_i}^{(t)} \|_F^2 \leq \frac{(E-1)(2E-1)}{3} \chi^2 (\eta^{(t)})^2 \left(\chi^2 M_{\B}^2 + 2 \mu^2 M_{\A}^2 (M_{\B}^4 + M_{\A}^4)\right).
    \end{align*}  
    Because $\sum_i p_i =1$, the same bounds hold for $\sum_i p_i \| D_{\TupleA,i,r_i}^{(t)} \|_F^2$ and $\sum_i p_i \| D_{\TupleB,i,r_i}^{(t)} \|_F^2$.
\end{proof}
\end{lem}

\begin{lem}
\label{lem:crossterm_gradient}
For each client $i \in [N]$ and round $t \in \{ 0,\cdots, T-1\}$, under \cref{assump:variance} and \cref{assump:normbound}, the following bounds hold:
    $\mathbb{E}_t \left\| \widebar{g}_{\TupleB,i,r_i}^{(t)} \, \widebar{g}_{\TupleA,i,r_i}^{(t)} \right\|_F^2 
    \;\leq\; M_{\A}^2 M_{\B}^2 \kappa^4.$
\begin{proof}
    By definition, $\widebar{g}_{\TupleB,i,r_i}^{(t)} \widebar{g}_{\TupleA,i,r_i}^{(t)}  =  \frac{1}{E^2} \sum_{e=0}^{E-1} \hat{g}_{\TupleB,i,r_i}^{(t,e)} \sum_{e=0}^{E-1} \hat{g}_{\TupleA,i,r_i}^{(t,e)}$, thus
    \begin{align*}
        & \mathbb{E}_t \Big\| \widebar{g}_{\TupleB,i,r_i}^{(t)} \widebar{g}_{\TupleA,i,r_i}^{(t)} \Big\|_F^2  = \mathbb{E}_t \left[ \left\| \frac{1}{E^2} \sum_{e=0}^{E-1} \hat{g}_{\TupleB,i,r_i}^{(t,e)} \sum_{e=0}^{E-1} \hat{g}_{\TupleA,i,r_i}^{(t,e)} \right\|_F^2 \right] \\
        & \stackrel{(a)}{\leq} \frac{1}{E^4}  E^2 \cdot \sum_{e_1,e_2} \mathbb{E}_t \left[ \left\| \hat{g}_{\TupleB,i,r_i}^{(t,e_1)}  \cdot  \hat{g}_{\TupleA,i,r_i}^{(t,e_2)}  \right\|_F^2 \right] \\
        & = \frac{1}{E^2} \cdot \sum_{e_1,e_2} \mathbb{E}_t \left[ \left\| \nabla_{\TupleUpdate} \Loss_{i}^{\text{full}}(\TupleUpdate_{i,r_i}^{(t,e_1)}; \xi_{i}^{(t,e_1)}) \cdot (\TupleA_{i,r_i}^{(t,e_1)})^\top \cdot  (\TupleB_{i,r_i}^{(t,e_2)})^\top \cdot \nabla_{\TupleUpdate} \Loss_{i}^{\text{full}}(\TupleUpdate_{i,r_i}^{(t,e_2)}; \xi_{i}^{(t,e_2)}) \right\| \right] \\
        & = \frac{1}{E^2} \sum_{e_1,e_2} \mathbb{E}_t \Bigg[ \sum_{l=1}^L \bigg\| \nabla_{\Update} \Loss_i^{\text{full}} (\Update_{i,l,r_{i,l}}^{(t,e_1)}; \xi_{i}^{(t,e_1)}) (\A_{i,l,r_{i,l}}^{(t,e_1)})^\top  \cdot \\
        & \quad \quad \quad \quad \quad \quad \quad \quad \quad \quad (\B_{i,l,r_{i,l}}^{(t,e_2)})^\top \nabla_{\Update} \Loss_i^{\text{full}} (\Update_{i,l,r_{i,l}}^{(t,e_2)}; \xi_{i}^{(t,e_2)}) \bigg\|_F^2 \Bigg] \\
        & \stackrel{(b)}{\leq} \frac{1}{E^2} \sum_{e_1,e_2} \mathbb{E}_t \Bigg[ \sum_{l=1}^L \left\| \nabla_{\Update} \Loss_i^{\text{full}} (\Update_{i,l,r_{i,l}}^{(t,e_1)}; \xi_{i}^{(t,e_1)}) \right\|_F^2 \cdot \left\| \A_{i,l,r_{i,l}}^{(t,e_1)} \right\|_2^2 \cdot \\
        & \quad \quad \quad \quad \quad \quad \quad \quad \quad \quad \left\| \B_{i,l,r_{i,l}}^{(t,e_2)} \right\|_2^2 \cdot \left\| \nabla_{\Update} \Loss_i^{\text{full}} (\Update_{i,l,r_{i,l}}^{(t,e_2)}; \xi_{i}^{(t,e_2)}) \right\|_F^2 \Bigg] \\
        & \stackrel{(c)}{\leq} \frac{M_{\A}^2 M_{\B}^2}{E^2} \sum_{e_1,e_2} \mathbb{E}_t \left[ \sum_{l=1}^L \left\| \nabla_{\Update} \Loss_i^{\text{full}} (\Update_{i,l,r_{i,l}}^{(t,e_1)}; \xi_{i}^{(t,e_1)}) \right\|_F^2 \cdot \left\| \nabla_{\Update} \Loss_i^{\text{full}} (\Update_{i,l,r_{i,l}}^{(t,e_2)}; \xi_{i}^{(t,e_2)}) \right\|_F^2 \right] \\
        & \stackrel{(d)}{\leq} \frac{M_{\A}^2 M_{\B}^2}{E^2} \sum_{e_1,e_2} \mathbb{E}_t \left[ \left\| \nabla_{\TupleUpdate} \Loss_i^{\text{full}} (\TupleUpdate_{i,r_i}^{(t,e_1)}; \xi_{i}^{(t,e_1)}) \right\|_F^2 \left\| \nabla_{\TupleUpdate} \Loss_i^{\text{full}} (\TupleUpdate_{i,r_i}^{(t,e_2)}; \xi_{i}^{(t,e_2)}) \right\|_F^2 \right] \\
        & \stackrel{(e)}{\leq} \frac{M_{\A}^2 M_{\B}^2}{E^2} \sum_{e_1,e_2} \sqrt{\mathbb{E}_t \left[ \left\| \nabla_{\TupleUpdate} \Loss_i^{\text{full}} (\TupleUpdate_{i,r_i}^{(t,e_1)}; \xi_{i}^{(t,e_1)}) \right\|_F^4 \right] \mathbb{E}_t \left[ \left\| \nabla_{\TupleUpdate} \Loss_i^{\text{full}} (\TupleUpdate_{i,r_i}^{(t,e_2)}; \xi_{i}^{(t,e_2)}) \right\|_F^4 \right]} \\
        & \stackrel{(f)}{\leq} \frac{M_{\A}^2 M_{\B}^2}{E^2} \sum_{e_1,e_2} \sqrt{\kappa^4 \cdot \kappa^4} \\
        & \leq M_{\A}^2 M_{\B}^2 \kappa^4,
    \end{align*}
    where in $(a)$ holds due to Jensen's inequality, in $(b)$ we employ $\|XY\|_F^2 \leq \|X\|_2^2 \cdot \|Y\|_F^2 \leq \|X\|_F^2 \cdot \|Y\|_F^2$, $(c)$ holds due to \cref{assump:normbound}, $(d)$ holds because $\sum x_l^2 y_l^2 \leq \sum x_l^2 \sum y_l^2$ for $x_l \geq 0$, $(e)$ holds due to Cauchy-Schwarz inequality, and $(f)$ uses \cref{assump:variance}.
\end{proof}
\end{lem}

\begin{lem}
\label{lem:average_gradients}
For each client $i \in [N]$ and round $t \in \{ 0,\cdots, T-1\}$, under \cref{assump:variance} and \cref{assump:normbound}, 
    $\mathbb{E}_t \left\| \widebar{g}_{\TupleA,i,r_i}^{(t)} \right\|_F^2 \leq M_{\B}^2 \chi^2$ \quad and \quad 
    $\mathbb{E}_t \left\| \widebar{g}_{\TupleB,i,r_i}^{(t)} \right\|_F^2 \leq M_{\A}^2 \chi^2.$
\begin{proof}
We prove $\mathbb{E}_t \| \widebar{g}_{\TupleA,i,r_i}^{(t)} \|_F^2 \leq M_\B^2 \chi^2$ below. The bound on $\widebar{g}_{\TupleB,i}^{(t)}$ is symmetric.
\begin{align*}
    \mathbb{E}_t \left\| \widebar{g}_{\TupleA,i,r_i}^{(t)} \right\|_F^2 & = \frac{1}{E^2} \mathbb{E}_t \left\| \sum_{e=0}^{E-1} \hat{g}_{\TupleA,i,r_i}^{(t,e)} \right\|_F^2 \\
    & \stackrel{(a)}{\leq} \frac{1}{E^2} E \sum_{e=0}^{E-1} \mathbb{E}_t \left\|  \hat{g}_{\TupleA,i,r_i}^{(t,e)} \right\|_F^2 \\
    & = \frac{1}{E}\sum_{e=0}^{E-1} \mathbb{E}_t \left\|  (\TupleB_{i,r_i}^{(t,e)})^\top \cdot \nabla_{\TupleUpdate} \Loss_{i}^{\text{full}}(\TupleUpdate_{i,r_i}^{(t,e)}; \xi_{i}^{(t,e)}) \right\|_F^2 \\
    & \stackrel{(b)}{\leq} \frac{1}{E} \sum_{e=0}^{E-1} \|\TupleB_{i,r_i}^{(t,e)}\|_2^2 \mathbb{E}_t \left\| \nabla_{\TupleUpdate} \Loss_{i}^{\text{full}}(\TupleUpdate_{i,r_i}^{(t,e)}; \xi_{i}^{(t,e)}) \right\|_F^2 \\
    & \stackrel{(c)}{\leq} M_{\B}^2 \chi^2,
\end{align*}
where $(a)$ is Jensen's inequality, $(b)$ uses $\|XY\|_F^2 \leq \|X\|_2^2 \|Y\|_F^2$, and $(c)$ applies \cref{assump:normbound} and \cref{assump:variance}. The bound on $\widebar{g}_{\TupleB,i,r_i}^{(t)}$ is identical with the roles of $\TupleA$ and $\TupleB$ swapped.
\end{proof}
\end{lem}

\subsection{Proof of \cref{thm:hetero_thm}}
\label[appendix]{proof:sefora}
\begin{proof}
    At round $t-1$ ($t\geq1$\footnote{Round $0$ uses the original initialization, where $\A_{l}^{(0)}$ with random Gaussian entries and $\B_{l}^{(0)}$ as the zero matrix.}), suppose the unsketching algorithm (\cref{alg:low_rank_app}) returns $\Bh_l^{(t)}$ and $\Ah_l^{(t)}$ for all layer $l \in [L]$. The corresponding global update is $\Update_l^{(t)}=\Bh_l^{(t)} \Ah_l^{(t)}$, with global rank $r_l$. Let $\Ah_l^{(t)}=\U_l^{(t)}\bSigma_l^{(t)}(\V_l^{(t)})^\top$ be its SVD. Since both $\Bh_l^{(t)}$ and $\U_l^{(t)}$ have orthonormal columns, $\Update_l^{(t)}=\Bh_l^{(t)} \U_l^{(t)}\bSigma_l^{(t)}(\V_l^{(t)})^\top$ is an SVD of $\Update_l^{(t)}$. We consequently define the global factors as $\B_l^{(t)}=\Bh_l^{(t)}\U_l^{(t)}(\bSigma_l^{(t)})^{\frac{1}{2}}$ and $\A_l^{(t)}=(\bSigma_l^{(t)})^{\frac{1}{2}}(\V_l^{(t)})^\top$, which is computed and broadcast by the federator. Since client $i$'s local rank $r_{i,l} \leq r_l$, according to line 4 of \cref{alg:hetsefora}, computing the initial local LoRA factors $\B_{i,l,r_{i,l}}^{(t,0)} \in \mathbb{R}^{d \times r_{i,l}}$ and $\A_{i,l,r_{i,l}}^{(t,0)} \in \mathbb{R}^{r_{i,l} \times s}$ is simply to perform truncation through forming the submatrices of the SVD decomposition of matrix $\Ah_l^{(t)}=\U_l^{(t)}\bSigma_l^{(t)}(\V_l^{(t)})^\top$, i.e., 
    \begin{align*}
        \B_{i,l,r_{i,l}}^{(t,0)} & \gets \Bh_l^{(t)} (\U_l^{(t)})_{:, 1:r_{i,l}} (\bSigma_l^{(t)})_{1:r_{i,l}, 1:r_{i,l}}^{1/2}, \\
        \A_{i,l,r_{i,l}}^{(t,0)} & \gets (\bSigma_l^{(t)})_{1:r_{i,l}, 1:r_{i,l}}^{1/2} (\V_l^{(t)})^\top_{1:r_{i,l}, :} \; .
    \end{align*}
    Whenever a quantity is indexed by both $i$ and $l$, its local rank is $r_{i,l}$; we suppress the rank subscript when no ambiguity arises.
    Thus, we have 
    \begin{align}
    \label{eq:truncation_operation}
        \B_{i,l}^{(t,0)} = \B_l^{(t)} \Trunc_{i,l} \;, \quad \A_{i,l}^{(t,0)} = (\Trunc_{i,l})^\top \A_{l}^{(t)},
    \end{align} 
    and $$\Update_{i,l}^{(t,0)}=\B_{i,l}^{(t,0)} \A_{i,l}^{(t,0)} = \B_l^{(t)} \Trunc_{i,l} (\Trunc_{i,l})^\top \A_{l}^{(t)} = \TruncOpe(\Update_{l}^{(t)}).$$

    Notice that, $\Trunc_{i,l}^\top \Trunc_{i,l} = \mathbf{I}_{r_{i,l}}$ and $\| \Trunc_{i,l} \Trunc_{i,l}^\top \|_2 = 1$. For $t=0$, \cref{eq:truncation_operation} also holds since each client obtains its initial local LoRA factors by truncating $\B_l^{(0)}$ and $\A_l^{(0)}$. Thus, the subsequent proof applies to every $t \in \{0,\cdots,T-1\}$.
    
    Let the global model produced by \hetalgname for round $t+1$ be $\TupleUpdate_{\hetalgname}^{(t+1)}$. Since the full objective is $\mu$-smooth with respect to $\TupleUpdate$, we have
    \begin{align}
    \label{eq:global_descent}
        &\mathbb{E}_t \left[ \Loss^{\text{full}} (\TupleUpdate_{\hetalgname}^{(t+1)}) \right] \\
        \nonumber \leq & \mathbb{E}_t \left[ \Loss^{\text{full}} (\TupleUpdate^{(t)}) \right]+ \underbrace{ \mathbb{E}_t \left[ \left\langle \nabla_{\TupleUpdate} \Loss^{\text{full}} (\TupleUpdate^{(t)}), \TupleUpdate_{\hetalgname}^{(t+1)} - \TupleUpdate^{(t)} \right\rangle \right]}_{\text{Term 1}} + \underbrace{\frac{\mu}{2} \mathbb{E}_t \left[ \| \TupleUpdate_{\hetalgname}^{(t+1)} - \TupleUpdate^{(t)} \|_F^2 \right]}_{\text{Term 2}},
    \end{align}
    Suppose the global LoRA factors at round $t$ are $\TupleA^{(t)}$ and $\TupleB^{(t)}$, under the rank-heterogeneous scenarios, each client $i$ first truncates the global LoRA factors to local rank $r_i$, i.e., $\llbracket \TupleA^{(t)} \rrbracket_{r_i}$ and $\llbracket \TupleB^{(t)} \rrbracket_{r_i}$, before proceeding with local training. The local fine-tuning starts from $(\TupleB_{i,r_i}^{(t,0)},\TupleA_{i,r_i}^{(t,0)})=(\llbracket \TupleB^{(t)} \rrbracket_{r_i}, \llbracket \TupleA^{(t)} \rrbracket_{r_i})$ and updates
    \begin{align*}
        \TupleB_{i,r_i}^{(t,e+1)} &= \TupleB_{i,r_i}^{(t,e)} - \eta^{(t)} \hat{g}_{\TupleB,i,r_i}^{(t,e)}, \quad \text{and} \\
        \TupleA_{i,r_i}^{(t,e+1)} &= \TupleA_{i,r_i}^{(t,e)} - \eta^{(t)} \hat{g}_{\TupleA,i,r_i}^{(t,e)}.
    \end{align*}
    Thus, 
    \begin{align*}
        \TupleB_{i,r_i}^{(t,E)} &= \TupleB_{i,r_i}^{(t,0)} - \eta^{(t)} \sum_{e=0}^{E-1} \hat{g}_{\TupleB,i,r_i}^{(t,e)} = \TupleB_{i,r_i}^{(t,0)} - \eta^{(t)} E \cdot \widebar g_{\TupleB,i,r_i}^{(t)}, \quad \text{and} \\
        \TupleA_{i,r_i}^{(t,E)} &= \TupleA_{i,r_i}^{(t,0)} - \eta^{(t)} \sum_{e=0}^{E-1} \hat{g}_{\TupleA,i,r_i}^{(t,e)} = \TupleA_{i,r_i}^{(t,0)} - \eta^{(t)} E \cdot \widebar g_{\TupleA,i,r_i}^{(t)}.
    \end{align*}
    Define $\mathbf{\mathtt{V}}^{(t+1)}= \TupleUpdate^{(t)} + \sum_{i \in [N]} p_i (\TupleUpdate_{i,r_i}^{(t,E)}-\TupleUpdate_{i,r_i}^{(t,0)})$. In round $t$ and for each layer $l \in [L]$, the federator aggregates $\Y_l^{(t+1)} = \Y_{\text{glo},l}^{(t)}+ \sum_{i \in [N]} p_i \Y^{(t)}_{i,l} \; \text{ and }\; 
    \Z_l^{(t+1)} = \Z_{\text{glo},l}^{(t)}+ \sum_{i \in [N]} p_i \Z^{(t)}_{i,l}$, which are two sketches of $\mathbf{V}_l^{(t+1)}$. Therefore, after unsketching, the updated model for round $t+1$ is $\TupleUpdate_{\hetalgname}^{(t+1)}=\mathbf{\mathtt{V}}^{(t+1)} + \mathcal{E}_{\text{sketch}}^{(t+1)}$, where $\mathcal{E}_{\text{sketch}}^{(t+1)}$ is the sketching error with $\mathbf{\mathtt{V}}^{(t+1)}$ being the target tuple. We have
    \begin{align*}
        \TupleUpdate_{\hetalgname}^{(t+1)}  & =  \mathbf{\mathtt{V}}^{(t+1)} + \mathcal{E}_{\text{sketch}}^{(t+1)} \\
        & =  \TupleUpdate^{(t)} + \sum_{i \in [N]} p_i (\TupleUpdate_{i,r_i}^{(t,E)}-\TupleUpdate_{i,r_i}^{(t,0)}) + \mathcal{E}_{\text{sketch}}^{(t+1)} 
    \end{align*}
    Define $\delta_{i,r_i}^{(t)} \triangleq \TupleUpdate_{i,r_i}^{(t,E)}-\TupleUpdate_{i,r_i}^{(t,0)}$, we have 
    \begin{align}
        \TupleUpdate_{\hetalgname}^{(t+1)} - \TupleUpdate^{(t)} & = \sum_{i \in [N]} p_i \delta_{i,r_i}^{(t)} + \mathcal{E}_{\text{sketch}}^{(t+1)},
    \end{align}
    and 
    \begin{align*}
        \delta_{i,r_i}^{(t)} & = \TupleB_{i,r_i}^{(t,E)} \TupleA_{i,r_i}^{(t,E)} - \TupleB_{i,r_i}^{(t,0)} \TupleA_{i,r_i}^{(t,0)} \\
        & =  (\TupleB_{i,r_i}^{(t,0)} - \eta^{(t)} E \cdot \widebar g_{\TupleB,i,r_i}^{(t)})(\TupleA_{i,r_i}^{(t,0)} - \eta^{(t)} E \cdot \widebar g_{\TupleA,i,r_i}^{(t)}) - \TupleB_{i,r_i}^{(t,0)} \TupleA_{i,r_i}^{(t,0)} \\
        & = - \eta^{(t)} E \cdot \TupleB_{i,r_i}^{(t,0)} \widebar g_{\TupleA,i,r_i}^{(t)} - \eta^{(t)} E \cdot \widebar g_{\TupleB,i,r_i}^{(t)} \TupleA_{i,r_i}^{(t,0)} + (\eta^{(t)})^2 E^2 \cdot \widebar g_{\TupleB,i,r_i}^{(t)} \widebar g_{\TupleA,i,r_i}^{(t)} \\
        & = - \eta^{(t)} E \cdot \TupleB^{(t)} \TupleTrunc_{i} \cdot \widebar g_{\TupleA,i,r_i}^{(t)} - \eta^{(t)} E \cdot \widebar g_{\TupleB,i,r_i}^{(t)} \cdot (\TupleTrunc_{i})^\top \TupleA^{(t)} + (\eta^{(t)})^2 E^2 \cdot \widebar g_{\TupleB,i,r_i}^{(t)} \widebar g_{\TupleA,i,r_i}^{(t)}
    \end{align*}
    \paragraph{For Term 1:}
    \begin{align*}
        \mathbb{E}_t \Bigg[ \bigg\langle \nabla_{\TupleUpdate} \Loss^{\text{full}} (\TupleUpdate^{(t)}),  \TupleUpdate_{\hetalgname}^{(t+1)} - \TupleUpdate^{(t)} \bigg\rangle \Bigg] & = \underbrace{\mathbb{E}_t \left[ \left\langle \nabla_{\TupleUpdate} \Loss^{\text{full}} (\TupleUpdate^{(t)}), \sum_{i \in [N]} p_i \delta_{i,r_i}^{(t)} \right\rangle \right]}_{\text{Part A}} \\
        & \quad + \underbrace{\mathbb{E}_t \left[ \left\langle \nabla_{\TupleUpdate} \Loss^{\text{full}} (\TupleUpdate^{(t)}), \mathcal{E}_{\text{sketch}}^{(t+1)} \right\rangle \right]}_{\text{Part B}} 
    \end{align*}
    We start by bounding Part A: 
    \begin{align*}
        & \mathbb{E}_t \Bigg[ \bigg\langle \nabla_{\TupleUpdate} \Loss^{\text{full}} (\TupleUpdate^{(t)}), \sum_{i \in [N]} p_i \delta_{i,r_i}^{(t)} \bigg\rangle \Bigg] \\
        & = - \eta^{(t)} E \cdot \sum_{i \in [N]} p_i \mathbb{E}_t \Bigg[ \bigg\langle \nabla_{\TupleUpdate} \Loss^{\text{full}} (\TupleUpdate^{(t)}), \TupleB^{(t)} \TupleTrunc_{i} \cdot \widebar g_{\TupleA,i,r_i}^{(t)} \bigg\rangle \Bigg] \\
        & \quad - \eta^{(t)} E \cdot \sum_{i \in [N]} p_i \mathbb{E}_t \Bigg[ \bigg\langle \nabla_{\TupleUpdate} \Loss^{\text{full}} (\TupleUpdate^{(t)}), \widebar g_{\TupleB,i,r_i}^{(t)} \cdot \TupleTrunc_{i}^\top \TupleA^{(t)} \bigg\rangle \Bigg] \\
        & \quad + (\eta^{(t)})^2 E^2 \cdot \sum_{i \in [N]} p_i \mathbb{E}_t \Bigg[ \bigg\langle \nabla_{\TupleUpdate} \Loss^{\text{full}} (\TupleUpdate^{(t)}), \widebar g_{\TupleB,i,r_i}^{(t)} \widebar g_{\TupleA,i,r_i}^{(t)} \bigg\rangle \Bigg] \\
        & = - \eta^{(t)} E \cdot \sum_{i \in [N]} p_i \bigg\langle \TupleTrunc_{i}^\top g_{\TupleA}^{(t)}, D_{\TupleA, i, r_i}^{(t)}+g_{\TupleA, i, r_i}^{(t)} \bigg\rangle - \eta^{(t)} E \cdot \sum_{i \in [N]} p_i \bigg\langle g_{\TupleB}^{(t)} \TupleTrunc_{i}, D_{\TupleB, i, r_i}^{(t)}+g_{\TupleB, i, r_i}^{(t)}  \bigg\rangle \\
        & \quad + (\eta^{(t)})^2 E^2 \cdot  \bigg\langle \nabla_{\TupleUpdate} \Loss^{\text{full}} (\TupleUpdate^{(t)}), \sum_{i \in [N]} p_i \mathbb{E}_t \Big[\widebar g_{\TupleB,i,r_i}^{(t)} \widebar g_{\TupleA,i,r_i}^{(t)} \Big]\bigg\rangle  \\
        & = - \eta^{(t)} E \cdot \sum_{i \in [N]} p_i \bigg\langle \TupleTrunc_{i}^\top g_{\TupleA}^{(t)}, g_{\TupleA, i, r_i}^{(t)} \bigg\rangle - \eta^{(t)} E \cdot \sum_{i \in [N]} p_i \bigg\langle g_{\TupleB}^{(t)} \TupleTrunc_{i}, g_{\TupleB, i, r_i}^{(t)}  \bigg\rangle \\
        & \quad - \eta^{(t)} E \cdot \sum_{i \in [N]} p_i \bigg\langle \TupleTrunc_{i}^\top g_{\TupleA}^{(t)}, D_{\TupleA, i, r_i}^{(t)} \bigg\rangle - \eta^{(t)} E \cdot \sum_{i \in [N]} p_i \bigg\langle g_{\TupleB}^{(t)} \TupleTrunc_{i}, D_{\TupleB, i, r_i}^{(t)} \bigg\rangle \\
        & \quad + (\eta^{(t)})^2 E^2 \cdot  \bigg\langle \nabla_{\TupleUpdate} \Loss^{\text{full}} (\TupleUpdate^{(t)}), \sum_{i \in [N]} p_i \mathbb{E}_t \Big[\widebar g_{\TupleB,i,r_i}^{(t)} \widebar g_{\TupleA,i,r_i}^{(t)} \Big]\bigg\rangle  
    \end{align*}
     where $(a)$ holds since the definition of the tuple inner product and the cyclic property of the trace, i.e., $$\left\langle \nabla_{\TupleUpdate} \Loss^{\text{full}} (\TupleUpdate^{(t)}), \TupleB^{(t)} \TupleTrunc_{i} \cdot \widebar g_{\TupleA,i,r_i}^{(t)} \right\rangle = \left\langle \TupleTrunc_{i}^\top g_{\TupleA}^{(t)}, \widebar g_{\TupleA,i,r_i}^{(t)} \right\rangle,$$ 
     $$\left\langle \nabla_{\TupleUpdate} \Loss^{\text{full}} (\TupleUpdate^{(t)}),  \widebar g_{\TupleB,i,r_i}^{(t)} \cdot \TupleTrunc_{i}^\top \TupleA^{(t)} \right\rangle = \left\langle g_{\TupleB}^{(t)} \TupleTrunc_{i}, \widebar g_{\TupleB,i,r_i}^{(t)} \right\rangle,$$
     and the definition of $D_{\TupleA}^{(t)}$ and $D_{\TupleB}^{(t)}$. We further write 
     $$g_{\TupleA, i, r_i}^{(t)}=(\TupleB_{i,r_i}^{(t,0)})^\top \nabla_{\TupleUpdate} \Loss_i^{\text{full}} ( \TupleUpdate_{i,r_i}^{(t,0)})= (\TupleB^{(t)} \TupleTrunc_{i})^\top \nabla_{\TupleUpdate} \Loss_i^{\text{full}} ( \TupleUpdate_{i,r_i}^{(t,0)}),$$ 
     $$g_{\TupleB, i, r_i}^{(t)}=\nabla_{\TupleUpdate} \Loss_i^{\text{full}} ( \TupleUpdate_{i,r_i}^{(t,0)}) (\TupleA_{i,r_i}^{(t,0)})^\top = \nabla_{\TupleUpdate} \Loss_i^{\text{full}} ( \TupleUpdate_{i,r_i}^{(t,0)})((\TupleTrunc_{i})^\top \TupleA^{(t)})^\top,$$ 
     and 
     $$\nabla_{\TupleUpdate} \Loss_i^{\text{full}} ( \TupleUpdate_{i,r_i}^{(t,0)}) = \nabla_{\TupleUpdate} \Loss^{\text{full}} ( \TupleUpdate^{(t)}) + (\nabla_{\TupleUpdate} \Loss_i^{\text{full}} ( \TupleUpdate_{i,r_i}^{(t,0)}) - \nabla_{\TupleUpdate} \Loss^{\text{full}} ( \TupleUpdate^{(t)})),$$
     thus, we have
     \begin{align*}
         \bigg\langle \TupleTrunc_{i}^\top g_{\TupleA}^{(t)}, g_{\TupleA, i, r_i}^{(t)} \bigg\rangle & = \bigg\langle \TupleTrunc_{i}^\top g_{\TupleA}^{(t)}, (\TupleB^{(t)} \TupleTrunc_{i})^\top \nabla_{\TupleUpdate} \Loss_i^{\text{full}} ( \TupleUpdate_{i,r_i}^{(t,0)}) \bigg\rangle\\
         & = \bigg\langle \TupleTrunc_{i}^\top g_{\TupleA}^{(t)},  \TupleTrunc_{i}^\top (\TupleB^{(t)})^\top \nabla_{\TupleUpdate} \Loss^{\text{full}} ( \TupleUpdate^{(t)}) \bigg\rangle \\
         & \quad + \bigg\langle \TupleTrunc_{i}^\top g_{\TupleA}^{(t)}, (\TupleB^{(t)} \TupleTrunc_{i})^\top (\nabla_{\TupleUpdate} \Loss_i^{\text{full}} ( \TupleUpdate_{i,r_i}^{(t,0)}) - \nabla_{\TupleUpdate} \Loss^{\text{full}} ( \TupleUpdate^{(t)})) \bigg\rangle \\
         & = \Big\|\TupleTrunc_{i}^\top g_{\TupleA}^{(t)} \Big\|_F^2 + \bigg\langle \TupleTrunc_{i}^\top g_{\TupleA}^{(t)}, (\TupleB^{(t)} \TupleTrunc_{i})^\top (\nabla_{\TupleUpdate} \Loss_i^{\text{full}} ( \TupleUpdate_{i,r_i}^{(t,0)}) - \nabla_{\TupleUpdate} \Loss^{\text{full}} ( \TupleUpdate^{(t)})) \bigg\rangle \\
         & = \Big\|\TupleTrunc_{i}^\top g_{\TupleA}^{(t)} \Big\|_F^2 + \bigg\langle \TupleTrunc_{i}^\top g_{\TupleA}^{(t)}, (\TupleB^{(t)} \TupleTrunc_{i})^\top (\nabla_{\TupleUpdate} \Loss_i^{\text{full}} ( \TupleUpdate_{i,r_i}^{(t,0)}) - \nabla_{\TupleUpdate} \Loss_i^{\text{full}} ( \TupleUpdate^{(t)})) \bigg\rangle \\
         & \quad + \bigg\langle \TupleTrunc_{i}^\top g_{\TupleA}^{(t)}, (\TupleB^{(t)} \TupleTrunc_{i})^\top (\nabla_{\TupleUpdate} \Loss_i^{\text{full}} ( \TupleUpdate^{(t)}) - \nabla_{\TupleUpdate} \Loss^{\text{full}} ( \TupleUpdate^{(t)})) \bigg\rangle \\
         \bigg\langle g_{\TupleB}^{(t)} \TupleTrunc_{i}, g_{\TupleB, i, r_i}^{(t)}  \bigg\rangle & = \Big\| g_{\TupleB}^{(t)} \TupleTrunc_{i} \Big\|_F^2 + \bigg\langle g_{\TupleB}^{(t)} \TupleTrunc_{i}, (\nabla_{\TupleUpdate} \Loss_i^{\text{full}} ( \TupleUpdate_{i,r_i}^{(t,0)}) - \nabla_{\TupleUpdate} \Loss_i^{\text{full}} ( \TupleUpdate^{(t)})) (\TupleA^{(t)})^\top \TupleTrunc_{i} \bigg\rangle \\
         & \quad + \bigg\langle g_{\TupleB}^{(t)} \TupleTrunc_{i}, (\nabla_{\TupleUpdate} \Loss_i^{\text{full}} ( \TupleUpdate^{(t)}) - \nabla_{\TupleUpdate} \Loss^{\text{full}} ( \TupleUpdate^{(t)})) (\TupleA^{(t)})^\top \TupleTrunc_{i} \bigg\rangle
    \end{align*}

    Define $q_{\min} \triangleq \min_{l \in [L]} \sum_{i: r_{i,l}=r_l} p_i$, with $q_{\min}>0$ since $r_l= \max_{i\in [N]}r_{i,l}$ guarantees that at least one client satisfies $r_l= r_{i,l}$ for each layer, and $p_i >0$. Consequently, we ensure $q_{\min}>0$. For client $i$ with $r_{i,l}=r_l$, $\Trunc_{i,l}=\mathbf{I}_{r_l}$. We derive two inequalities below: 
    \begin{align*}
        \sum_{i \in [N]} p_i \Big\|\TupleTrunc_{i}^\top g_{\TupleA}^{(t)} \Big\|_F^2 & \!\! = \!\!\sum_{l \in [L]} \sum_{i \in [N]} p_i \Big\|\Trunc_{i,l}^\top g_{\A,l}^{(t)} \Big\|_F^2 \!\!\geq \!\! \sum_{l \in [L]} \sum_{i: r_{i,l} = r_l} p_i \left\| \Trunc_{i,l}^\top g_{\A,l}^{(t)} \right\|_F^2 \!\!\geq \!\! \sum_{l \in [L]} \sum_{i: r_{i,l} = r_l} p_i \left\| g_{\A,l}^{(t)} \right\|_F^2 \\
        & \geq q_{\min} \sum_{l \in [L]} \left\| g_{\A,l}^{(t)} \right\|_F^2 = q_{\min} \left\| g_{\TupleA}^{(t)} \right\|_F^2,
    \end{align*}
    and similarly, $\sum_{i \in [N]} p_i \Big\| g_{\TupleB}^{(t)} \TupleTrunc_{i} \Big\|_F^2 \geq q_{\min} \left\| g_{\TupleB}^{(t)} \right\|_F^2.$
    By substituting the above terms into Part A, we have 
    \begin{align*}
        & \mathbb{E}_t \Bigg[ \bigg\langle \nabla_{\TupleUpdate} \Loss^{\text{full}} (\TupleUpdate^{(t)}), \sum_{i \in [N]} p_i \delta_{i,r_i}^{(t)} \bigg\rangle \Bigg] = - \eta^{(t)} E \cdot \sum_{i \in [N]} p_i \bigg( \Big\|\TupleTrunc_{i}^\top g_{\TupleA}^{(t)} \Big\|_F^2 + \Big\| g_{\TupleB}^{(t)} \TupleTrunc_{i} \Big\|_F^2 \bigg) \\
        & \quad - \eta^{(t)} E \cdot \sum_{i \in [N]} p_i \bigg\langle \TupleTrunc_{i}^\top g_{\TupleA}^{(t)}, (\TupleB^{(t)} \TupleTrunc_{i})^\top (\nabla_{\TupleUpdate} \Loss_i^{\text{full}} ( \TupleUpdate_{i,r_i}^{(t,0)}) - \nabla_{\TupleUpdate} \Loss_i^{\text{full}} ( \TupleUpdate^{(t)})) \bigg\rangle \\
        & \quad - \eta^{(t)} E \cdot \sum_{i \in [N]} p_i \bigg\langle \TupleTrunc_{i}^\top g_{\TupleA}^{(t)}, (\TupleB^{(t)} \TupleTrunc_{i})^\top (\nabla_{\TupleUpdate} \Loss_i^{\text{full}} ( \TupleUpdate^{(t)}) - \nabla_{\TupleUpdate} \Loss^{\text{full}} ( \TupleUpdate^{(t)})) \bigg\rangle \\
        & \quad - \eta^{(t)} E \cdot \sum_{i \in [N]} p_i \bigg\langle g_{\TupleB}^{(t)} \TupleTrunc_{i}, (\nabla_{\TupleUpdate} \Loss_i^{\text{full}} ( \TupleUpdate_{i,r_i}^{(t,0)}) - \nabla_{\TupleUpdate} \Loss_i^{\text{full}} ( \TupleUpdate^{(t)})) (\TupleA^{(t)})^\top \TupleTrunc_{i} \bigg\rangle \\
        & \quad - \eta^{(t)} E \cdot \sum_{i \in [N]} p_i \bigg\langle g_{\TupleB}^{(t)} \TupleTrunc_{i}, (\nabla_{\TupleUpdate} \Loss_i^{\text{full}} ( \TupleUpdate^{(t)}) - \nabla_{\TupleUpdate} \Loss^{\text{full}} ( \TupleUpdate^{(t)})) (\TupleA^{(t)})^\top \TupleTrunc_{i} \bigg\rangle \\
        & \quad - \eta^{(t)} E \cdot \sum_{i \in [N]} p_i \bigg\langle \TupleTrunc_{i}^\top g_{\TupleA}^{(t)}, D_{\TupleA, i, r_i}^{(t)} \bigg\rangle - \eta^{(t)} E \cdot \sum_{i \in [N]} p_i \bigg\langle g_{\TupleB}^{(t)} \TupleTrunc_{i}, D_{\TupleB, i, r_i}^{(t)} \bigg\rangle \\
        & \quad + (\eta^{(t)})^2 E^2 \cdot  \bigg\langle \nabla_{\TupleUpdate} \Loss^{\text{full}} (\TupleUpdate^{(t)}), \sum_{i \in [N]} p_i \mathbb{E}_t \Big[\widebar g_{\TupleB,i,r_i}^{(t)} \widebar g_{\TupleA,i,r_i}^{(t)} \Big]\bigg\rangle \\
        & \stackrel{(a)}{=} - \eta^{(t)} E \cdot \sum_{i \in [N]} p_i \bigg( \Big\|\TupleTrunc_{i}^\top g_{\TupleA}^{(t)} \Big\|_F^2 + \Big\| g_{\TupleB}^{(t)} \TupleTrunc_{i} \Big\|_F^2 \bigg) \\
        & \quad - \eta^{(t)} E \cdot \bigg\langle g_{\TupleA}^{(t)}, \sum_{i \in [N]} p_i \TupleTrunc_{i} \TupleTrunc_{i}^\top (\TupleB^{(t)})^\top (\nabla_{\TupleUpdate} \Loss_i^{\text{full}} ( \TupleUpdate_{i,r_i}^{(t,0)}) - \nabla_{\TupleUpdate} \Loss_i^{\text{full}} ( \TupleUpdate^{(t)})) \bigg\rangle \\
        & \quad - \eta^{(t)} E \cdot \bigg\langle g_{\TupleA}^{(t)}, \sum_{i \in [N]} p_i \TupleTrunc_{i} \TupleTrunc_{i}^\top (\TupleB^{(t)})^\top (\nabla_{\TupleUpdate} \Loss_i^{\text{full}} ( \TupleUpdate^{(t)}) - \nabla_{\TupleUpdate} \Loss^{\text{full}} ( \TupleUpdate^{(t)})) \bigg\rangle \\
        & \quad - \eta^{(t)} E \cdot  \bigg\langle g_{\TupleB}^{(t)}, \sum_{i \in [N]} p_i (\nabla_{\TupleUpdate} \Loss_i^{\text{full}} ( \TupleUpdate_{i,r_i}^{(t,0)}) - \nabla_{\TupleUpdate} \Loss_i^{\text{full}} ( \TupleUpdate^{(t)})) (\TupleA^{(t)})^\top \TupleTrunc_{i} \TupleTrunc_{i}^\top \bigg\rangle \\
        & \quad - \eta^{(t)} E \cdot \bigg\langle g_{\TupleB}^{(t)}, \sum_{i \in [N]} p_i (\nabla_{\TupleUpdate} \Loss_i^{\text{full}} ( \TupleUpdate^{(t)}) - \nabla_{\TupleUpdate} \Loss^{\text{full}} ( \TupleUpdate^{(t)})) (\TupleA^{(t)})^\top \TupleTrunc_{i} \TupleTrunc_{i}^\top \bigg\rangle \\
        & \quad - \eta^{(t)} E \cdot \bigg\langle g_{\TupleA}^{(t)}, \sum_{i \in [N]} p_i  \TupleTrunc_{i} D_{\TupleA, i, r_i}^{(t)} \bigg\rangle - \eta^{(t)} E \cdot \bigg\langle g_{\TupleB}^{(t)}, \sum_{i \in [N]} p_i D_{\TupleB, i, r_i}^{(t)} \TupleTrunc_{i}^\top \bigg\rangle \\
        & \quad + (\eta^{(t)})^2 E^2 \cdot  \bigg\langle \nabla_{\TupleUpdate} \Loss^{\text{full}} (\TupleUpdate^{(t)}), \sum_{i \in [N]} p_i \mathbb{E}_t \Big[\widebar g_{\TupleB,i,r_i}^{(t)} \widebar g_{\TupleA,i,r_i}^{(t)} \Big]\bigg\rangle \\
        & \stackrel{(b)}{\leq} - \eta^{(t)} E \cdot q_{\min} (\left\| g_{\TupleA}^{(t)} \right\|_F^2 + \left\| g_{\TupleB}^{(t)} \right\|_F^2) \\
        & \quad + \eta^{(t)} E  \bigg( \frac{q_{\min}}{6}\Big\| g_{\TupleA}^{(t)} \Big\|_F^2 \!\!+\! \frac{3}{2q_{\min}} \Big\| \sum_{i \in [N]} p_i \TupleTrunc_{i} \TupleTrunc_{i}^\top (\TupleB^{(t)})^\top (\nabla_{\TupleUpdate} \Loss_i^{\text{full}} ( \TupleUpdate_{i,r_i}^{(t,0)}) - \nabla_{\TupleUpdate} \Loss_i^{\text{full}} ( \TupleUpdate^{(t)})) \Big\|_F^2 \bigg) \\
        & \quad + \eta^{(t)} E \bigg( \frac{q_{\min}}{6}\Big\| g_{\TupleA}^{(t)} \Big\|_F^2 \!\!+\! \frac{3}{2q_{\min}} \Big\| \sum_{i \in [N]} p_i \TupleTrunc_{i} \TupleTrunc_{i}^\top (\TupleB^{(t)})^\top (\nabla_{\TupleUpdate} \Loss_i^{\text{full}} ( \TupleUpdate^{(t)}) - \nabla_{\TupleUpdate} \Loss^{\text{full}} ( \TupleUpdate^{(t)})) \Big\|_F^2 \bigg) \\
        & \quad + \eta^{(t)} E  \bigg( \frac{q_{\min}}{6}\Big\| g_{\TupleB}^{(t)} \Big\|_F^2 \!\!+\! \frac{3}{2q_{\min}} \Big\| \sum_{i \in [N]} p_i (\nabla_{\TupleUpdate} \Loss_i^{\text{full}} ( \TupleUpdate_{i,r_i}^{(t,0)}) - \nabla_{\TupleUpdate} \Loss_i^{\text{full}} ( \TupleUpdate^{(t)})) (\TupleA^{(t)})^\top \TupleTrunc_{i} \TupleTrunc_{i}^\top \Big\|_F^2 \bigg) \\
        & \quad + \eta^{(t)} E \bigg( \frac{q_{\min}}{6}\Big\| g_{\TupleB}^{(t)} \Big\|_F^2 \!\!+\! \frac{3}{2q_{\min}} \Big\| \sum_{i \in [N]} p_i (\nabla_{\TupleUpdate} \Loss_i^{\text{full}} ( \TupleUpdate^{(t)}) - \nabla_{\TupleUpdate} \Loss^{\text{full}} ( \TupleUpdate^{(t)})) (\TupleA^{(t)})^\top \TupleTrunc_{i} \TupleTrunc_{i}^\top \Big\|_F^2 \bigg) \\
        & \quad + \eta^{(t)} E \bigg( \frac{q_{\min}}{6}\Big\| g_{\TupleA}^{(t)} \Big\|_F^2 + \frac{3}{2q_{\min}} \Big\| \sum_{i \in [N]} p_i  \TupleTrunc_{i} D_{\TupleA, i, r_i}^{(t)} \Big\|_F^2 \bigg) \\
        & \quad + \eta^{(t)} E \bigg(  \frac{q_{\min}}{6}\Big\| g_{\TupleB}^{(t)} \Big\|_F^2 + \frac{3}{2q_{\min}} \Big\| \sum_{i \in [N]} p_i D_{\TupleB, i, r_i}^{(t)} \TupleTrunc_{i}^\top \Big\|_F^2 \bigg) \\
        & \quad + (\eta^{(t)})^2 E^2 \Big\| \nabla_{\TupleUpdate} \Loss^{\text{full}} (\TupleUpdate^{(t)}) \Big\|_F \bigg\| \sum_{i \in [N]} p_i \mathbb{E}_t \Big[\widebar g_{\TupleB,i,r_i}^{(t)} \widebar g_{\TupleA,i,r_i}^{(t)} \Big] \bigg\|_F \\
        & \stackrel{(c)}{\leq} - \eta^{(t)} E \cdot \frac{q_{\min}}{2} (\left\| g_{\TupleA}^{(t)} \right\|_F^2 + \left\| g_{\TupleB}^{(t)} \right\|_F^2) \\
        & \quad + \eta^{(t)} E \cdot \frac{3}{2q_{\min}} \sum_{i \in [N]} p_i \Big\| \TupleTrunc_{i} \TupleTrunc_{i}^\top (\TupleB^{(t)})^\top (\nabla_{\TupleUpdate} \Loss_i^{\text{full}} ( \TupleUpdate_{i,r_i}^{(t,0)}) - \nabla_{\TupleUpdate} \Loss_i^{\text{full}} ( \TupleUpdate^{(t)})) \Big\|_F^2  \\
        & \quad + \eta^{(t)} E \cdot \frac{3}{2q_{\min}} \sum_{i \in [N]} p_i \Big\| \TupleTrunc_{i} \TupleTrunc_{i}^\top (\TupleB^{(t)})^\top (\nabla_{\TupleUpdate} \Loss_i^{\text{full}} ( \TupleUpdate^{(t)}) - \nabla_{\TupleUpdate} \Loss^{\text{full}} ( \TupleUpdate^{(t)})) \Big\|_F^2 \\
        & \quad + \eta^{(t)} E \cdot \frac{3}{2q_{\min}} \sum_{i \in [N]} p_i \Big\| (\nabla_{\TupleUpdate} \Loss_i^{\text{full}} ( \TupleUpdate_{i,r_i}^{(t,0)}) - \nabla_{\TupleUpdate} \Loss_i^{\text{full}} ( \TupleUpdate^{(t)})) (\TupleA^{(t)})^\top \TupleTrunc_{i} \TupleTrunc_{i}^\top \Big\|_F^2 \\
        & \quad + \eta^{(t)} E \cdot \frac{3}{2q_{\min}} \sum_{i \in [N]} p_i \Big\| (\nabla_{\TupleUpdate} \Loss_i^{\text{full}} ( \TupleUpdate^{(t)}) - \nabla_{\TupleUpdate} \Loss^{\text{full}} ( \TupleUpdate^{(t)})) (\TupleA^{(t)})^\top \TupleTrunc_{i} \TupleTrunc_{i}^\top \Big\|_F^2 \\
        & \quad + \eta^{(t)} E \cdot \frac{3}{2q_{\min}} \sum_{i \in [N]} p_i  \Big\| \TupleTrunc_{i} D_{\TupleA, i, r_i}^{(t)} \Big\|_F^2 + \eta^{(t)} E \cdot \frac{3}{2q_{\min}} \sum_{i \in [N]} p_i \Big\| D_{\TupleB, i, r_i}^{(t)} \TupleTrunc_{i}^\top \Big\|_F^2 \\
        & \quad + (\eta^{(t)})^2 E^2 \cdot  \Big\| \nabla_{\TupleUpdate} \Loss^{\text{full}} (\TupleUpdate^{(t)}) \Big\|_F \cdot \sum_{i \in [N]} p_i \bigg\| \mathbb{E}_t \Big[\widebar g_{\TupleB,i,r_i}^{(t)} \widebar g_{\TupleA,i,r_i}^{(t)} \Big] \bigg\|_F \\
        & \stackrel{(d)}{\leq} - \eta^{(t)} E \cdot \frac{q_{\min}}{2} (\left\| g_{\TupleA}^{(t)} \right\|_F^2 + \left\| g_{\TupleB}^{(t)} \right\|_F^2) \\
        & \quad + \eta^{(t)} E \cdot \frac{3}{2q_{\min}} (M_{\B}^2 + M_{\A}^2) \mu^2 \sum_{i \in [N]} p_i \Big\| \TupleUpdate_{i,r_i}^{(t,0)} - \TupleUpdate^{(t)} \Big\|_F^2  \\
        & \quad + \eta^{(t)} E \cdot \frac{3}{2q_{\min}} (M_{\B}^2+M_{\A}^2) \sum_{i \in [N]} p_i \Big\| \nabla_{\TupleUpdate} \Loss_i^{\text{full}} ( \TupleUpdate^{(t)}) - \nabla_{\TupleUpdate} \Loss^{\text{full}} ( \TupleUpdate^{(t)}) \Big\|_F^2 \\
        & \quad + \eta^{(t)} E \cdot \frac{3}{2q_{\min}} \sum_{i \in [N]} p_i  \Big\| D_{\TupleA, i, r_i}^{(t)} \Big\|_F^2 + \eta^{(t)} E \cdot \frac{3}{2q_{\min}} \sum_{i \in [N]} p_i \Big\| D_{\TupleB, i, r_i}^{(t)} \Big\|_F^2 \\
        & \quad + (\eta^{(t)})^2 E^2 \cdot  \Big\| \nabla_{\TupleUpdate} \Loss^{\text{full}} (\TupleUpdate^{(t)}) \Big\|_F \cdot \sum_{i \in [N]} p_i \bigg\| \mathbb{E}_t \Big[\widebar g_{\TupleB,i,r_i}^{(t)} \widebar g_{\TupleA,i,r_i}^{(t)} \Big] \bigg\|_F \\
        & \stackrel{(e)}{\leq} - \eta^{(t)} E \cdot \frac{q_{\min}}{2} (\left\| g_{\TupleA}^{(t)} \right\|_F^2 + \left\| g_{\TupleB}^{(t)} \right\|_F^2) \\
        & \quad + \eta^{(t)} E \cdot \frac{3}{2q_{\min}} (M_{\B}^2 + M_{\A}^2) \mu^2 \sum_{i \in [N]} p_i \Big\| \TupleUpdate_{i,r_i}^{(t,0)} - \TupleUpdate^{(t)} \Big\|_F^2  \\
        & \quad + \eta^{(t)} E \cdot \frac{3}{2q_{\min}} (M_{\B}^2+M_{\A}^2) \zeta^2  \\
        & \quad + (\eta^{(t)})^3 \frac{(E-1)E(2E-1)}{3} \cdot \frac{3}{2q_{\min}}  \chi^2 \Big(\chi^2 (M_\A^2+ M_\B^2) + 2\mu^2 (M_\A^2+ M_\B^2) (M_\B^4 + M_\A^4)\Big) \\
        & \quad + (\eta^{(t)})^2 E^2 \cdot \chi \cdot M_{\A} M_{\B} \kappa^2,
    \end{align*}
    where $(a)$ holds since the definition of the tuple inner product and the cyclic property of the trace, the first term in $(b)$ holds since $ \sum_{i \in [N]} p_i \Big\|\TupleTrunc_{i}^\top g_{\TupleA}^{(t)} \Big\|_F^2 \geq q_{\min} \left\| g_{\TupleA}^{(t)} \right\|_F^2$ and $\sum_{i \in [N]} p_i \Big\| g_{\TupleB}^{(t)} \TupleTrunc_{i} \Big\|_F^2 \geq q_{\min} \left\| g_{\TupleB}^{(t)} \right\|_F^2$, the middle six terms hold due to Young's inequality, i.e., $\langle X,Y \rangle \leq \frac{\lambda}{2} \|X\|_F^2 + \frac{1}{2\lambda} \|Y\|_F^2$, and choosing $\lambda=\frac{q_{\min}}{3}$, and the last term holds since $\langle X,Y \rangle \leq |\langle X,Y \rangle| \leq \|X\|_F \|Y\|_F$, $(c)$ is obtained by rearranging the factors of $\left\| g_{\TupleA}^{(t)} \right\|_F^2 $ and $\left\| g_{\TupleB}^{(t)} \right\|_F^2$, using Jensen's inequality, $(d)$ holds by using $\|XY\|_F \leq \|X\|_2 \|Y\|_F$, $\Trunc_{i,l}^\top \Trunc_{i,l} = \mathbf{I}_{r_{i,l}}$, \cref{assump:normbound}, and $\| \Trunc_{i,l} \Trunc_{i,l}^\top \|_2 = 1$, $(e)$ is obtained by using \cref{assump:heterogeneity}, \cref{lem:drift} \cref{lem:crossterm_gradient}, and 
    \begin{align*}
        \nonumber \left\| \nabla_{\TupleUpdate} \Loss^{\text{full}} (\TupleUpdate^{(t)}) \right\|_F^2 & = \left\| \sum_{i \in [N]} p_i \nabla_{\TupleUpdate} \Loss_i^{\text{full}} (\TupleUpdate^{(t)}) \right\|_F^2 \\
        \nonumber & \leq \sum_{i \in [N]} p_i \left\| \nabla_{\TupleUpdate} \Loss_i^{\text{full}} (\TupleUpdate^{(t)}) \right\|_F^2 \\
        \nonumber & = \sum_{i \in [N]} p_i  \left\| \mathbb{E}_t \left[ \nabla_{\TupleUpdate} \Loss_i^{\text{full}} (\TupleUpdate^{(t)};\xi_i) \right] \right\|_F^2 \\
        \nonumber & \leq \sum_{i \in [N]} p_i  \mathbb{E}_t \left[ \left\| \nabla_{\TupleUpdate} \Loss_i^{\text{full}} (\TupleUpdate^{(t)};\xi_i) \right\|_F^2\right] \\
        & \leq \chi^2,
    \end{align*}
    We now derive the bound for $\sum_{i \in [N]} p_i \Big\| \TupleUpdate_{i,r_i}^{(t,0)} - \TupleUpdate^{(t)} \Big\|_F^2$. For layer $l \in [L]$:
    \begin{align*}
        & \Update_{i,l}^{(t,0)} - \Update_l^{(t)} = \B_{i,l}^{(t,0)} \A_{i,l}^{(t,0)}  - \B_l^{(t)} \A_l^{(t)} = \B_l^{(t)} \Trunc_{i,l} (\Trunc_{i,l})^\top \A_{l}^{(t)} - \B_l^{(t)} \A_l^{(t)} \\
        & = - \B_l^{(t)} (\mathbf{I}_{r_l} - \Trunc_{i,l} (\Trunc_{i,l})^\top) \A_{l}^{(t)} = - \Bh_l^{(t)} \U_l^{(t)} (\bSigma_l^{(t)})^{\frac{1}{2}} (\mathbf{I}_{r_l} - \Trunc_{i,l} (\Trunc_{i,l})^\top) (\bSigma_l^{(t)})^{\frac{1}{2}} (\V_l^{(t)})^\top \\
        & = - \Bh_l^{(t)} \U_l^{(t)} (\mathbf{I}_{r_l} - \Trunc_{i,l} (\Trunc_{i,l})^\top) \bSigma_l^{(t)} (\V_l^{(t)})^\top
    \end{align*}
    Consequently, 
    \begin{align*}
        \Big\| \Update_{i,l}^{(t,0)} - \Update_l^{(t)} \Big\|_F^2 & =\Big\| - \Bh_l^{(t)} \U_{l,\text{tail}}^{(t)} \bSigma_{l,\text{tail}}^{(t)} (\V_{l,\text{tail}}^{(t)})^\top \Big\|_F^2 \\
        & = \mathrm{tr} \Big( \V_{l,\text{tail}}^{(t)} (\bSigma_{l,\text{tail}}^{(t)})^\top (\U_{l,\text{tail}}^{(t)})^\top (\Bh_l^{(t)})^\top \Bh_l^{(t)} \U_{l,\text{tail}}^{(t)} \bSigma_{l,\text{tail}}^{(t)} (\V_{l,\text{tail}}^{(t)})^\top \Big) \\
        & \stackrel{(a)}{=} \mathrm{tr} \Big( \V_{l,\text{tail}}^{(t)} (\bSigma_{l,\text{tail}}^{(t)})^2 (\V_{l,\text{tail}}^{(t)})^\top \Big) = \mathrm{tr} \Big( (\bSigma_{l,\text{tail}}^{(t)})^2 (\V_{l,\text{tail}}^{(t)})^\top \V_{l,\text{tail}}^{(t)} \Big) \\
        & = \mathrm{tr} \Big( (\bSigma_{l,\text{tail}}^{(t)})^2 \Big) = \sum_{j=r_{i,l}+1}^{r_l} \big( \sigma_{l,j}^{(t)} \big)^2,
    \end{align*}
    where the subscript ``tail'' represents the columns/rows that correspond to rank $r_{i,l} < j \leq r_l$, $\sigma_{l,j}^{(t)}$ is the $j$-th singular value of $\Update_l^{(t)}$, and $(a)$ holds since both $\Bh_l^{(t)}$ and $\U_{l,\text{tail}}^{(t)}$ have orthonormal columns. Hence, $\sum_{i \in [N]} p_i \Big\| \TupleUpdate_{i,r_i}^{(t,0)} - \TupleUpdate^{(t)} \Big\|_F^2 = \sum_{i \in [N]} p_i \sum_{l \in [L]} \sum_{j=r_{i,l}+1}^{r_l} \big( \sigma_{l,j}^{(t)} \big)^2$.

    For Part B, we have
    \begin{align*}
        \mathbb{E}_t \left[ \left\langle \nabla_{\TupleUpdate} \Loss^{\text{full}} (\TupleUpdate^{(t)}), \mathcal{E}_{\text{sketch}}^{(t+1)} \right\rangle \right] \leq \Big\| \nabla_{\TupleUpdate} \Loss^{\text{full}} (\TupleUpdate^{(t)}) \Big\|_F \mathbb{E}_t \Big\| \mathcal{E}_{\text{sketch}}^{(t+1)} \Big\|_F  \leq \chi \mathbb{E}_t \Big\| \mathcal{E}_{\text{sketch}}^{(t+1)} \Big\|_F.
    \end{align*}
    By combining the previous bounds on Part A and Part B, Term 1 is bounded as:
    \begin{align*}
        & \mathbb{E}_t \Bigg[ \bigg\langle \nabla_{\TupleUpdate} \Loss^{\text{full}} (\TupleUpdate^{(t)}),  \TupleUpdate_{\hetalgname}^{(t+1)} - \TupleUpdate^{(t)} \bigg\rangle \Bigg] \leq - \eta^{(t)} E \cdot \frac{q_{\min}}{2} (\left\| g_{\TupleA}^{(t)} \right\|_F^2 + \left\| g_{\TupleB}^{(t)} \right\|_F^2) \\
        & \quad + \eta^{(t)} E \cdot \frac{3}{2q_{\min}} (M_{\B}^2 + M_{\A}^2) \mu^2 \sum_{i \in [N]} p_i \sum_{l \in [L]} \sum_{j=r_{i,l}+1}^{r_l} \big( \sigma_{l,j}^{(t)} \big)^2  \\
        & \quad + \eta^{(t)} E \cdot \frac{3}{2q_{\min}} (M_{\B}^2+M_{\A}^2) \zeta^2  \\
        & \quad + (\eta^{(t)})^3 \frac{(E-1)E(2E-1)}{3} \cdot \frac{3}{2q_{\min}}  \chi^2 \Big(\chi^2 (M_\A^2+ M_\B^2) + 2\mu^2 (M_\A^2+ M_\B^2) (M_\B^4 + M_\A^4)\Big) \\
        & \quad + (\eta^{(t)})^2 E^2 \cdot \chi \cdot M_{\A} M_{\B} \kappa^2 + \chi \cdot \mathbb{E}_t \Big\| \mathcal{E}_{\text{sketch}}^{(t+1)} \Big\|_F.
    \end{align*}

    \paragraph{For Term 2:}
    \begin{align*}
        & \mathbb{E}_t \left[ \| \TupleUpdate_{\hetalgname}^{(t+1)} - \TupleUpdate^{(t)} \|_F^2 \right] \\
        & = \mathbb{E}_t \left[ \bigg\| \sum_{i \in [N]} p_i \Big( - \!\eta^{(t)} E \TupleB^{(t)} \TupleTrunc_{i} \widebar g_{\TupleA,i,r_i}^{(t)} \!-\! \eta^{(t)} E \widebar g_{\TupleB,i,r_i}^{(t)} (\TupleTrunc_{i})^\top \TupleA^{(t)} \!+\! (\eta^{(t)})^2 E^2 \widebar g_{\TupleB,i,r_i}^{(t)} \widebar g_{\TupleA,i,r_i}^{(t)} \Big) \!+\! \mathcal{E}_{\text{sketch}}^{(t+1)} \bigg\|_F^2 \right] \\
        & \stackrel{(a)}{\leq} 4 (\eta^{(t)})^2 E^2 \cdot \sum_{i \in [N]} p_i \mathbb{E}_t \left[ \bigg\| \TupleB^{(t)} \TupleTrunc_{i} \cdot \widebar g_{\TupleA,i,r_i}^{(t)} \bigg\|_F^2 \right] + 4 (\eta^{(t)})^2 E^2 \cdot \sum_{i \in [N]} p_i \mathbb{E}_t \left[ \bigg\| \widebar g_{\TupleB,i,r_i}^{(t)} \cdot (\TupleTrunc_{i})^\top \TupleA^{(t)} \bigg\|_F^2 \right] \\
        & \quad + 4 (\eta^{(t)})^4 E^4 \sum_{i \in [N]} p_i \mathbb{E}_t \left[ \bigg\| \widebar g_{\TupleB,i,r_i}^{(t)} \widebar g_{\TupleA,i,r_i}^{(t)} \bigg\|_F^2 \right] + 4\mathbb{E}_t \left[ \bigg\|\mathcal{E}_{\text{sketch}}^{(t+1)} \bigg\|_F^2 \right] \\
        & \stackrel{(b)}{\leq} 4 (\eta^{(t)})^2 E^2 \sum_{i \in [N]} p_i \mathbb{E}_t \left[ \bigg\| \TupleB^{(t)} \TupleTrunc_{i} \bigg\|_2^2 \bigg\| \widebar g_{\TupleA,i,r_i}^{(t)} \bigg\|_F^2 \right] + 4 (\eta^{(t)})^2 E^2 \sum_{i \in [N]} p_i \mathbb{E}_t \left[ \bigg\| \widebar g_{\TupleB,i,r_i}^{(t)} \bigg\|_F^2 \bigg\| (\TupleTrunc_{i})^\top \TupleA^{(t)} \bigg\|_2^2 \right] \\
        & \quad + 4 (\eta^{(t)})^4 E^4 \sum_{i \in [N]} p_i \mathbb{E}_t \left[ \bigg\| \widebar g_{\TupleB,i,r_i}^{(t)} \widebar g_{\TupleA,i,r_i}^{(t)} \bigg\|_F^2 \right] + 4\mathbb{E}_t \left[ \bigg\|\mathcal{E}_{\text{sketch}}^{(t+1)} \bigg\|_F^2 \right] \\
        & \stackrel{(c)}{\leq} 4 (\eta^{(t)})^2 E^2 \cdot (M_{\A}^4+M_{\B}^4) \chi^2 + 4 (\eta^{(t)})^4 E^4 M_{\A}^2 M_{\B}^2 \kappa^4 + 4\mathbb{E}_t \left[ \bigg\|\mathcal{E}_{\text{sketch}}^{(t+1)} \bigg\|_F^2 \right] \\
    \end{align*}
    where $(a)$ holds by using $\|W + X + Y +Z\|_F^2 \leq 4 ( \| W \|_F^2 + \| X \|_F^2 + \| Y \|_F^2+ \| Z \|_F^2)$ and Jensen's inequality, $(b)$ uses $\| XY \|_F \leq \|X\|_2 \| Y \|_F$, $(c)$ holds since we have $\big\| \B_l^{(t)} \Trunc_{i,l} \big\|_2 \leq \big\| \B_l^{(t)} \big\|_2 \big\| \Trunc_{i,l} \big\|_2 \leq M_{\B}$, \cref{lem:crossterm_gradient} and \cref{lem:average_gradients}.

    By combining the bounds on Term 1 and Term 2 into \cref{eq:global_descent}, we have
    \begin{align*}
        & \mathbb{E}_t \left[ \Loss^{\text{full}} (\TupleUpdate_{\hetalgname}^{(t+1)}) \right] - \mathbb{E}_t \left[ \Loss^{\text{full}} (\TupleUpdate^{(t)}) \right] \\
        & \leq \mathbb{E}_t \left[ \left\langle \nabla_{\TupleUpdate} \Loss^{\text{full}} (\TupleUpdate^{(t)}), \TupleUpdate_{\hetalgname}^{(t+1)} - \TupleUpdate^{(t)} \right\rangle \right] + \frac{\mu}{2} \mathbb{E}_t \left[ \| \TupleUpdate_{\hetalgname}^{(t+1)} - \TupleUpdate^{(t)} \|_F^2 \right] \\
        & \leq - \eta^{(t)} E \cdot \frac{q_{\min}}{2} (\left\| g_{\TupleA}^{(t)} \right\|_F^2 + \left\| g_{\TupleB}^{(t)} \right\|_F^2) \\
        & \quad + \eta^{(t)} E \cdot \frac{3}{2q_{\min}} (M_{\B}^2 + M_{\A}^2) \Big( \mu^2 \sum_{i \in [N]} p_i \sum_{l \in [L]} \sum_{j=r_{i,l}+1}^{r_l} \big( \sigma_{l,j}^{(t)} \big)^2 + \zeta^2 \Big) \\
        & \quad + (\eta^{(t)})^2 E^2 \cdot \chi \cdot M_{\A} M_{\B} \kappa^2 + (\eta^{(t)})^2 E^2 \cdot 2\mu (M_{\A}^4+M_{\B}^4) \chi^2 \\
        & \quad + (\eta^{(t)})^3 E^3 \frac{1}{q_{\min}} \chi^2 \Big(\chi^2 (M_\A^2+ M_\B^2) + 2\mu^2 (M_\A^2+ M_\B^2) (M_\B^4 + M_\A^4)\Big) \\
        & \quad + (\eta^{(t)})^4 E^4 \cdot 2\mu M_{\A}^2 M_{\B}^2 \kappa^4  + \chi \cdot \mathbb{E}_t \Big\| \mathcal{E}_{\text{sketch}}^{(t+1)} \Big\|_F + 2\mu \mathbb{E}_t \left[ \bigg\|\mathcal{E}_{\text{sketch}}^{(t+1)} \bigg\|_F^2 \right].
    \end{align*}
    Therefore, 
    \begin{align*}
        & \left\| g_{\TupleA}^{(t)} \right\|_F^2 + \left\| g_{\TupleB}^{(t)} \right\|_F^2 \leq \frac{2}{\eta^{(t)} E q_{\min}} ( \mathbb{E}_t \left[ \Loss^{\text{full}} (\TupleUpdate^{(t)}) \right] -\mathbb{E}_t \left[ \Loss^{\text{full}} (\TupleUpdate_{\hetalgname}^{(t+1)}) \right]) \\
        & \quad + \frac{6}{2q_{\min}^2} (M_{\B}^2 + M_{\A}^2) \Big( \mu^2 \sum_{i \in [N]} p_i \sum_{l \in [L]} \sum_{j=r_{i,l}+1}^{r_l} \big( \sigma_{l,j}^{(t)} \big)^2 + \zeta^2 \Big) \\
        & \quad + (\eta^{(t)}) E \cdot \frac{2}{q_{\min}} \Big( \chi M_{\A} M_{\B} \kappa^2 + 2\mu (M_{\A}^4+M_{\B}^4) \chi^2 \Big) \\
        & \quad + (\eta^{(t)})^2 E^2 \cdot \frac{2}{q_{\min}^2} \chi^2 \Big(\chi^2 (M_\A^2+ M_\B^2) + 2\mu^2 (M_\A^2+ M_\B^2) (M_\B^4 + M_\A^4)\Big) \\
        & \quad + (\eta^{(t)})^3 E^3 \cdot \frac{4}{q_{\min}} \mu M_{\A}^2 M_{\B}^2 \kappa^4 + \frac{2}{\eta^{(t)} E q_{\min}} \bigg( \chi \cdot \mathbb{E}_t \Big\| \mathcal{E}_{\text{sketch}}^{(t+1)} \Big\|_F + 2\mu \mathbb{E}_t \Big\|\mathcal{E}_{\text{sketch}}^{(t+1)} \Big\|_F^2  \bigg).
    \end{align*}

    By taking constant learning rate $\eta^{(t)}=\eta$, taking total expectation $\mathbb{E}[\cdot]$ on both sides, summing this inequality from $t=0$ to $T-1$, computing the average, and defining $\Delta_0 = \mathbb{E} \left[ \Loss^{\text{full}} (\TupleUpdate^{(0)}) \right] - \Loss^{\text{full}, \star} $, we have
    \begin{align*}
        & \frac{1}{T} \sum_{t=0}^{T-1} \left( \mathbb{E} \left[ \left\| \nabla_{\TupleA} \Loss^{\text{lora}} (\TupleB^{(t)}, \TupleA^{(t)}) \right\|_F^2 \right] + \mathbb{E} \left[ \left\| \nabla_{\TupleB} \Loss^{\text{lora}} (\TupleB^{(t)}, \TupleA^{(t)}) \right\|_F^2 \right] \right) \\
        & \leq \frac{2\Delta_0}{\eta E T q_{\min}} + \frac{1}{T} \sum_{t=0}^{T-1} \frac{2}{\eta E q_{\min}} \mathbb{E} \bigg[ \chi \cdot \mathbb{E}_t \Big\| \mathcal{E}_{\text{sketch}}^{(t+1)} \Big\|_F + 2\mu \mathbb{E}_t \Big\|\mathcal{E}_{\text{sketch}}^{(t+1)} \Big\|_F^2  \bigg] \\
        & \quad + \frac{6}{2q_{\min}^2} (M_{\B}^2 + M_{\A}^2) \bigg( \mu^2 \frac{1}{T} \sum_{t=0}^{T-1} \mathbb{E} \Big[ \sum_{i \in [N]} p_i \sum_{l \in [L]} \sum_{j=r_{i,l}+1}^{r_l} \big( \sigma_{l,j}^{(t)} \big)^2 \Big] + \zeta^2 \bigg) \\
        & \quad + \eta E \cdot \frac{2}{q_{\min}} \Big( \chi M_{\A} M_{\B} \kappa^2 + 2\mu (M_{\A}^4+M_{\B}^4) \chi^2 \Big) \\
        & \quad + \eta^2 E^2 \frac{2}{q_{\min}^2} \chi^2 \Big(\chi^2 (M_\A^2+ M_\B^2) + 2\mu^2 (M_\A^2+ M_\B^2) (M_\B^4 + M_\A^4)\Big) + \eta^3 E^3 \frac{4}{q_{\min}} \mu M_{\A}^2 M_{\B}^2 \kappa^4.
    \end{align*}

    Recall that $\mathcal E_{\text{sketch}}^{(t+1)}=\TupleUpdate_{\hetalgname}^{(t+1)}-\mathbf{\mathtt{V}}^{(t+1)}$. Let $C_1=M_{\B}^2 + M_{\A}^2$, $C_2 = \chi M_{\A} M_{\B} \kappa^2 + 2\mu (M_{\A}^4+M_{\B}^4) \chi^2$, $C_3 = \chi^2 \Big(\chi^2 (M_\A^2+ M_\B^2) + 2\mu^2 (M_\A^2+ M_\B^2) (M_\B^4 + M_\A^4)\Big)$, $C_4 = \mu M_{\A}^2 M_{\B}^2 \kappa^4$, the sketching floor $S_1^{(t+1)} = \chi \cdot \mathbb{E}_t \Big\| \mathcal{E}_{\text{sketch}}^{(t+1)} \Big\|_F + 2\mu \mathbb{E}_t \Big\|\mathcal{E}_{\text{sketch}}^{(t+1)} \Big\|_F^2$, $\widebar{S_1} = \frac{1}{T} \sum_{t=0}^{T-1} \mathbb{E} \left[ S_1^{(t+1)} \right]$,
    and $\Delta_0 = \mathbb{E} \left[ \Loss^{\text{full}} (\TupleUpdate^{(0)}) \right] - \Loss^{\text{full}, \star}$. \hetalgname satisfies:
    \begin{align*}
        & \frac{1}{T} \sum_{t=0}^{T-1} \left( \mathbb{E} \left[ \left\| \nabla_{\TupleA} \Loss^{\text{lora}} (\TupleB^{(t)}, \TupleA^{(t)}) \right\|_F^2 \right] + \mathbb{E} \left[ \left\| \nabla_{\TupleB} \Loss^{\text{lora}} (\TupleB^{(t)}, \TupleA^{(t)}) \right\|_F^2 \right] \right)  \\
        & \leq \frac{2\Delta_0}{\eta E T q_{\min}} + \frac{2}{\eta E q_{\min}} \widebar{S_1} + \eta E \cdot \frac{2C_2}{q_{\min}} + \eta^2 E^2 \cdot \frac{2C_3}{q_{\min}^2}  + \eta^3 E^3 \cdot \frac{4C_4}{q_{\min}} \\
        & \quad + \frac{6C_1}{2q_{\min}^2} \bigg( \mu^2 \frac{1}{T} \sum_{t=0}^{T-1} \mathbb{E} \Big[ \sum_{i \in [N]} p_i \sum_{l \in [L]} \sum_{j=r_{i,l}+1}^{r_l} \big( \sigma_{l,j}^{(t)} \big)^2 \Big] + \zeta^2 \bigg).
    \end{align*}
    This concludes the proof of \cref{thm:hetero_thm}
    
\end{proof}

\subsection{Proof of \cref{cor:tail}}
\label[appendix]{app:cor}
\begin{proof}
    Recall that $\mathcal E_{\text{sketch}}^{(t+1)}=\TupleUpdate_{\hetalgname}^{(t+1)}-\mathbf{\mathtt{V}}^{(t+1)}$, where $\mathbf{\mathtt{V}}^{(t+1)} \triangleq \TupleUpdate^{(t)} + \sum_{i \in [N]} p_i (\TupleUpdate_{i,r_i}^{(t,E)}-\TupleUpdate_{i,r_i}^{(t,0)})$. According to \cref{lem:low_rank_error}, the sketching error is bounded as:
    \begin{align*}
        \mathbb{E}_t \| \mathcal{E}_{\text{sketch}}^{(t+1)} \|_F^2 & = \sum_{l=1}^L \mathbb{E}_t \| \mathcal{E}_{\text{sketch},l}^{(t+1)} \|_F^2\\
        & \leq \sum_{l=1}^L (1+\frac{r_l}{k_l-r_l-1}) \cdot \min_{\varrho'_l < r_l-1} (1+\frac{\varrho'_l}{r_l-\varrho'_l-1}) \cdot \mathbb{E}_t \left[ \tau_{\varrho'_l+1}^2(\mathbf{V}_l^{(t+1)}) \right] \\
        & \leq \sum_{l=1}^L c_l \cdot \mathbb{E}_t \left[ \tau_{\varrho_l+1}^2(\mathbf{V}_l^{(t+1)}) \right],
    \end{align*}
    where the second inequality holds by evaluating at the fixed $\varrho_l$ from \cref{assump:tail}, and $c_l \triangleq (1+\frac{r_l}{k_l-r_l-1}) \cdot (1+\frac{\varrho_l}{r_l-\varrho_l-1})$.

    Define $\epsilon^2 \triangleq \sum_{l=1}^L c_l \cdot \widebar{\tau_l}^2$. Averaging over the $T$ communication rounds and invoking \cref{assump:tail}, we obtain
    \begin{align*}
        \frac{1}{T} \sum_{t=0}^{T-1} \mathbb{E} \left[ \| \mathcal{E}_{\text{sketch}}^{(t+1)} \|_F^2 \right] \leq \frac{1}{T} \sum_{t=0}^{T-1} \sum_{l=1}^L c_l \cdot \mathbb{E}_t \left[ \tau_{\varrho_l+1}^2(\mathbf{V}_l^{(t+1)}) \right] \leq \sum_{l=1}^L c_l \cdot \widebar{\tau_l}^2 = \epsilon^2.
    \end{align*}
        
    Thus, the sketching floor is
    \begin{align*}
        & \widebar{S_1} = \frac{1}{T} \sum_{t=0}^{T-1} \mathbb{E}\left[S_1^{(t+1)}\right] = \frac{1}{T} \sum_{t=0}^{T-1} \mathbb{E}\left[ \chi \cdot \mathbb{E}_t \Big\| \mathcal{E}_{\text{sketch}}^{(t+1)} \Big\|_F + 2\mu \mathbb{E}_t \Big\|\mathcal{E}_{\text{sketch}}^{(t+1)} \Big\|_F^2 \right] \\
        & \leq \chi \cdot \frac{1}{T} \sum_{t=0}^{T-1} \sqrt{\mathbb{E}\left[ \left\| \mathcal{E}_{\text{sketch}}^{(t+1)} \right\|_F^2 \right] } + 2\mu \cdot \frac{1}{T} \sum_{t=0}^{T-1} \mathbb{E} \left[ \| \mathcal{E}_{\text{sketch}}^{(t+1)} \|_F^2 \right] \\
        & \leq \chi \cdot \sqrt{\frac{1}{T} \sum_{t=0}^{T-1} \mathbb{E}\left[ \left\| \mathcal{E}_{\text{sketch}}^{(t+1)} \right\|_F^2 \right] } + 2\mu \cdot \frac{1}{T} \sum_{t=0}^{T-1} \mathbb{E} \left[ \| \mathcal{E}_{\text{sketch}}^{(t+1)} \|_F^2 \right] \\
        & \leq \chi \epsilon + 2\mu \epsilon^2,
    \end{align*}
    where the inequality holds due to Jensen's inequality.
    
    By substituting this bound into the result of \cref{thm:hetero_thm}, we obtain 
    \begin{align*}
        & \frac{1}{T} \sum_{t=0}^{T-1} \left( \mathbb{E} \left[ \left\| \nabla_{\TupleA} \Loss^{\text{lora}} (\TupleB^{(t)}, \TupleA^{(t)}) \right\|_F^2 \right] + \mathbb{E} \left[ \left\| \nabla_{\TupleB} \Loss^{\text{lora}} (\TupleB^{(t)}, \TupleA^{(t)}) \right\|_F^2 \right] \right) \\
        & \leq \frac{2\Delta_0}{\eta E T q_{\min}} + \frac{2}{\eta E q_{\min}} ( \chi \epsilon + 2\mu \epsilon^2) + \eta E \cdot \frac{2C_2}{q_{\min}} + \eta^2 E^2 \cdot \frac{2C_3}{q_{\min}^2} + \eta^3 E^3 \cdot \frac{4C_4}{q_{\min}} \\
        & \quad + \frac{6C_1}{2q_{\min}^2} \bigg( \mu^2 \frac{1}{T} \sum_{t=0}^{T-1} \mathbb{E} \Big[ \sum_{i \in [N]} p_i \sum_{l \in [L]} \sum_{j=r_{i,l}+1}^{r_l} \big( \sigma_{l,j}^{(t)} \big)^2 \Big] + \zeta^2 \bigg).
    \end{align*}
    
    Define $X=\frac{2\Delta_0}{Tq_{\min}} + \frac{2}{q_{\min}}(\chi \epsilon + 2\mu \epsilon^2)$ and $Y=\frac{2C_2}{q_{\min}}$, the function $f(\eta E)= \frac{1}{\eta E}\cdot X+ \eta E \cdot Y$ is minimized at $(\eta E)^\star = \sqrt{\frac{X}{Y}}$, and $f((\eta E)^\star) = 2\sqrt{X Y}$. Thus, by choosing 
    $$\eta = \frac{1}{E} \sqrt{\frac{X}{Y}}= \frac{1}{E} \sqrt{\frac{\frac{\Delta_0}{T} + \chi \epsilon + 2\mu \epsilon^2}{C_2}},$$ we balance the first two terms in the bound and have
    \begin{align*}
        & \frac{1}{T} \sum_{t=0}^{T-1} \left( \mathbb{E} \left[ \left\| \nabla_{\TupleA} \Loss^{\text{lora}} (\TupleB^{(t)}, \TupleA^{(t)}) \right\|_F^2 \right] + \mathbb{E} \left[ \left\| \nabla_{\TupleB} \Loss^{\text{lora}} (\TupleB^{(t)}, \TupleA^{(t)}) \right\|_F^2 \right] \right) \\
        & \leq 2\sqrt{X Y} + \frac{X}{Y} \frac{2C_3}{q_{\min}^2} + \left(\frac{X}{Y}\right)^{\frac{3}{2}}\! \frac{4C_4}{q_{\min}} + \frac{6C_1}{2q_{\min}^2} \bigg( \mu^2 \frac{1}{T} \sum_{t=0}^{T-1} \mathbb{E} \Big[ \sum_{i \in [N]} p_i \sum_{l \in [L]} \sum_{j=r_{i,l}+1}^{r_l} \big( \sigma_{l,j}^{(t)} \big)^2 \Big] \!+\! \zeta^2 \bigg) \\
        & = 4 \sqrt{\!\left(\frac{\Delta_0}{T q_{\min}} \!+ \!\frac{\chi \epsilon \!+\! 2\mu \epsilon^2}{q_{\min}}\right) \cdot \frac{C_2}{q_{\min}}} + \frac{\frac{\Delta_0}{T} + \chi \epsilon + 2\mu \epsilon^2}{C_2} \!\cdot\! \frac{2C_3}{q_{\min}^2} \!+4 \left( \frac{\frac{\Delta_0}{T} + \chi \epsilon + 2\mu \epsilon^2}{C_2} \right)^{\frac{3}{2}} \frac{C_4}{q_{\min}} \\
        & \quad + \frac{6C_1}{2q_{\min}^2} \bigg( \mu^2 \frac{1}{T} \sum_{t=0}^{T-1} \mathbb{E} \Big[ \sum_{i \in [N]} p_i \sum_{l \in [L]} \sum_{j=r_{i,l}+1}^{r_l} \big( \sigma_{l,j}^{(t)} \big)^2 \Big] + \zeta^2 \bigg).
    \end{align*}
\end{proof}

\end{document}